\documentclass[11pt]{article}
 
\usepackage[top=1in,bottom=1in,left=1in,right=1in]{geometry}
\usepackage{natbib}

\usepackage{setspace}

\usepackage{tablefootnote}
\usepackage{algorithm}
\usepackage{algorithmicx}
\usepackage{amsmath,amssymb,amsfonts}
\usepackage{graphicx,color}
\usepackage{url}
\usepackage{subfiles}
\usepackage{booktabs}
\usepackage{mathtools}
\usepackage{multirow}
\usepackage{microtype}
\usepackage[english]{babel}
\usepackage{ulem}
\usepackage[toc,page,title,titletoc]{appendix}

\usepackage[dvipsnames]{xcolor}
\usepackage{tikz, pgfplots, pgfplotstable}
 \usepgfplotslibrary{fillbetween}

\usepackage{hyperref}       
\hypersetup{
  colorlinks=true,
  linkcolor=black,
  urlcolor=black,
  citecolor=black
}
\usepackage[noend]{algpseudocode}

\usepackage{yhmath}
\usepackage[figurename=Fig.,font={small,stretch=0.84}]{caption}

\newtheorem{theorem}{Theorem}[section]    
\newtheorem{lemma}[theorem]{Lemma}         
\newtheorem{corollary}[theorem]{Corollary}
\newtheorem{proposition}[theorem]{Proposition}
\newtheorem{definition}[theorem]{Definition}

\newtheorem{assumption}[theorem]{Assumption}
\newtheorem{remark}[theorem]{Remark}

\usepackage{cleveref}
\crefname{section}{Section}{Sections}
\crefname{figure}{Figure}{Figures}
\crefname{theorem}{Theorem}{Theorems}
\crefname{lemma}{Lemma}{Lemmas}
\crefname{remark}{Remark}{Remarks}
\crefname{appendix}{Appendix}{Appendices}
\crefname{proposition}{Proposition}{Propositions}
\crefname{equation}{Equation}{Equations}
\crefname{algorithm}{Algorithm}{Algorithms}
\crefname{table}{Table}{Tables}

\usetikzlibrary{arrows.meta, calc, topaths, positioning, automata}
\pgfplotsset{compat=1.16}
\def\addlegendimage{\csname pgfplots@addlegendimage\endcsname}
\def\BibTeX{{\rm B\kern-.05em{\sc i\kern-.025em b}\kern-.08em
    T\kern-.1667em\lower.7ex\hbox{E}\kern-.125emX}}

\newcommand{\Reg}[0]{\mathrm{Reg}}

\newenvironment{myproof}{ {\noindent\it Proof.\ }}{\hfill $\square$\par}

\usepackage[dvipsnames]{xcolor}
\usepackage{tikz, pgfplots, pgfplotstable}
\usetikzlibrary{
  arrows.meta,positioning,shapes.geometric,fit,backgrounds,calc,patterns
}
\usepackage{hyperref}       
\usepackage{algorithm}
\usepackage{algorithmicx}
\makeatletter
\@ifundefined{theHALG@line}
  {\newcommand{\theHALG@line}{\thealgorithm.\arabic{ALG@line}}}
  {\renewcommand{\theHALG@line}{\thealgorithm.\arabic{ALG@line}}}
\makeatother
\usepackage{amsmath,amssymb,amsfonts}
\usepackage{graphicx,color}
\usepackage{url}
\usepackage{booktabs}
\usepackage{mathtools}
\usepackage{microtype}
\usepackage[english]{babel}
\usepackage{endnotes}
\usepackage{bm}

\title{
  Minimax-Optimal Semiparametric Contextual Dynamic Pricing with Multimodal Revenue
}

\usepackage{authblk}

\newif\ifanonymous
\ifdefined\ORANONYMOUS
  \anonymoustrue
\else
  \anonymousfalse
\fi

\ifanonymous
\author{Anonymous Authors}
\else
\author[1]{Xueping Gong}
\author[2]{Zhuoluo Zhang}
\author[3]{Zhaowei Miao}
\author[4]{Jiheng Zhang}

\affil[1]{School of Management, Xiamen University, \url{xgongah@xmu.edu.cn}}
\affil[2]{School of Management, Xiamen University, \url{zhangzhuoluo@xmu.edu.cn}}
\affil[3]{School of Management, Xiamen University, \url{miaozhaowei@xmu.edu.cn}}
\affil[4]{Department of Industrial Engineering and Decision Analytics, The Hong Kong University of Science and Technology, \url{jiheng@ust.hk}}

\fi

\date{}

\hypersetup{
  pdftitle={Minimax-Optimal Semiparametric Contextual Dynamic Pricing with Multimodal Revenue},
  pdfsubject={Contextual dynamic pricing with semiparametric demand and general revenue geometry},
  pdfkeywords={contextual dynamic pricing, semiparametric demand, bounded quantity feedback, shape-free revenue, minimax regret}
}
\ifanonymous
\hypersetup{pdfauthor={}}
\else
\hypersetup{
  pdfauthor={Xueping Gong, Zhuoluo Zhang, Jiheng Zhang, and Zhaowei Miao},
}
\fi

\begin{document}
\maketitle

\begin{abstract}
  We study contextual dynamic pricing with arbitrary covariate sequences and
  bounded, possibly nonbinary purchase quantities. Demand follows a
  semiparametric surplus-index model with an unknown linear valuation parameter
  and an unknown H\"older-smooth response. We impose neither concavity nor
  strong unimodality on revenue and allow nonunique optimal prices. We develop
  a pilot-corrected layered decision-partitioning policy that combines
  directional pilot estimation, local polynomial learning, predictable data
  assignment, and global action elimination. Pilot correction removes the
  first-order effect of valuation-parameter error, while permanent labels
  enable concentration under adaptive sampling. The policy attains the minimax
  smoothness-dependent horizon rate up to logarithmic factors; a matching lower
  bound already holds for a constant-context binary-demand subclass. 
\end{abstract}

\medskip
\noindent\textbf{Keywords:} contextual dynamic pricing; semiparametric demand;
bounded quantity feedback; shape-free revenue; minimax regret.
\par

\section{Introduction}
\label{sec:introduction}

Dynamic pricing is a central problem in revenue management. A seller must learn
how demand responds to prices while simultaneously using the accumulated
information to generate revenue. This learning problem becomes more challenging
when customer, product, or market heterogeneity is observed through covariates.
Although such information enables personalized pricing, every posted price also
determines what the seller subsequently learns. This feedback creates the
exploration--exploitation trade-off at the heart of contextual dynamic pricing;
see \citet{DP_review} for a broad review and \citet{DP_application} for
applications.

A widely studied framework models a customer's latent valuation as a
finite-dimensional function of the observed covariates plus an additive market
shock. In semiparametric formulations, the contextual component is typically
linear, whereas the distribution of the market shock is left unspecified
\citep{DP_parametricF,Explore_UCB,d_free_DP,DP_Fm}. This framework combines an
interpretable representation of customer heterogeneity with a flexible demand
curve. Binary purchase feedback dominates this single-index literature.
Shape-free regret guarantees are available for Lipschitz links, whereas the
higher-order-smooth literature typically couples smoothness with distributional
coverage and/or stable revenue geometry.

At Lipschitz smoothness, \citet{improvedCDP} allow arbitrary adaptive contexts
and \citet{gong2025minimax} allow general valuation classes under independent
and identically distributed contexts; both obtain the optimal shape-free
two-thirds horizon exponent with binary feedback. For smoother links, recent
improvements rely on nondegenerate or feature-diverse stochastic
contexts together with stable revenue geometry
\citep{CDP_smooth,CDP_feature_diversity}, or on a smooth oracle-price map
induced by strong unimodality \citep{CDP_oracle_price_map}.

A particularly consequential restriction is strong unimodality
\citep{dp_tight_2,CDP_smooth,CDP_oracle_price_map}. In the form used by recent
smooth-pricing analyses, this condition requires every contextual revenue
function to have a unique interior maximizer and its revenue loss to be
uniformly comparable to the squared distance from that maximizer. It is
therefore considerably stronger than smoothness or a local second-order
condition. Strong unimodality rules out separated revenue modes, flat
optimal-price regions, and many asymmetric or boundary-optimal revenue
landscapes. At the same time, it provides several powerful analytical
advantages: the policy can safely localize around one estimated optimizer,
price-estimation errors translate into quadratic revenue losses, and the
context-dependent optimal price can often be represented by a regular oracle
price map. Indeed, these properties are central to the algorithms and improved
regret guarantees in \citet{dp_tight_2}, \citet{CDP_smooth}, and
\citet{CDP_oracle_price_map}.

This paper studies contextual pricing without these structural conveniences.
In each period, the seller observes a covariate vector, posts a price, and
observes a bounded purchase quantity, which may be binary, discrete, or
continuous. Latent valuation is linear in the observed covariates and subject
to an unknown market shock. After integrating out the shock, expected demand
is an unknown function of the difference between price and a linear contextual
valuation index; both the index parameter and the response function are
unknown. The
usual binary-purchase model is obtained as a special case, but the formulation
also accommodates bounded unit sales and purchase volume. We allow the
covariates to arrive arbitrarily and impose only H\"older smoothness on the
observable response function. The induced revenue function may consequently be
multimodal, may contain a flat optimal region, and need not have a unique
maximizer. The smoothness assumption also has natural primitive foundations:
it follows, for example, when the market-shock distribution is smooth and the
quantity-response function has bounded variation, or when the
quantity-response function is itself sufficiently smooth.

Removing strong unimodality fundamentally changes the learning problem.
H\"older smoothness controls only local approximation and provides no
information about the global geometry of revenue. Learning accurately near a
provisional optimizer is therefore insufficient: another, statistically
similar region may contain a better and well-separated revenue mode. A policy
that localizes prematurely may never collect enough information to discover
that region. Consequently, the algorithm must preserve and compare candidates
across the entire price domain rather than organize its exploration around a
single estimated optimum or oracle price map.

A second difficulty arises from the interaction between the unknown valuation
parameter and the nonparametric response function. Directly inserting a pilot
parameter estimate into a nonparametric regression transfers the pilot error
to the fitted response curve at first order. This error can dominate the
higher-order approximation accuracy that smoothness would otherwise provide.
A natural remedy, used in classical single-index estimation and recent
semiparametric pricing analyses, is to profile a nonparametric fit over
candidate index parameters and refine the parameter through constrained least
squares
\citep{hardle1993optimal,ichimura1993sls,horowitz1996direct,
dp_tight_2,CDP_smooth}. This remedy is statistically and computationally
demanding. Because the local fits and their underlying sample assignments vary
with the candidate parameter, the resulting objective is generally nonconvex,
even in the twice-smooth case. Standard local optimization methods thus do not
provide a global-solution guarantee. Moreover, the same observations are used
both to construct the nonparametric fit and to evaluate the least-squares
criterion, creating a complex dependence structure in online settings and
making finite-sample confidence analysis delicate
\citep{CDP_smooth}. These issues motivate an approach that exploits
higher-order smoothness without repeatedly solving a nonconvex joint-estimation
problem.

Adaptive data collection introduces a third challenge. Posted prices, residual
bins, and subsequent exploration decisions all depend on previous outcomes.
If past observations are reassigned whenever the valuation estimate changes,
membership in a local dataset is determined retrospectively using information
that was unavailable when the observation was collected. This destroys the
predictable sampling structure needed for martingale concentration and makes
standard offline nonparametric guarantees inapplicable.

We address these challenges through a pilot-corrected layered
decision-partitioning (LDP) policy. The pilot module uses uniform-price
exploration to estimate the linear valuation component and evaluates its
uncertainty only along the currently observed covariate direction. A
well-covered direction enters the main pricing stage immediately, whereas an
insufficiently covered direction triggers additional pilot exploration. This
directional mechanism avoids distributional assumptions on the covariate
sequence.

During the main pricing stage, prices are represented relative to the estimated
valuation index, and demand is learned locally in the resulting residual
coordinate. An augmented local-polynomial representation absorbs the leading
pilot perturbation into the regression coefficients, leaving only a
second-order pilot error in addition to the usual nonparametric approximation
error. Thus, the policy benefits from higher-order smoothness without solving
the nonconvex constrained least-squares problems used in alternative joint
estimation procedures. Each observation is also assigned a permanent
layer--bin label before its demand is observed. These labels are never
recomputed after the pilot estimate changes, preserving predictable sampling
and enabling self-normalized concentration within every local dataset.

Finally, the layered policy maintains candidate actions over the entire
residual domain. An action is eliminated only when its optimistic revenue is
statistically separated from the best surviving benchmark. This design
protects distant modes and flat near-optimal regions and therefore performs
global rather than local revenue learning.

Balancing pilot exploration, local approximation, and statistical uncertainty
yields the minimax smoothness-dependent horizon rate, up to logarithmic
factors, for fixed problem primitives. We establish a matching expected-regret
lower bound over an admissible constant-context, binary-demand subclass. The
construction begins with a smooth instance whose revenue is flat over a
nondegenerate interval and introduces statistically indistinguishable
perturbations in separated price regions. 
Together, the upper and lower bounds characterize the minimax
dependence on the horizon for the general revenue class considered here.

\section{Contributions and Related Literature}
\label{sec:contributions}

\paragraph{Contributions.}
Our contributions are twofold.

\begin{itemize}

\item \textbf{A shape-free model and higher-order online method.}
We combine arbitrary context sequences, general H\"older smoothness, and
bounded quantity feedback in a translated-residual single-index model without
concavity, strong unimodality, or a unique optimal price. Our pilot-corrected
LDP policy integrates directional exploration, higher-order local polynomials,
predictable permanent labels, and global elimination. Relative to the
Lipschitz LDP framework of \citet{gong2025minimax}, this replaces an episodic
offline pilot and piecewise-constant learning with an online directional pilot,
first-order error correction, and higher-order residual learning.

\item \textbf{Minimax-optimal regret for shape-free smooth links.}
For fixed dimension and fixed problem primitives, the proposed policy achieves
$
    \widetilde{\mathcal O}\!\left(
        T^{\frac{\beta+1}{2\beta+1}}
    \right)
$
regret, matching
\(
    \Omega\!\left(T^{\frac{\beta+1}{2\beta+1}}\right)
\)
lower bound. 
Together, the upper and lower bounds characterize the minimax horizon dependence for the general revenue class studied in this paper.

\end{itemize}

\subsection{Related Literature}
\label{subsec:related-literature}

\paragraph{Semiparametric contextual dynamic pricing.}
Contextual pricing with a linear valuation component and an additive shock has
been studied under known parametric noise distributions
\citep{DP_parametricF,PersonalizedDP}, unknown nonparametric noise
\citep{Explore_UCB,d_free_DP,DP_Fm,DP_cox,LP_LV,CDP_feature_diversity}, and more general valuation or
utility classes \citep{DP_generalV,gong2025minimax}. These works differ in the
feedback available, the distributional assumptions on contexts, and the geometry
imposed on revenue. The papers most closely related to ours are summarized in
Table~\ref{tab:closest-contextual-pricing}. The rates in the last column suppress
logarithmic factors and polynomial dependence on fixed problem primitives.

\begin{table}[!htb]
    \centering
    \caption{Representative results for smooth semiparametric contextual
    pricing. Here \(\beta\) denotes link or tail smoothness; rates suppress
    logarithmic factors and dependence on fixed problem primitives.}
    \label{tab:closest-contextual-pricing}
    \resizebox{\textwidth}{!}{%
    \begin{tabular}{@{}llllll@{}}
        \toprule
        Work
        & Feedback
        & Contexts
        & Smoothness
        & Revenue geometry
        & Regret \\
        \midrule
        \citet{improvedCDP}
        & Binary
        & Arbitrary
        & \(\beta=1\)
        & General
        & \(\widetilde {\mathcal{O}}(T^{2/3})\) \\
        \citet{gong2025minimax}
        & Binary
        & i.i.d.
        & \(\beta=1\)
        & General
        & \(\widetilde {\mathcal{O}}(T^{2/3})\) \\
        \citet{dp_tight_2}
        & Binary
        & i.i.d.
        & \(\beta = 2\)
        & Strongly unimodal
        & \(\widetilde {\mathcal{O}}(T^{3/5})\) \\
        \citet{CDP_smooth}
        & Binary
        & i.i.d.
        & \(\beta\ge 2\)
        & Strongly unimodal
        & \(\widetilde {\mathcal{O}}(T^{\frac{\beta+1}{2\beta+1}})\) \\
        \citet{CDP_oracle_price_map}
        & Binary
        & Arbitrary
        & \(\beta\ge 2\)
        & Strongly unimodal
        & \(\widetilde {\mathcal{O}}(T^{\frac{2\beta-1}{4\beta-3}}+\sqrt T)\) \\
        This paper
        & Bounded quantity
        & Arbitrary
        & \(\beta\ge 1\)
        & General
        & \(\widetilde {\mathcal{O}}(T^{\frac{\beta+1}{2\beta+1}})\) \\
        \bottomrule
    \end{tabular}
    }
\end{table}

The principal comparison at \(\beta=1\) is \citet{improvedCDP}: it already
allows arbitrary adaptive contexts and general revenue geometry, and attains
the same \(T^{2/3}\) horizon exponent, but is restricted to binary feedback and
Lipschitz smoothness. \citet{gong2025minimax} introduce the LDP architecture for
a shape-free Lipschitz problem with binary feedback and i.i.d.\ contexts. Thus,
our claim at \(\beta=1\) is not a better horizon exponent. The contribution is
to retain arbitrary contexts and shape-free revenue while covering general
\(\beta\) and bounded quantity feedback.

For twice-smooth demand, \citet{dp_tight_2} establish the sharp
\(\widetilde{\mathcal O}(T^{3/5})\) regret rate under strong unimodality
and nondegenerate i.i.d.\ contexts, using contextual successive elimination
together with semiparametric estimation. \citet{CDP_smooth} extend this line
to general H\"older smoothness by combining local-polynomial estimation with
the stationary learning subroutine of \citet{dp_tight_2}, while retaining
similar revenue-geometry and context-coverage conditions.
\citet{CDP_oracle_price_map} also impose strong unimodality but allow
arbitrary contexts; by directly learning the resulting smooth oracle-price
map, they obtain a faster horizon exponent when \(\beta>2\). Our result
addresses the complementary regime in which higher-order smoothness is
available but the revenue geometry remains unrestricted. The flat-optimum
construction underlying our lower bound shows that this distinction is
structural: without strong unimodality, the hard instances need not admit a
unique and stable oracle-price map on which localization can be based.

\citet{CDP_non_lipschitz} study an adjacent, nonnested regime: binary feedback
with possibly discontinuous demand, including jumps induced by atoms, under
stochastic well-conditioned contexts. Their model relaxes link regularity,
whereas ours uses H\"older smoothness to handle arbitrary contexts and bounded
quantity feedback.

Algorithmically, our closest antecedent is the policy of
\citet{gong2025minimax}. Relative to its episodic offline pilot and
piecewise-constant residual learning, our policy uses an online directional
pilot, higher-order local polynomials, an augmented correction that makes pilot
error second order, and permanent residual labels that preserve predictable
sampling as the pilot evolves. Uniform-price exploration in
\citet{DP_parametricF,DP_Fm,DP_generalV} and the arbitrary-context mechanisms in
\citet{improvedCDP} provide additional points of contact.

\paragraph{Nonparametric and multimodal pricing.}
Without covariates, nonparametric dynamic pricing has been studied under a range
of smoothness and shape conditions
\citep{besbes2009dynamic,non_DP_linear,Multimodal_DP}. In particular,
\citet{Multimodal_DP} establish the minimax rate
\(\widetilde {\mathcal{O}}(T^{(\beta+1)/(2\beta+1)})\) for smooth multimodal revenue using
local-polynomial optimism. Their confidence indices explicitly incorporate a
local approximation envelope, and thus require the corresponding smoothness
radius. 

Bounded quantity observations are not new by themselves.
\citet{Multimodal_DP} analyze general demand feedback without contexts, while
\citet{PlinearDP,DP_separable} study partially linear and separable contextual
demand models. The present model instead places the unknown context effect
inside the translated residual \(p-\bm x^\top\bm\theta_\star\). Consequently,
the valuation shift and the nonparametric link must be learned jointly under an
arbitrary context sequence. This translated single-index structure, rather
than nonbinary feedback alone, is the relevant distinction.

\paragraph{Layered contextual bandits.}
Our analysis draws on standard confidence-based methods for contextual bandits
\citep{supLinUCB,linContextual,improvedUCB,banditBook}. Classical layered
partitioning separates observations by precision level to obtain sharper
confidence control. In the present problem, however, the local regression model
itself depends on an evolving pilot index and is only approximately specified.
The permanent labeling rule is therefore essential: it fixes the residual
coordinate, local bin, and stopping layer used for each observation before the
outcome is revealed. This converts the adaptively collected local samples into
predictable martingale arrays while retaining global price exploration.

\paragraph{Organization.}
Section~\ref{sec:setting} introduces the model, the induced demand link, and its
regularity. Section~\ref{sec:algorithm} presents the directional pilot,
pilot-corrected local regression, and layered pricing policy.
Section~\ref{sec:regret} establishes the regret upper bound and the matching
minimax lower bound. The paper
concludes with directions for future research.

\section{Problem Formulation}
\label{sec:setting}

\paragraph{Notation.}
Throughout the paper, we write \([n]:=\{1,\ldots,n\}\) for any positive integer \(n\),
denote the cardinality of a set \(A\) by \(|A|\),
and use \(\mathbf{1}\{E\}\) for the indicator of an event \(E\).
For vectors, \(\|\cdot\|_p\) denotes the standard \(\ell_p\) norm for \(1\le p\le\infty\).
The notation \(\widetilde{\mathcal O}\) suppresses absolute constants and logarithmic factors.
For a positive definite matrix \(A\), define \(\|\bm z\|_A:=\sqrt{\bm z^\top A\bm z}\).
For an interval \(I=[a,b]\), let \(\Pi_I(z):=\min\{b,\max\{a,z\}\}\) denote projection onto \(I\).

\paragraph{Basic Model.}
We consider a contextual dynamic pricing problem over a finite horizon \(T\). In each period \(t\in[T]\), a customer arrives and the seller observes a covariate vector \(\bm{x}_t\in\mathcal{X}\subseteq\mathbb{R}^d\) that characterizes the observable customer, product, or market characteristics.   
Unlike standard i.i.d. assumptions,
we do not impose any distributional assumption on the covariate sequence \(\{\bm{x}_t\}_{t=1}^T\); it can be arbitrary and possibly dependent.  
The covariate space is bounded, i.e., there
exists \(C_x<\infty\) such that
\(\sup_{\bm x\in\mathcal X}\|\bm x\|_2\le C_x\). For each \(t\in[T]\), let \(\mathcal F_{t-1}\) denote the
\(\sigma\)-field generated by all observations and policy randomization
available up to the end of period \(t-1\). Let \(U_t\) collect all
within-period randomization used by the policy after observing \(\bm x_t\).
Conditional on \(\mathcal F_{t-1}\vee\sigma(\bm x_t)\), the seed \(U_t\) is
drawn independently of the current valuation shock.

The customer's valuation for the product is modeled as a linear function of the covariates plus random noise:
\[
v_t = \bm x^\top_t\bm\theta_\star+ \epsilon_t,
\]
where \(\bm{\theta}_\star\in\Theta=\{\bm{\theta}\in\mathbb R^d:\|\bm{\theta}\|_2\le C_\theta\}\) is an unknown parameter vector, and \(\epsilon_t\) is an unobserved random shock.  
The shocks \(\{\epsilon_t\}_{t=1}^T\) are independent and identically
distributed according to an unknown cumulative distribution function
\(F\) with zero mean. For every \(t\in[T]\), \(\epsilon_t\) is independent
of \(\mathcal F_{t-1}\vee\sigma(\bm x_t,U_t)\).

\begin{assumption}[Bounded Valuations]
\label{ass:bounded}
There exist \(B>0\) and \(B_\epsilon\in(0,B/2)\) such that:
\begin{itemize}
    \item The noise is uniformly bounded: \(|\epsilon_t| \le B_\epsilon\) almost surely.
    \item The linear valuation satisfies \(\bm x^\top\bm\theta_\star \in [B_\epsilon, B - B_\epsilon]\) for all \(\bm{x}\in\mathcal{X}\).
\end{itemize}
\end{assumption}

Assumption~\ref{ass:bounded} guarantees that \(v_t=\bm x^\top_t\bm\theta_\star+\epsilon_t\in[0,B]\) almost surely, so the realized valuation and any feasible posted price lie in the same known interval \([0,B]\). 
This is a standard and natural requirement in practical pricing applications.

After observing \(\bm{x}_t\), the seller posts a per-unit price
\(p_t\in[0,B]\). The price and every auxiliary action label constructed by
the policy before demand is observed are measurable with respect to the full
pre-demand information field
$
    \mathcal G_t
    :=
    \mathcal F_{t-1}\vee\sigma(\bm x_t,U_t).
$
If the realized valuation satisfies \(v_t \ge p_t\), a sale occurs and the seller observes a positive purchase quantity \(y_t>0\); otherwise, no sale occurs and \(y_t=0\), i.e.,
\(\mathbf 1\{y_t>0\}=\mathbf 1\{v_t\ge p_t\}\)
almost surely.
The observed quantity \(y_t\) may be discrete (e.g., number of units) or
continuous (e.g., physical volume or usage), depending on the application.
We assume there exists a constant \(0<D<\infty\) such that \(0\le y_t\le D\) almost surely. 
The realized revenue at time \(t\) is \(p_t y_t\), and each round yields the observation triple \((\bm{x}_t,p_t,y_t)\). We set
\(\mathcal F_t:=\mathcal G_t\vee\sigma(y_t)\), so \(\mathcal G_t\) and
\(\mathcal F_t\) describe, respectively, all information immediately before
and after current demand is observed.

\paragraph{Generalized Single Index Model.}
The basic model above only specifies the qualitative relationship between $y_t$ and the surplus $v_t - p_t$. 
However, to enable statistical learning and regret minimization, the seller must quantify how the expected demand varies with the posted price and covariates. 
To make the problem tractable while retaining flexibility, we adopt a surplus-dependent single-index structure. 
This formulation generalizes the widely studied binary purchase model \citep{CDP_oracle_price_map,CDP_smooth,d_free_DP,DP_Fm,gong2025minimax} by allowing the realized demand quantity to depend on the surplus through a function $q(\cdot)$, rather than restricting to a binary sale indicator.

Specifically, we posit that, conditional on the information available
before demand is observed and on the realized valuation, the expected
demand depends on the current price and covariates only through the net
surplus \(v_t-p_t\).
This single-index form is natural: it implies that a customer's purchasing quantity decision is driven by the perceived value relative to the price, 
rather than by their absolute levels
separately. 
Under this premise, there exists an unknown baseline function $q(\cdot)$ that maps the surplus to the expected quantity. 
Importantly, we impose no parametric form, such as a linear or logistic specification, on \(q\), allowing it to accommodate richly heterogeneous and potentially nonmonotone demand responses.
The formal statement is as follows.

\begin{assumption}[Surplus-dependent demand response]
\label{ass:quantity-response}
There exists an unknown measurable function \(q:\mathbb R\to[0,D]\) with
\(q(z)=0\) for \(z<0\) or \(z>B\) and \(q(z)>0\) for
\(z\in[0,B]\), such that, for every \(t\in[T]\),
\[
    \mathbb E\!\left[
        y_t
        \,\middle|\,
        \mathcal G_t
        \vee
        \sigma(v_t)
    \right]
    =
    q(v_t-p_t)
    \quad
    \text{almost surely}.
\]
\end{assumption}

Assumption~\ref{ass:quantity-response} formalizes the surplus-dependent single-index structure introduced above. 
It posits that, conditional on all pre-demand information and the realized
valuation, the conditional mean demand is
determined solely by the realized surplus \(v_t-p_t\) and the functional form \(q(\cdot)\) is left fully nonparametric. 
The zero restriction for negative surplus is consistent with the no-sale
condition. The zero extension beyond \(B\) is only a technical convention,
because such surplus cannot arise under Assumption~\ref{ass:bounded}. For an
affordable purchase, the incidence condition implies positive realized
quantity, and we choose the version of the conditional-mean function that is
strictly positive on \([0,B]\).
Crucially, we impose no monotonicity, concavity, or unimodality assumptions on \(q\). 
This allows \(q\) to capture rich and potentially irregular demand patterns—such as quantity discounts, satiation effects, or Giffen-like behaviors at the micro level—that are often excluded by parametric specifications.

Define the effective residual domain \(\mathcal I_g:=[-B+B_\epsilon,B-B_\epsilon]\).
The function \(q\) characterizes the expected demand conditional on the realized valuation \(v_t=s\), which is unobservable to the seller. For decision-making and regret analysis, however, we require the expected demand conditional on the information available
before the current demand is observed. To bridge this gap, we integrate out the noise \(\epsilon_t\) from the surplus \(s-p_t = \bm\theta_\star^\top \bm x_t + \epsilon_t - p_t\). This yields an induced link function \(g\), defined as the convolution of \(q\) with the noise distribution \(F\):
\[
    g(u) := \mathbb E\!\left[ q(\epsilon_t - u) \right], \qquad u \in \mathcal I_g,
\]
where \(u = p - \bm\theta_\star^\top \bm x\). Then, by iterated expectations and the sequential exogeneity of
\(\epsilon_t\), we obtain,
\[ \mathbb E\!\left[ y_t \,\middle|\, \mathcal G_t \right]=g\!\left( p_t-\bm x_t^\top\bm\theta_\star \right). \]
In this way, \(g\) serves as the demand link that replaces the latent \(q\), and all subsequent estimation and regret guarantees will be stated in terms of \(g\).

\paragraph{Revenue and Regret.}
Given covariate \(\bm{x}\) and price \(p\), the expected revenue is
$
\mathsf{Rev}(\bm{x},p)=p\,g(p-\bm{\theta}_\star^\top\bm{x}),
$
Let
$
    p^\star(\bm{x})
    \in
    \arg\max_{p\in[0,B]}
    \mathsf{Rev}(\bm{x},p)
$
be an arbitrary maximizer, and write \( p_t^\star:=p^\star(\bm{x}_t)\) for period \(t\).
The cumulative regret is
\[
\mathrm{Reg}(T)=\sum_{t=1}^T\bigl(\mathsf{Rev}(\bm{x}_t,p_t^\star)-\mathsf{Rev}(\bm{x}_t,p_t)\bigr).
\]
The goal is to design a nonanticipating pricing policy that, for each \(t\), selects \(p_t\) based on the current covariate \(\bm{x}_t\) and all past observations \(\{(\bm{x}_s,p_s,y_s)\}_{s=1}^{t-1}\), so as to minimize \(\mathrm{Reg}(T)\) while learning the unknown valuation parameter \(\bm{\theta}_\star\) and the induced demand link \(g\).

\paragraph{Regularity.}
To enable nonparametric estimation, we directly assume the observable induced link function $g$ belongs to a H\"older class. 
This direct approach is standard in the dynamic pricing literature \citep{DP_separable,Multimodal_DP} and facilitates transparent convergence analysis, since $g$ is the direct object of estimation.

\begin{definition}[H\"older class]
\label{def:holder-class}
Let \(\mathcal I\subset\mathbb R\) be a compact interval, \(\beta\ge1\), and \(0<L<\infty\). Define
\[
\varpi(\beta):=\max\{k\in\mathbb Z_{\ge 0}:k<\beta\}.
\]
A function \(f:\mathcal I\to\mathbb R\) belongs to \(\mathcal H(\beta,L;\mathcal I)\) if \(f\) is \(\varpi(\beta)\)-times differentiable on \(\mathcal I\) and
\[
\left|
f^{(\varpi(\beta))}(u)-f^{(\varpi(\beta))}(u')
\right|
\le
L|u-u'|^{\beta-\varpi(\beta)},
\qquad
\forall u,u'\in\mathcal I.
\]

\end{definition}

\begin{assumption}[Smoothness of the Induced Link]
    \label{ass:g_smooth}
    The induced link function satisfies
    \(g\in \mathcal H(\beta, L_g; \mathcal I_g)\)
    for some \(\beta\ge 1\) and \(0<L_g<\infty\).
    \end{assumption}

For the contextual policy, the smoothness parameters \(\beta\) and \(L_g\)
are treated as known. Throughout this paper, \(C_g\)
denotes a fixed, class-level upper envelope for the derivatives of \(g\) up to
order \(\varpi(\beta)\). Because \(0\le g\le D\) on the
interval \(\mathcal I_g\), a standard one-dimensional interpolation inequality
allows \(C_g\) to be chosen as a function only of
\(B,B_\epsilon,D,\beta\), and \(L_g\). Thus, the constants used by the policy
do not depend on unknown instance-specific derivative values.

When the domain \(\mathcal I\) is clear from the context, we write 
    \(\mathcal H(\beta)\) for simplicity.

    \medskip
    \noindent\textit{Justification of Assumption~\ref{ass:g_smooth}.}
    A natural question is whether directly assuming Hölder smoothness on \(g\) is too restrictive. In our framework, this assumption is in fact mild, because the convolutional structure
    \(g(u)=\mathbb E[q(\epsilon_t-u)]\) allows \(g\) to inherit smoothness from more primitive components. The following concrete scenario shows when Assumption~\ref{ass:g_smooth} automatically holds.
    
    \medskip
    \noindent\textbf{Smooth noise distribution.}
    Suppose the noise distribution \(F\) itself is Hölder smooth, i.e., \(F\in\mathcal H(\beta,L_F;[-B,B])\). Even if the nonparametric demand response \(q\) only has bounded variation (allowing jumps and discontinuities), the convolution with \(F\) regularizes \(g\). The following proposition formalizes this.
    
    \begin{proposition}
        \label{prop:induced_F}
        Under Assumptions~\ref{ass:bounded} and \ref{ass:quantity-response}, and assuming that \(q\) is of bounded variation on \([0,B]\), if \(F\in\mathcal H(\beta,L_F;[-B,B])\), then \(g\in\mathcal H(\beta, (D+V_q)L_F; \mathcal I_g)\), where \(V_q\) denotes the total variation of \(q\) on \([0,B]\).
    \end{proposition}

    In the special case of binary demand where \(q(z)=\mathbf 1\{0\le z\le B\}\), 
    the induced link reduces to the survival function \(g(u)=1-F(u)\) whenever \(F\) is continuous, thereby recovering the classical binary-purchase dynamic pricing model. 
    This special case encompasses a broad range of existing work that assumes \(F\) to be Lipschitz \citep{besbes2009dynamic, gong2025minimax}, to have a Lipschitz first derivative \citep{dp_tight_2}, or to be \(m\)-times continuously differentiable \citep{DP_Fm, Multimodal_DP}. 
    This example demonstrates that Assumption~\ref{ass:g_smooth} is not an
    extraneous restriction: it follows from smooth valuation noise even when
    the quantity response itself has jumps.
    The model contains the bounded binary-purchase formulation as a special case
    while relaxing its revenue-geometry restrictions; it does not require the
    revenue function to be concave or the optimal price to be unique.

    \begin{remark}[Strong unimodality]
        \label{rem:without-strong-unimodality}
        Several recent studies impose a strong-unimodality condition: for every
        context \(\bm x\), the function
        \(\mathsf{Rev}(\bm x,\cdot)\) has a unique interior maximizer, and the
        revenue loss is comparable to the squared distance from that maximizer
        \citep{dp_tight_2,CDP_smooth,CDP_oracle_price_map}.
        This structure permits the policy to localize around a single price and,
        under additional smoothness, to learn a stable oracle-price map.
        
        We do not impose strong unimodality. Consequently,
        \(\mathsf{Rev}(\bm x,\cdot)\) may have multiple separated modes or
        nonunique maximizers, and small estimation errors may change which price
        region appears optimal. Localized search and oracle-price-map reductions
        are therefore not generally valid, motivating the global residual-space
        learning procedure developed in Section~\ref{sec:algorithm}.
    \end{remark}

    \section{Algorithm}
    \label{sec:algorithm}
    
    The policy has three coupled components, summarized in
    Figure~\ref{fig:method-overview}. First, uncertainty-triggered uniform-price
    exploration estimates the valuation index only along the current context
    direction. Second, a pilot-corrected local-polynomial feature absorbs the
    first-order index error, leaving the usual local approximation error and a
    second-order pilot term. Third, a layered rule compares candidate residual
    actions globally, rather than localizing around one provisional optimizer.

    Every main-policy observation receives its residual-bin and stopping-layer
    labels before demand is observed, and these labels are never recomputed.
    The resulting predictable datasets support martingale concentration even
    though prices and visited bins are selected adaptively. The following
    subsections define the pilot, correction, confidence radius, and global
    elimination rule in that order.

    \begin{figure*}[t]
        \centering
        \resizebox{\textwidth}{!}{
        \begin{tikzpicture}[
            >=Latex,
            font=\small,
            box/.style={
                draw,
                rounded corners=2pt,
                align=center,
                minimum height=7mm,
                inner sep=3pt
            },
            pilot/.style={box,fill=blue!8},
            learn/.style={box,fill=orange!12},
            ldp/.style={box,fill=green!10},
            decision/.style={
                diamond,
                draw,
                aspect=2.2,
                align=center,
                inner sep=1pt,
                fill=gray!8
            },
            arr/.style={->,thick},
            darr/.style={<->,thick,dashed},
            active/.style={draw=black,fill=white,line width=0.65pt},
            survivor/.style={
                draw=black,
                pattern=north east lines,
                pattern color=black!65,
                line width=0.65pt
            },
            eliminated/.style={draw=gray!65,fill=gray!30},
            finalaction/.style={draw=black,fill=black!75,line width=0.7pt},
            lab/.style={font=\scriptsize,align=center}
        ]
        
        \begin{scope}[xshift=0cm]
        \node[font=\bfseries] at (0,4.25) {(a) Adaptive pilot};
        
        \node[pilot] (context) at (0,3.45)
            {Observe context $\bm x_t$};
        
        \node[pilot] (score) at (0,2.40)
            {Compute $\widehat{\bm\theta}_t$ and 
             $\gamma_T\|\bm x_t\|_{M_t^{-1}}$};
        
        \node[decision] (gate) at (0,1.15)
            {$>\eta$?};
        
        \node[pilot,text width=3.0cm] (explore) at (-2.0,-0.15)
            {Pilot exploration\\
             $p_t\sim\mathrm{Unif}[0,B]$\\
             Update $(M_t,\bm b_t)$};
        
        \node[pilot,text width=3.0cm] (main) at (2.0,-0.15)
            {Certified pilot index\\
             $\widehat u_t=
             \Pi_{[0,B]}(\bm x_t^\top\widehat{\bm\theta}_t)$};

        \draw[arr] (context) -- (score);
        \draw[arr] (score) -- (gate);
        \draw[arr] (gate) -- node[lab,left] {yes} (explore);
        \draw[arr] (gate) -- node[lab,right] {no} (main);
        
        \node[lab,blue!70!black] at (2.0,-1.3)
            {$|\widehat u_t-\bm x_t^\top\bm\theta_\star|\le\eta$};
        \end{scope}
        
        \begin{scope}[xshift=4.4cm]
        \node[font=\bfseries] at (3.1,4.25)
            {(b) Global residual learning};
        
        \node[learn,text width=5.8cm] at (3.1,3.45)
            {$\widehat p_t(w)=\widehat u_t+w$,
             \qquad $w\in[-B,B]$};
        
        \draw[arr] (0.1,2.25) -- (6.2,2.25)
            node[right] {$w$};
        
        \foreach \x in {0.2,1.2,2.2,3.2,4.2,5.2,6.2}
            \draw (\x,2.05) -- (\x,2.45);
        
        \node[lab] at (0.7,1.93) {$I_1$};
        \node[lab] at (1.7,1.93) {$I_2$};
        \node[lab] at (2.7,1.93) {$I_j$};
        \node[lab] at (3.7,1.93) {$I_{j+1}$};
        \node[lab] at (5.7,1.93) {$I_N$};
        
        \fill[orange!25] (2.2,2.05) rectangle (3.2,2.4);
        
        \fill[blue!70!black] (2.53,2.25) circle (2pt);
        \node[lab,above] at (2.53,2.34) {$w$};
        
        \draw[red!75!black,thick] (2.97,2.25) circle (2.5pt);
        \draw[darr,red!75!black]
            (2.53,2.62) -- (2.97,2.62);

        \node[lab,red!75!black] at (3.30,2.88)
            {$w_t^{\rm true}:=\widehat p_t(w)-\bm x_t^\top\bm\theta_\star$,
             \quad $|w_t^{\rm true}-w|\le\eta$};
        
        \node[learn,text width=6.1cm] at (3.1,0.0)
        {
        $
        \begin{pmatrix}
        \phi_j(w)+(\widehat p_t(w)-w)\phi_j'(w)\\[-1mm]
        -\mathbf X_j(\bm x,w)
        \end{pmatrix}
        $
        \\[1mm]
        \scriptsize
        local polynomial \hspace{4mm}+\hspace{4mm}
        first-order pilot correction
        };
        
        \node[lab,orange!70!black,text width=6.2cm] at (3.1,-1.32)
        {
        $\displaystyle
        g\!\left(
        \widehat p_t(w)-\bm x_t^\top\bm\theta_\star
        \right)
        =
        \psi_{t,j}(w)^\top\bm z_j
        +
        \mathcal{O}(h^\beta+\eta^2)
        $
        };
        \end{scope}
        
        \begin{scope}[xshift=12.3cm]
        \node[font=\bfseries] at (3.0,4.25)
            {(c) Layered global UCB};
        
        \node[lab] at (0.25,3.30) {$s=1$};
        \node[lab] at (0.25,2.30) {$s=2$};
        \node[lab] at (0.25,1.30) {$s=3$};
        
        \foreach \k in {0,...,7}{
            \draw[active]
            ({0.75+0.55*\k},3.12) rectangle
            ({1.13+0.55*\k},3.48);
        }
        
        \foreach \k in {0,1,6}{
            \draw[survivor]
            ({0.75+0.55*\k},2.12) rectangle
            ({1.13+0.55*\k},2.48);
        }
        \foreach \k in {2,3,4,5,7}{
            \draw[eliminated]
            ({0.75+0.55*\k},2.12) rectangle
            ({1.13+0.55*\k},2.48);
        }
        
        \foreach \k in {1}{
            \draw[finalaction]
            ({0.75+0.55*\k},1.12) rectangle
            ({1.13+0.55*\k},1.48);
        }
        \foreach \k in {0,2,3,4,5,6,7}{
            \draw[eliminated]
            ({0.75+0.55*\k},1.12) rectangle
            ({1.13+0.55*\k},1.48);
        }

        \draw[arr] (3.0,2.96) --
            node[lab,right] {precision pass\\and elimination}
            (3.0,2.55);
        
        \draw[arr] (3.0,1.96) --
            node[lab,right] {finer layer}
            (3.0,1.55);
        
        \node[ldp,text width=2.8cm] (stop) at (7,2.75)
            {Explore and stop at $s_t$};
        
        \node[ldp,text width=2.8cm] (last) at (7,1.25)
            {Choose largest UCB};
        
        \draw[arr] (4.95,3.30) -- (stop);
        \draw[arr] (4.95,1.30) -- (last);
        
        \node[ldp,text width=5.8cm] (label) at (3.2,-0.5)
        {
        Fix $(s_t,j_t,w_t)$ before observing $y_t$
        \\
        $\displaystyle
        p_t=\widehat p_t(w_t)
        \;\longrightarrow\;
        y_t
        \;\longrightarrow\;
        \Psi_{t+1,s_t}^{j_t}
        =
        \Psi_{t,s_t}^{j_t}\cup\{t\}
        $
        };
        
        \end{scope}
        
        \draw[arr,blue!65!black]
            (2.2,0.8) -- (4.35,3.45);
        
        \draw[arr,orange!75!black]
            (10.8,3.45) -- (12.,3.45);
        
        \end{tikzpicture}
        }
        \caption{
        Overview of the proposed adaptive semiparametric pricing policy.
        Panel (a) shows the uncertainty-triggered pilot module.
        Panel (b) illustrates global residual-space learning and the augmented
        first-order correction for pilot-index error.
        Panel (c) depicts nested survivor sets under layered UCB exploration and elimination;
        outline, hatching, and solid fill distinguish active, surviving, and final actions in grayscale.
        }
        \label{fig:method-overview}
        \end{figure*}
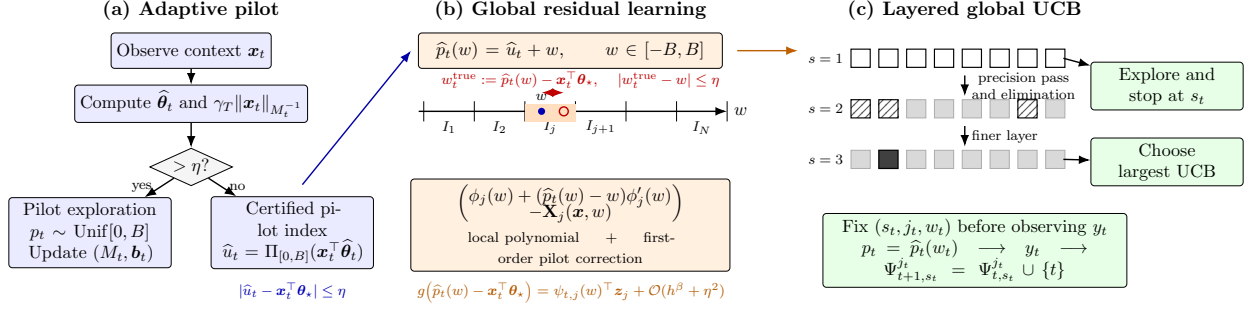

    \subsection{Adaptive Pilot Estimation}
    
    The pilot module is introduced first because the subsequent residual
    representation requires an accurate valuation index at the current
    covariate. Importantly, the policy does not require
    \(\widehat{\bm\theta}_t\) to approximate \(\bm\theta_\star\) uniformly
    in Euclidean norm. It only needs to certify the scalar index
    \(\bm x_t^\top\bm\theta_\star\) along the currently observed direction.
    This context-specific requirement is particularly useful when the
    covariates need not be independently sampled: directions that are
    already well covered can immediately enter the main pricing module,
    whereas insufficiently covered directions trigger additional pilot
    exploration.
    
    The pilot estimator is constructed exclusively from uniform-price
    exploration rounds. On such a round, the pseudo-response
    \(B\mathbf 1\{y_t>0\}\) is an unbiased observation of
    \(\bm x_t^\top\bm\theta_\star\). Accordingly, \(M_t\) and \(\bm b_t\)
    form the regularized design matrix and response vector based on the
    pilot-exploration data, and
    \(\widehat{\bm\theta}_t=M_t^{-1}\bm b_t\). The quantity
    $
        \gamma_T\|\bm x_t\|_{M_t^{-1}}
    $
    measures the remaining uncertainty in the current valuation index. If
    this quantity exceeds the target accuracy \(\eta\), the policy collects
    an additional uniform-price observation. Otherwise, it invokes the
    layered pricing module.

    \begin{algorithm}[htbp]
    \caption{Contextual Dynamic Pricing with Adaptive Pilot and Permanent Bin Labels}
    \label{alg:adaptive_pilot}
    \begin{algorithmic}[1]
    \Require Pilot accuracy \(\eta\), horizon \(T\), smoothness \(\beta\),
    bounds \(C_\theta,C_x,B,D\), 
    number of bins \(N\), regularization \(\lambda>0\), 
    confidence constants \(C_z,C_\psi,C_l\), and confidence level \(\delta\in(0,1)\)
    
    \State Set
    $
        \gamma_T\gets
        B\sqrt{d\log\left(1+TC_x^2/d\right)
        +2\log\left(4/\delta\right)}+C_\theta
    $
    and 
    $
        S\gets\max\left\{1,\left\lceil\log_2\sqrt T\right\rceil\right\},
    $
    \State Initialize \(\mathcal{T}^{\rm exp}\gets\emptyset\),
    \(M_1\gets I_d\), \(\bm b_1\gets\bm0_d\), and
    \(\Psi_{1,s}^j\gets\emptyset\) for every \(s\in[S]\) and \(j\in[N]\)
    
    \For{\(t=1,\ldots,T\)}
        \State Observe \(\bm x_t\) and set
        \(\widehat{\bm\theta}_t\gets M_t^{-1}\bm b_t\)
        \If{\(\gamma_T\|\bm x_t\|_{M_t^{-1}}>\eta\)}
            \State Set \(\mathcal{T}^{\rm exp}\gets \mathcal{T}^{\rm exp}\cup\{t\}\)
            \State Draw \(p_t\sim{\rm Unif}[0,B]\), post \(p_t\), and observe
            \(y_t\in[0,D]\)
            \State Set 
            \(
                M_{t+1}\gets M_t+\bm x_t\bm x_t^\top,
                \quad
                \bm b_{t+1}\gets\bm b_t+B\mathbf1\{y_t>0\}\bm x_t
            \)
            \State Set \(\Psi_{t+1,s}^j\gets
        \Psi_{t,s}^j,\) for all \(s\in[S]\) and \(j\in[N]\)
        \Else
            \State Run Algorithm~\ref{alg:lpdm_pilot_residual} and obtain
            \((s_t,j_t,w_t,p_t)\)
            \State Record \((s_t,j_t,w_t)\) as the permanent layer, bin, and residual-action labels of period \(t\)
            \State Post \(p_t\) and observe \(y_t\in[0,D]\)
            \State Set \(M_{t+1}\gets M_t\), \(\bm b_{t+1}\gets\bm b_t\), and
            \[
                \Psi_{t+1,s}^j
                \gets
                \begin{cases}
                    \Psi_{t,s}^j\cup\{t\},&s=s_t,j=j_t\\
                    \Psi_{t,s}^j,& \text{otherwise}.
                \end{cases}
            \]
        \EndIf
    \EndFor
    \end{algorithmic}
    \end{algorithm}
    
    Algorithm~\ref{alg:adaptive_pilot} gives the outer wrapper of the
    policy; the layered pricing subroutine is specified subsequently.
    We refer to rounds on which uniform pricing is used as
    \emph{pilot-exploration rounds}, and to all remaining rounds as
    \emph{LDP rounds}. Although later pilot-exploration rounds may update
    \(\widehat{\bm\theta}_t\), the bin and layer labels assigned to an LDP
    observation are permanent and are never recomputed using a later pilot.
    
    Algorithm~\ref{alg:lpdm_pilot_residual} assigns each LDP round a stopping layer \(s_t\) and a permanent bin label \(j_t\). For each pair \((s,j)\), the outer wrapper maintains a dataset
    $
        \Psi_{t,s}^j,
    $
    which stores all past LDP observations carrying that fixed label. 
    The task of selecting \((s_t,j_t)\) is entirely delegated to Algorithm~\ref{alg:lpdm_pilot_residual}; the outer algorithm only appends the current observation to the corresponding \(\Psi_{t+1,s_t}^{j_t}\).

    The following lemma validates the uncertainty gate in
    Algorithm~\ref{alg:adaptive_pilot}. It provides a time-uniform,
    direction-dependent confidence bound and therefore guarantees that the
    valuation index is sufficiently accurate whenever the policy enters an
    LDP round.
    
    \begin{lemma}[Pilot estimation confidence]
    \label{lem:pilot-confidence}
    With the value of \(\gamma_T\) specified in
    Algorithm~\ref{alg:adaptive_pilot}, with probability at least
    \(1-\delta/2\), simultaneously for every \(t\le T\) and every
    \(\bm x\in\mathbb R^d\),
    \begin{equation}
    \label{eq:pilot-confidence-event}
        \left|
            \bm x^\top
            \bigl(
                \widehat{\bm\theta}_t-\bm\theta_\star
            \bigr)
        \right|
        \le
        \gamma_T\|\bm x\|_{M_t^{-1}}.
    \end{equation}
    \end{lemma}

    Consequently, on the event in
    Lemma~\ref{lem:pilot-confidence}, every LDP round \(t\) satisfies
    $
        \left|
            \Pi_{[0,B]}
            \bigl(\bm x_t^\top\widehat{\bm\theta}_t\bigr)
            -
            \bm x_t^\top\bm\theta_\star
        \right|
        \le \eta.
    $
    Therefore, for every feasible pilot-residual action \(w\),
    \[
        \widehat p_t(w)-\bm x_t^\top\bm\theta_\star
        =
        w+
        \left[
            \Pi_{[0,B]}
            \bigl(\bm x_t^\top\widehat{\bm\theta}_t\bigr)
            -
            \bm x_t^\top\bm\theta_\star
        \right],
    \]
    so the true residual differs from its pilot counterpart \(w\) by at
    most \(\eta\). This certification localizes the evaluation point of
    \(g\), but it does not by itself eliminate the resulting first-order
    error. We next incorporate this displacement directly into the local
    regression feature.
    
    \subsection{Pilot-Corrected Local Regression}
    \label{subsec:local-regression}
    
    The pilot certificate controls the valuation-index error at the current
    context, but does not by itself remove its effect on estimating the unknown
    demand link. 
    On the pilot-confidence event, \(|\widehat p_t(w)-w-\bm x_t^\top\bm\theta_\star|\le\eta\).
    An ordinary local-polynomial regression indexed by \(w\) ignores this
    displacement and therefore incurs a first-order error of order \(\eta\).
    When \(\beta>1\), this error may dominate the target local-polynomial bias
    \(h^\beta\), where \(h\) denotes the bin width defined below.
    
    Existing smooth semiparametric pricing methods address this coupling by
    jointly refining the index parameter and the local approximation of \(g\)
    \citep{dp_tight_2,CDP_smooth}. Conditional on a candidate
    \(\bm\theta\), both the local regressors and their Gram matrix depend on
    \(\bm\theta\). Profiling out the polynomial coefficients consequently
    leads to a constrained least-squares problem that is generally nonconvex
    and does not, in general, admit a globally certified solution.
    
    Our key simplification is to target prediction rather than structural
    joint identification. Pricing requires accurate predictions of
    $
        g\!\left(p-\bm x^\top\bm\theta_\star\right),
    $
    but does not require a separate refinement of \(\bm\theta_\star\) within
    every local regression. We therefore retain the complete first-order
    effect of the pilot-index displacement and estimate the resulting products
    as composite coefficients. This prediction-preserving lifting converts
    the coupled estimation problem into a convex ridge regression while
    leaving only a second-order pilot remainder.
    
    To formalize this construction, partition the residual space \([-B,B]\) \footnote{For the analysis, we extend \(g\) from
\(\mathcal I_g\) to \([-B,B]\) by its endpoint Taylor polynomials of
order \(\varpi(\beta)\), and continue to denote the extension by \(g\).
The extension agrees with the original induced link on
\(\mathcal I_g\) and belongs to
\(\mathcal H(\beta,L_g;[-B,B])\). It is used only to define the
comparison coefficients and does not modify the demand model or the
policy.}
    into \(N\) bins of equal width \(h=2B/N\). Let
    $
        a_j=-B+\frac{2Bj}{N},
    $
    for $j=0,1,\ldots,N,$
    and define
    \[
        I_j=[a_{j-1},a_j),
        \quad j=1,\ldots,N-1,
        \qquad
        I_N=[a_{N-1},a_N].
    \]
    For each \(I_j\), define the local-polynomial feature anchored at
    \(a_{j-1}\) by
    \[
        \phi_j(u)
        :=
        \bigl(
            1,\,
            u-a_{j-1},\,
            \ldots,\,
            (u-a_{j-1})^{\varpi(\beta)}
        \bigr)^\top
        \in\mathbb R^{1+\varpi(\beta)},
    \]
    and let \(\phi_j'(u)\) denote its componentwise derivative. When
    \(\varpi(\beta)\ge1\), define
    \[
        \mathbf X_j(\bm x,u)
        :=
        \begin{pmatrix}
            \bm x\\
            2(u-a_{j-1})\bm x\\
            \vdots\\
            \varpi(\beta)(u-a_{j-1})^{\varpi(\beta)-1}\bm x
        \end{pmatrix}
        \in\mathbb R^{\varpi(\beta)d};
    \]
    when \(\varpi(\beta)=0\), let \(\mathbf X_j(\bm x,u)\) be the empty
    vector.
    
    For \(w\in I_j\), expanding around the bin anchor and retaining the
    first-order pilot displacement gives
    \begin{equation}
    \label{eq:pilot-corrected-expansion}
    \begin{aligned}
        g\!\left(
            \widehat p_t(w)-\bm x_t^\top\bm\theta_\star
        \right)
        ={}&
        \phi_j(w)^\top
        \begin{pmatrix}
            g(a_{j-1})\\
            g'(a_{j-1})\\
            \vdots\\
            g^{(\varpi(\beta))}(a_{j-1})/\varpi(\beta)!
        \end{pmatrix} \\
        &+
        (\widehat p_t(w)-w-\bm x_t^\top\bm\theta_\star)\phi_j'(w)^\top
        \begin{pmatrix}
            g(a_{j-1})\\
            g'(a_{j-1})\\
            \vdots\\
            g^{(\varpi(\beta))}(a_{j-1})/\varpi(\beta)!
        \end{pmatrix}
        +R_{t,j}(w),
    \end{aligned}
    \end{equation}
    where, under \(\eta\le h\),
    $
        |R_{t,j}(w)|
        \le
        C_l\left(h^\beta+\eta^2\right).
    $
    The two components of the remainder correspond to the local-polynomial
    approximation error and the second-order pilot-displacement error,
    respectively.
    
    Hence, the first-order representation
    in \eqref{eq:pilot-corrected-expansion} can be rewritten as
    \[
    \begin{aligned}
        &\bigl[
            \phi_j(w)
            +\bigl(\widehat p_t(w)-w\bigr)\phi_j'(w)
        \bigr]^\top
        \begin{pmatrix}
            g(a_{j-1})\\
            g'(a_{j-1})\\
            \vdots\\
            g^{(\varpi(\beta))}(a_{j-1})/\varpi(\beta)!
        \end{pmatrix} 
        -
        \mathbf X_j(\bm x_t,w)^\top
        \begin{pmatrix}
            g'(a_{j-1})\bm\theta_\star\\
            \vdots\\
            \dfrac{g^{(\varpi(\beta))}(a_{j-1})}
            {\varpi(\beta)!}\bm\theta_\star
        \end{pmatrix}.
    \end{aligned}
    \]
    The second term contains products between the local derivative
    coefficients and \(\bm\theta_\star\). Enforcing their common factorization
    through \(\bm\theta_\star\) would recover the nonlinear coupling underlying
    the constrained joint estimators in prior work. Instead, we treat these
    products directly as composite coefficients.
    
    Specifically, define a $1+\varpi(\beta)(d+1)$-dimensional column vector
    \[
    \bm z_j:= \Bigg(
        g(a_{j-1}),\ 
        g'(a_{j-1}),\ 
        \ldots,\ 
        \frac{g^{(\varpi(\beta))}(a_{j-1})}{\varpi(\beta)!},\ 
        \bigl(g'(a_{j-1})\bm\theta_\star\bigr)^\top,\ 
        \ldots,\ 
        \left(\frac{g^{(\varpi(\beta))}(a_{j-1})}{\varpi(\beta)!}\bm\theta_\star\right)^\top
    \Bigg)^\top,
    \]
    and, for every \(t\), \(j\), and \(w\in I_j\), define the
    pilot-corrected feature
    \begin{equation}
    \label{eq:augmented-local-feature}
        \psi_{t,j}(w)
        :=
        \begin{pmatrix}
            \phi_j(w)
            +\bigl(\widehat p_t(w)-w\bigr)\phi_j'(w)\\[1mm]
            -\mathbf X_j(\bm x_t,w)
        \end{pmatrix}
        \in\mathbb R^{1+\varpi(\beta)(d+1)}.
    \end{equation}
    Equation~\eqref{eq:pilot-corrected-expansion} then becomes
    \begin{equation}
    \label{eq:lifted-mean-representation}
        g\!\left(
            \widehat p_t(w)-\bm x_t^\top\bm\theta_\star
        \right)
        =
        \psi_{t,j}(w)^\top\bm z_j
        +
        R_{t,j}(w),
        \qquad
        |R_{t,j}(w)|
        \le
        C_l\left(h^\beta+\eta^2\right).
    \end{equation}
    Thus, up to a higher-order deterministic remainder, the conditional mean
    belongs to a finite-dimensional linear prediction class. When
    \(\beta=1\), we have \(\varpi(\beta)=0\); the derivative block is absent,
    and the pilot displacement is absorbed into the \(\mathcal{O}(h)\)
    approximation error under \(\eta\le h\).
    
    For each layer--bin pair \((s,j)\), we estimate the composite coefficient
    using observations carrying the corresponding permanent labels. Define
    $
        \Lambda_{t,s}^j
        :=
        \sum_{i\in\Psi_{t,s}^j}
        \psi_{i,j}(w_i)\psi_{i,j}(w_i)^\top.
    $
    Fix a regularization parameter $\lambda>0$.
    The ridge estimator is
    \begin{equation}
    \label{eq:lifted-ridge-estimator}
    \begin{aligned}
        \widehat{\bm z}_{t,s}^j
        &:=
        \operatorname*{argmin}_{\bm z\in\mathbb R^{1+\varpi(\beta)(d+1)}}
        \left\{
            \sum_{i\in\Psi_{t,s}^j}
            \bigl(
                y_i-\psi_{i,j}(w_i)^\top\bm z
            \bigr)^2
            +
            \lambda\|\bm z\|_2^2
        \right\} =
        \bigl(\lambda I+\Lambda_{t,s}^j\bigr)^{-1}
        \sum_{i\in\Psi_{t,s}^j}
        y_i\psi_{i,j}(w_i).
    \end{aligned}
    \end{equation}
    For \(w\in I_j\), the corresponding demand estimate based on
    \(\Psi_{t,s}^j\) is
    \begin{equation}
    \label{eq:lifted-demand-estimator}
        \widehat g_{t,s}^j(w)
        :=
        \Pi_{[0,D]}\!\left(
            \psi_{t,j}(w)^\top
            \widehat{\bm z}_{t,s}^j
        \right).
    \end{equation}
    
    The estimator in \eqref{eq:lifted-ridge-estimator} does not approximate a
    local solution of the preceding nonconvex joint-estimation problem.
    Instead, it relaxes the factorization restrictions among the derivative
    coefficients and estimates the prediction-sufficient composite
    coefficients directly. The population vector \(\bm z_j\) remains in
    the lifted class, so this relaxation introduces no additional first-order
    approximation error. Its deterministic misspecification is bounded by
    \(\mathcal{O}(h^\beta+\eta^2)\).
    
    Once the lifted features have been constructed,
    \eqref{eq:lifted-ridge-estimator} is a standard strictly convex ridge
    regression and therefore has the displayed unique global solution.
    Moreover, each feature \(\psi_{i,j}(w_i)\) is determined before \(y_i\) is
    observed, allowing the stochastic regression errors to be analyzed as a
    predictable martingale transform. Under
    \(\eta^2\lesssim h^\beta\), the lifting preserves the local-polynomial
    approximation order while avoiding nonconvex joint refinement altogether.
    The next subsection converts these point predictions into confidence
    bounds and incorporates them into the layered decision rule.

    \subsection{Layered Decision Making}

    The preceding pilot-corrected regression converts the coupled semiparametric
    estimation problem into a convex linear prediction problem and provides a
    point estimate of demand for each feasible residual action. Point predictions
    alone, however, are insufficient for sequential pricing: the amount of
    information varies across layer--bin pairs, and the lifted representation
    retains a deterministic approximation error of order
    $\mathcal{O}(h^\beta+\eta^2)$. We therefore construct confidence bounds that jointly
    account for stochastic estimation error, ridge regularization, and this
    remaining approximation error, and use them to form optimistic estimates of
    the corresponding revenues.

    \paragraph{Confidence radius.}
    The explicit form of the radius is
    \begin{equation}
        \label{eq:r}   
        \begin{aligned}
        r_{t,s}^j(w)
        =\min\Biggl\{D,\,
        &\underbrace{
        C_{\psi}D\sqrt{\iota_T}
        \left\|\psi_{t,j}(w)\right\|_
        {(\Lambda_{t,s}^j+\lambda I)^{-1}}
        }_{\textnormal{stochastic estimation error}}
        +
        \underbrace{
        C_z\sqrt{\lambda}
        \left\|\psi_{t,j}(w)\right\|_
        {(\Lambda_{t,s}^j+\lambda I)^{-1}}
        }_{\textnormal{ridge bias}}\\
        &+
        \underbrace{
        C_l(h^\beta+\eta^2)\sqrt{|\Psi_{t,s}^j|}
        \left\|\psi_{t,j}(w)\right\|_
        {(\Lambda_{t,s}^j+\lambda I)^{-1}}
        }_{\textnormal{historical approximation error}}+
        \underbrace{
        C_l(h^\beta+\eta^2)
        }_{\textnormal{current-action approximation error}}
        \Biggr\},
        \end{aligned}
    \end{equation}
    with
    \begin{equation}
    \label{eq:iota_T}
        \iota_T
        :=
        2\log\left(\frac{4SN}{\delta}\right)
        +
        \bigl[1+\varpi(\beta)(d+1)\bigr]
        \log\left(1+\frac{T}{\lambda}\right).
    \end{equation}
    If $|\Psi_{t,s}^j|=0$, we set $r_{t,s}^j(w)=D$.

    The four summands in \eqref{eq:r} have distinct roles: self-normalized
    concentration controls the centered demand noise; the ridge penalty
    contributes \(C_z\sqrt\lambda\); the accumulated historical remainders
    contribute the term proportional to \(\sqrt{|\Psi_{t,s}^j|}\); and the
    final term controls the current action's approximation error. Projection
    onto \([0,D]\) cannot increase error relative to a target in that interval.
    Proposition~\ref{prop:ucb} formalizes this decomposition and proves that
    \(r_{t,s}^j(w)\) is a simultaneous confidence radius.

    
    \begin{proposition}[Confidence bound]
    \label{prop:ucb}
    Suppose Assumptions~\ref{ass:bounded},
    \ref{ass:quantity-response}, and \ref{ass:g_smooth}
    hold. Let \(\lambda>0\), \(\delta\in(0,1)\), and \(0<\eta\le h\). Let \(C_z,C_l,C_\psi>0\) be the finite constants specified in
\eqref{eq:C-z-definition}, \eqref{eq:C-l-definition}, and \eqref{eq:C-psi-definition}. Then, with probability at least \(1-\delta\), simultaneously
    for every LDP round \(t\), every \(s\in[s_t]\), and every
    \((j,w)\in\mathcal A_{t,s}\), 
    \begin{equation}
    \label{eq:linearized-uniform-confidence}
        \left|
            \widehat g_{t,s}^j(w)
            -g\!\left(
                \widehat p_t(w)-\bm x_t^{\top}\bm\theta_\star
            \right)
        \right|
        \le
        r_{t,s}^j(w).
    \end{equation}
    \end{proposition}

    \paragraph{Layered exploration and elimination.}

    \begin{algorithm}[htbp]
      \caption{Layered Decision Making with Pilot-Residual Actions}
      \label{alg:lpdm_pilot_residual}
      \begin{algorithmic}[1]
      \Require Current round \(t\), context \(\bm x_t\), certified pilot
      \(\widehat{\bm\theta}_t\), permanent datasets
      \(\{\Psi_{t,s}^j\}_{s\in[S],j\in[N]}\), bin partition \(\{I_j\}_{j\in[N]}\), and
      constants \(C_{\psi},C_z,C_l\)
      
      \State Discretize the residual-action space by
      $
          \mathcal W
          :=
          \left\{-B+kT^{-1/2}:k=0,1,\ldots,\lfloor2B\sqrt T\rfloor\right\}
          \cup\{0,B\}.
      $
      \State Set
      $
          \widehat u_t
          \gets
          \Pi_{[0,B]}
          \!\left(\bm x_t^\top\widehat{\bm\theta}_t\right),
          \widehat p_t(w)
          \gets
          \widehat u_t+w .
      $
      
      \For{each \(j\in[N]\)}
          \State Set
          $
              \mathcal W_{t,j}
              \gets
              \left\{
                  w\in\mathcal W\cap I_j:
                  0\le \widehat p_t(w)\le B
              \right\}.
          $
      \EndFor
      
      \State Set
      $
          s\gets1,
          \mathcal A_{t,1}
          \gets
          \left\{
              (j,w):j\in[N],\ w\in\mathcal W_{t,j}
          \right\}.
      $
      
      \Repeat
          \State Set
          $
              \mathcal J_{t,s}
              \gets
              \left\{
                  j\in[N]:
                  (j,w)\in\mathcal A_{t,s}
                  \text{ for some }w
              \right\}.
          $
      
          \For{each \(j\in\mathcal J_{t,s}\)}
              \State Compute
              $
                  \widehat{\bm z}_{t,s}^j
              $
              by \eqref{eq:lifted-ridge-estimator}
          \EndFor
      
          \For{each \((j,w)\in\mathcal A_{t,s}\)}
              \State Compute
              $
                  \widehat g_{t,s}^j(w)
                  \gets
                  \Pi_{[0,D]}\!\left(
                      \psi_{t,j}(w)^\top
                      \widehat{\bm z}_{t,s}^j
                  \right).
              $
      
              \State Compute $r_{t,s}^j(w)$ by \eqref{eq:r}
      
              \State Set
              $
                  U_{t,s}^j(w)
                  \gets
                  \widehat p_t(w)
                  \min\left\{
                      D,\,
                      \widehat g_{t,s}^j(w)+r_{t,s}^j(w)
                  \right\}, 
                  \operatorname{wid}_{t,s}^j(w)
                  \gets
                  \widehat p_t(w)r_{t,s}^j(w).
              $
          \EndFor
      
          \If{\(s=S\)}
              \State Select
              $
                  (j_t,w_t)
                  \in
                  \operatorname*{argmax}_{(j,w)\in\mathcal A_{t,s}}
                  U_{t,s}^j(w),
                  \qquad
                  s_t\gets s.
              $
      
          \ElsIf{
              \(
                  \displaystyle
                  \max_{(j,w)\in\mathcal A_{t,s}}
                  \operatorname{wid}_{t,s}^j(w)
                  \le BD\,2^{-s}
              \)
          }
              \State Set
                  $
                  \mathcal A_{t,s+1}
                  \gets
                  \Biggl\{
                      (j,w)\in\mathcal A_{t,s}:
                      U_{t,s}^j(w)
                      \ge{}
                      \max_{(j',w')\in\mathcal A_{t,s}}
                      U_{t,s}^{j'}(w')
                      -BD\,2^{1-s}
                  \Biggr\},
                 $
              \State Set \(s\gets s+1\).
      
          \Else
              \State Set
              $
                  \mathcal A_{t,s}^{\mathrm u}
                  \gets
                  \left\{
                      (j,w)\in\mathcal A_{t,s}:
                      \operatorname{wid}_{t,s}^j(w)>BD\,2^{-s}
                  \right\}.
              $
              \State Select
              $
                  (j_t,w_t)
                  \in
                  \operatorname*{argmax}_{(j,w)\in\mathcal A_{t,s}^{\mathrm u}}
                  U_{t,s}^j(w),
                  \qquad
                  s_t\gets s.
              $
          \EndIf
      \Until{\((j_t,w_t)\) has been selected}
      
      \State Set \(p_t\gets\widehat p_t(w_t)\).
      \State \Return \((s_t,j_t,w_t,p_t)\).
      \end{algorithmic}
      \end{algorithm}

    The layered decision rule proceeds by refining the revenue resolution
    level by level. At layer~$s$, the target resolution is $BD\,2^{-s}$.
    For every action $(j,w)$ surviving in the current active set
    $\mathcal A_{t,s}$, the algorithm maintains an optimistic revenue
    $U_{t,s}^j(w)$ and a width $\operatorname{wid}_{t,s}^j(w)$ that
    measures the remaining uncertainty in its estimated revenue.
    The policy then decides whether to sample at this layer or to advance.
    
    If there exists an action in $\mathcal A_{t,s}$ with width strictly
    greater than $BD\,2^{-s}$, the policy selects one such under-explored
    action and assigns the current observation permanently to layer~$s$
    and the corresponding bin~$j$. This sampling step directly reduces
    the uncertainty of that action.
    
    Otherwise, if every surviving action satisfies
    $\operatorname{wid}_{t,s}^j(w)\le BD\,2^{-s}$, then the layer is
    considered sufficiently resolved. 
    Our policy eliminates all actions whose optimistic revenue falls more than $BD\,2^{1-s}$ below the largest optimistic
    revenue among the survivors, and passes the remaining actions to the
    next layer~$s+1$. The safety of this elimination follows from the confidence bounds.
    Crucially, this argument does not require strong
    unimodality, local curvature, or uniqueness of the revenue maximizer,
    because it compares all surviving actions globally.
    
    If the policy proceeds to layer~$s+1$, it repeats the same check with
    the finer resolution $BD\,2^{-(s+1)}$. Once all layers up to~$S$ have
    been passed, the final action is chosen as the one
    with the largest UCB among the remaining active set.

        Algorithm~\ref{alg:lpdm_pilot_residual} implements this layered
        procedure. It starts with the full feasible residual grid and
        sequentially examines the layers. At each layer, it first computes the
        required estimates and widths; if an under-explored action is found,
        the round stops and the observation is assigned to that layer.
        Otherwise, it performs elimination and continues without consuming the
        current observation. The process terminates when either an action is
        selected for sampling or the final layer is reached, in which case the
        action with the highest UCB is chosen.

        \section{Regret Analysis}
        \label{sec:regret}
        
        \subsection{Upper Bound}
        We now analyze the regret of the proposed policy. The proof follows
        the two operating modes of Algorithm~\ref{alg:adaptive_pilot}. On LDP
        rounds, the confidence bounds translate the stopping layer into a
        one-period revenue guarantee. We then control how often the policy can
        stop at each layer by combining the layer-specific uncertainty
        threshold with an elliptical-potential bound for the permanently
        labeled observations. This yields a bound on the regret accumulated
        over all LDP rounds. Separately, the uncertainty gate limits the number
        of pilot-exploration rounds through the growth of the pilot design
        matrix. Combining these two components gives the finite-sample regret
        bound, after which we establish a matching lower bound in its
        dependence on the horizon.
        
        We begin with the one-period analysis of an LDP round. The confidence
        event guarantees that the best action in the current discretized
        candidate set survives every completed layer. Consequently, the
        stopping layer controls the revenue loss relative to the best
        discretized action, while the mesh of the residual grid controls the
        additional loss relative to the continuous oracle price.

        \begin{lemma}
        \label{lem:ldp_revenue_gap_discrete_fixed_residual}
        Fix an LDP round \(t\), and suppose that the confidence bounds \eqref{eq:linearized-uniform-confidence} hold at every layer visited in
        period \(t\). For each visited layer \(s\), define
        $
            V_{t,s}
            :=
            \max_{(j,w)\in\mathcal A_{t,s}}
            \mathsf{Rev}\bigl(\bm x_t,\widehat p_t(w)\bigr).
        $
        Then:
        \begin{enumerate}
            \item If the algorithm passes the precision check at a layer \(s<S\),
            then
            $
                V_{t,s+1}=V_{t,s}.
            $
        
            \item If \(s_t\ge2\), then
            $
                V_{t,1}
                -
                \mathsf{Rev}(\bm x_t,p_t)
                \le
                8BD\,2^{-s_t}.
            $
        \end{enumerate}
        Moreover, there exists a finite constant \(L_{\mathrm{Rev}}>0\), such that, for every
        \(s_t\in[S]\),
        \begin{equation}
        \label{eq:one-step-revenue-gap-discrete-ldp-simplified}
            \mathsf{Rev}(\bm x_t,p_t^\star)
            -
            \mathsf{Rev}(\bm x_t,p_t)
            \le
            \frac{L_{\mathrm{Rev}}}{\sqrt T}
            +
            8BD\,2^{-s_t}.
        \end{equation}
        \end{lemma}

        Lemma~\ref{lem:ldp_revenue_gap_discrete_fixed_residual} reduces the
        regret over LDP rounds to the weighted layer occupancy
        $
            \sum_{s=1}^S 2^{-s}|\Psi_{T+1,s}|.
        $
        Here \(\Psi_{T+1,s}:=\bigcup_{j=1}^N \Psi_{T+1,s}^j\) denotes the
        set of all LDP rounds assigned to layer \(s\).
        It therefore remains to control the number of rounds terminating at each
        layer. The required auxiliary results are established in the appendix.
        Lemma~\ref{lem:bound_r_layered} bounds the cumulative revenue uncertainty
        within a layer by applying a sequential elliptical-potential argument
        within each permanently labeled residual bin and then aggregating across
        bins; the resulting bound separates statistical uncertainty from the
        approximation error induced by the local-polynomial remainder and the
        residual second-order pilot-index error. Lemma~\ref{lem:bound_Psi_layered}
        then exploits the fact that every round terminating at a nonterminal layer
        \(s<S\) has revenue width exceeding \(BD2^{-s}\) to convert the
        cumulative-width bound into an occupancy bound whenever the layer
        resolution dominates the approximation floor
        \(\mathcal O(h^\beta+\eta^2)\). These results provide the
        ingredients needed to aggregate the layer-dependent losses.

        These ingredients lead to a natural three-regime decomposition of the
        LDP regret. For nonterminal layers whose resolution remains above the
        approximation floor, the occupancy bounds apply, and the resulting
        contribution is governed by statistical uncertainty. Once the resolution
        falls to the approximation scale, a separate occupancy bound is no longer
        needed, because the per-round loss at these finer layers can be charged
        directly to the approximation error. Finally, each round reaching the
        terminal layer contributes at most its prescribed resolution \(2^{-S}\).
        Summing the contributions from these three regimes yields the total regret
        incurred over the LDP rounds.

        \begin{proposition}[Regret of the discretized LDP step]
        \label{prop:ldp_step_regret_discrete_fixed_residual}
         On the uniform
        confidence event in Proposition~\ref{prop:ucb},
        we have
        \begin{equation}
        \label{eq:ldp_step_regret_with_S}
        \begin{aligned}
        \sum_{t\in[T]\setminus\mathcal T^{\rm exp}}
        \left[
            \mathsf{Rev}(\bm x_t,p_t^\star)
            -
            \mathsf{Rev}(\bm x_t,p_t)
        \right]                                                
        \le&
        64C_\omega B
        \bigl(D\sqrt{\iota_T}+\sqrt{\lambda}\bigr)
        \sqrt{NT\iota_T}\\
        &+
        16C_\omega BT
        (h^\beta+\eta^2)\sqrt{\iota_T}
        +
        \bigl(8BD+L_{\mathrm{Rev}}\bigr)\sqrt T.
        \end{aligned}
        \end{equation}
        Suppose that \(\lambda>0\) is fixed independently of \(T\), \(\eta^2\le c_\eta h^\beta\) for some fixed constant \(c_\eta>0\) in addition to 
        \(\eta\le h\), \(N=\left\lceil
                T^{\frac{1}{2\beta+1}}
            \right\rceil\) and \(
            h=2B/N\),
        then there exists
        a finite constant $C_{\mathrm{LDP}} > 0$ such that
        \begin{equation}
        \label{eq:ldp_step_regret_exact_rate}
        \begin{aligned}
        &\sum_{t\in[T]\setminus\mathcal T^{\rm exp}}
        \left[
            \mathsf{Rev}(\bm x_t,p_t^\star)
            -
            \mathsf{Rev}(\bm x_t,p_t)
        \right]                                                  \le
        C_{\mathrm{LDP}}\,
        \iota_T
        T^{\frac{\beta+1}{2\beta+1}}.
        \end{aligned}
        \end{equation}
        \end{proposition}
        
        It remains to control the pilot-exploration rounds. Such a round is
        triggered only when the leverage score of the current covariate exceeds
        the target index accuracy. Each trigger produces a multiplicative
        increase in the determinant of the pilot design matrix. Since the
        determinant is also bounded above by the trace of that matrix, only a
        limited number of pilot-exploration rounds can occur. This argument is
        pathwise and does not consume any additional failure probability.

        \begin{lemma}[Pilot-exploration count and regret]
        \label{lem:pilot_exploration_regret}
        Suppose that Algorithm~\ref{alg:adaptive_pilot} is implemented with
        \(\eta>0\). Then the number of pilot-exploration rounds satisfies,
        pathwise,
        \begin{equation}
        \label{eq:pilot-exploration-count}
        \begin{aligned}
            \left|\mathcal T^{\rm exp}\right|
            &\le
            \min\left\{
                T,\,
                \frac{
                    d\log\left(
                        1+TC_x^2/d
                    \right)
                }{
                    \log\left(
                        1+\eta^2/\gamma_T^2
                    \right)
                }
            \right\} \le
            \min\left\{
                T,\,
                d\left(
                    1+\frac{\gamma_T^2}{\eta^2}
                \right)
                \log\left(
                    1+\frac{TC_x^2}{d}
                \right)
            \right\}.
        \end{aligned}
        \end{equation}
        Consequently, the regret incurred on pilot-exploration rounds satisfies
        \begin{equation}
        \label{eq:pilot-exploration-regret}
        \begin{aligned}
            &\sum_{t\in\mathcal T^{\rm exp}}
            \left[
                \mathsf{Rev}(\bm x_t,p_t^\star)
                -
                \mathsf{Rev}(\bm x_t,p_t)
            \right]  \le
            BD
            \min\left\{
                T,\,
                d\left(
                    1+\frac{\gamma_T^2}{\eta^2}
                \right)
                \log\left(
                    1+\frac{TC_x^2}{d}
                \right)
            \right\}.
        \end{aligned}
        \end{equation}
        \end{lemma}
        
        The regret decomposition is now complete. The
        pilot-exploration bound in
        Lemma~\ref{lem:pilot_exploration_regret} holds pathwise, whereas the
        LDP bound in
        Proposition~\ref{prop:ldp_step_regret_discrete_fixed_residual} holds
        on the uniform confidence event. Hence, their sum holds on the same
        event, without an additional union bound. The following theorem first
        states the resulting finite-sample guarantee and then specializes it
        to compatible choices of the residual-bin width and pilot accuracy.

        \begin{theorem}[Regret upper bound]
        \label{thm:regret_upper_bound}
        Suppose Assumptions~\ref{ass:bounded},
        \ref{ass:quantity-response}, and \ref{ass:g_smooth} hold.
        Let \(0<\eta\le h\), and 
        \(\gamma_T, \iota_T\) be defined as in
        Algorithm~\ref{alg:adaptive_pilot} and \eqref{eq:iota_T}, respectively. Then, with probability at least \(1-\delta\),
        \begin{equation}
        \label{eq:regret_bound_discrete_fixed_pilot_residual}
        \begin{aligned}
            \mathrm{Reg}(T)
            \le{}&
            BD
            \min\left\{
                T,\,
                \frac{
                    d\log\left(
                        1+TC_x^2/d
                    \right)
                }{
                    \log\left(
                        1+\eta^2/\gamma_T^2
                    \right)
                }
            \right\}+
            64C_\omega B
            \bigl(D\sqrt{\iota_T}+\sqrt{\lambda}\bigr)
            \sqrt{NT\iota_T}                                             \\
            &+
            16C_\omega BT
            (h^\beta+\eta^2)\sqrt{\iota_T}
            +
            \bigl(8BD+L_{\mathrm{Rev}}\bigr)\sqrt T,
        \end{aligned}
        \end{equation}
        where \(C_\omega\) and \(L_{\mathrm{Rev}}\) are the constants in
        Lemma~\ref{lem:bound_r_layered} and
        Lemma~\ref{lem:ldp_revenue_gap_discrete_fixed_residual}.
        
        Moreover, let \(N=
            \left\lceil
                T^{\frac{1}{2\beta+1}}
            \right\rceil\) and \(\eta^2=\min\{h^2,h^\beta\}\).
        Then with probability at least \(1-\delta\), there exists a finite constant \(C>0\) such that
        \begin{equation}
        \label{eq:regret_bound_discrete_fixed_pilot_residual_rate}
        \begin{aligned}
            \mathrm{Reg}(T)
            \le
            C
            \left[
                d\bigl(1+\gamma_T^2\bigr)
                \log\left(
                    1+\frac{TC_x^2}{d}
                \right)
                +
                (1+\sqrt{\lambda})\iota_T
            \right]
            T^{\frac{\beta+1}{2\beta+1}}.
        \end{aligned}
        \end{equation}

        \end{theorem}

        The theorem separates the main sources of regret. Pilot exploration
        pays for certifying the valuation index along the covariate directions
        encountered by the policy; its contribution is controlled by the
        uncertainty-triggered determinant argument. On LDP rounds,
        \(\sqrt{NT}\) is the statistical cost of learning across the \(N\)
        residual bins, whereas \(Th^\beta\) is the local-polynomial
        approximation cost. The pilot-corrected lifting leaves only the
        second-order contribution \(T\eta^2\), so choosing
        \(\eta^2\le h^\beta\) prevents pilot-index error from worsening the
        nonparametric rate. The residual-grid discretization and terminal layer
        contribute only order \(\sqrt T\). Choosing \(N\asymp T^{\frac{1}{2\beta+1}}\), \(h\asymp T^{-\frac{1}{2\beta+1}}\) and \(\eta^2\le h^\beta\)
        balances the statistical and approximation terms. For fixed \(d\) and
        fixed problem primitives, this yields the regret on the order of 
        \(
            \widetilde{\mathcal O}\!\left(
                T^{\frac{\beta+1}{2\beta+1}}
            \right)
        \).

        \begin{remark}[Adaptation to an unknown time horizon]
        The knowledge of \(T\) can be removed using the
        standard doubling trick, as commonly adopted in the dynamic-pricing
        literature
        \citep{DP_parametricF,d_free_DP,DP_Fm,dp_tight_2,CDP_smooth}.
        Specifically, the policy is restarted over epochs of geometrically
        increasing lengths, with all horizon-dependent parameters, including \(N\),
        calibrated to the current epoch length. 
        \end{remark}

        \subsection{Lower Bound}
        \label{subsec:lower_bound}
        
        We next establish a minimax lower bound that matches the horizon
        dependence of the regret upper bound in
        Theorem~\ref{thm:regret_upper_bound}, up to logarithmic factors.
        The construction uses a constant-covariate, binary-demand subclass.
        Hence, the lower bound isolates the difficulty of learning an unknown
        smooth demand function and does not rely on estimating the contextual
        parameter.
        
        \paragraph{A hard binary-demand subclass.}
        Consider the admissible demand subclass
        \begin{equation}
        \label{eq:lower-bound-demand-subclass}
            y_t
            =
            D\mathbf 1\{v_t\ge p_t\},
            \qquad
            q(z)
            =
            D\mathbf 1\{0\le z\le B\}.
        \end{equation}
        For every \(u\in\mathcal I_g\), this subclass gives
        $
            g(u)
            =
            \mathbb E[q(\epsilon_t-u)]
            =
            D\mathbb P(\epsilon_t\ge u).
        $
        Indeed, if \(u\in\mathcal I_g\) and \(\epsilon_t\ge u\), then
        $
            0
            \le
            \epsilon_t-u
            \le
            B_\epsilon-(-B+B_\epsilon)
            =
            B.
        $
        Thus, on the effective residual domain \(\mathcal I_g\), the
        normalized link \(g/D\) is the survival function of the valuation
        shock.
        
        Fix an admissible
        \(\bm\theta^\circ\in\Theta\) satisfying
        Assumption~\ref{ass:bounded} over \(\mathcal X\), and fix any
        \(\bm x^\circ\in\mathcal X\).
        Because deterministic covariate sequences are admissible, the minimax
        problem contains the subclass
        $
            \bm x_t\equiv\bm x^\circ,
            \bm\theta_\star=\bm\theta^\circ.
        $
        For every \(p\in[0,B]\),
        $
            p-(\bm x^\circ)^\top\bm\theta^\circ
            \in
            [-(\bm x^\circ)^\top\bm\theta^\circ,
              B-(\bm x^\circ)^\top\bm\theta^\circ]
            \subseteq
            \mathcal I_g,
        $
        and the corresponding expected revenue reduces to
        $
            \mathsf{Rev}(p)
            =
            p\,g\!\left(
                p-(\bm x^\circ)^\top\bm\theta^\circ
            \right).
        $
        
        \begin{theorem}[Minimax lower bound]
        \label{thm:lower_bound_survival}
        There exist finite constants
        \(\bar L_g>0\), \(c>0\), and \(T_0<\infty\)
        such that, for every \(L_g\ge\bar L_g\) and every \(T\ge T_0\),
        \begin{equation}
        \label{eq:minimax-lower-bound}
            \inf_{\pi}
            \sup_{\mathcal P(L_g)}
            \mathbb E^\pi\!\left[\mathrm{Reg}(T)\right]
            \ge
            cT^{\frac{\beta+1}{2\beta+1}}.
        \end{equation}
        Here, the infimum is over all nonanticipating, possibly randomized,
        pricing policies, and the supremum is over all instances in
        \(\mathcal P(L_g)\). The class \(\mathcal P(L_g)\) contains all
        instances satisfying Assumptions~\ref{ass:bounded},
        \ref{ass:quantity-response}, and \ref{ass:g_smooth} with H\"older
        constant at most \(L_g\), together with all admissible covariate
        sequences. The expectation is taken over both the demand randomness
        and randomization of the pricing policy.
        \end{theorem}

        \begin{corollary}[Expected-regret upper bound and minimax horizon rate]
        \label{cor:minimax-matching-rate}
        Fix \(L_g\ge \bar L_g\), the dimension, and the problem primitives,
        and suppose the confidence
        constants used by the policy are chosen uniformly over
        \(\mathcal P(L_g)\). For all sufficiently large \(T\), there are finite
        constants \(0< \bar{c}\le \bar{C}<\infty\), independent of \(T\), such that
        \begin{equation}
        \label{eq:minimax-matching-rate}
            \bar{c}T^{\frac{\beta+1}{2\beta+1}}
            \le
            \inf_\pi\sup_{\mathcal P(L_g)}
            \mathbb E^\pi[\mathrm{Reg}(T)]
            \le
            \bar{C}(\log T)^2T^{\frac{\beta+1}{2\beta+1}}.
        \end{equation}
        In particular, the upper and lower bounds match in their horizon
        exponent.
        \end{corollary}

        \paragraph{Construction overview.}
        The main difficulty is to construct a hard family that simultaneously
        satisfies the smoothness, monotonicity, bounded-support, and zero-mean
        requirements imposed on the valuation shock. We first construct a
        smooth baseline link whose expected revenue is constant over a
        nondegenerate interior price interval. We then place a smooth local
        perturbation in one of \(K_T\) disjoint subintervals of this flat
        region. Each perturbation creates a hidden revenue improvement, while
        a smaller adjustment away from the flat region preserves the
        zero-mean condition.
        
        The following lemma provides the required baseline link and the
        function used for the compensating adjustment.
        
        \begin{lemma}[A flat zero-mean baseline]
        \label{lem:lower_bound_flat_baseline}
        There exist constants
        $
            0<p_L<p_U<B,
        $
        a nonnegative infinitely differentiable function
        \(\chi:\mathbb R\to[0,\infty)\), and a nonincreasing function
        \(g_0:\mathbb R\to[0,D]\) satisfying the following properties:
        \begin{enumerate}
            \item
            The function \(g_0\) is infinitely differentiable,
            \(g_0(u)=D\) for \(u\le-B_\epsilon\), and
            \(g_0(u)=0\) for \(u\ge B_\epsilon\).
        
            \item
            The function \(g_0\) belongs to \(\mathcal H(\beta)\) on
            \([-B,B]\), with a finite H\"older constant.
        
            \item
            The normalized function \(g_0/D\) is the survival function of a
            zero-mean distribution supported on
            \([-B_\epsilon,B_\epsilon]\). In view of the monotonicity and
            endpoint conditions above, its zero-mean property is equivalently
            expressed as
            $
                \int_{-B_\epsilon}^{B_\epsilon}
                g_0(u)\,\mathrm du
                =
                DB_\epsilon.
            $
        
            \item
            The baseline revenue satisfies
            \begin{equation}
            \label{eq:lower-bound-flat-revenue}
                p\,g_0\!\left(
                    p-(\bm x^\circ)^\top\bm\theta^\circ
                \right)
                \le
                D\left(
                    (\bm x^\circ)^\top\bm\theta^\circ
                    -\frac{B_\epsilon}{2}
                \right),
                \qquad
                p\in[0,B],
            \end{equation}
            with equality for every \(p\in[p_L,p_U]\).
        
            \item
            The function \(\chi\) satisfies
            $
                \operatorname{supp}(\chi)
                \subset
                (-B_\epsilon,B_\epsilon)
                \setminus
                [p_L-(\bm x^\circ)^\top\bm\theta^\circ,
                 p_U-(\bm x^\circ)^\top\bm\theta^\circ],
                \int_{\mathbb R}\chi(u)\,\mathrm du
                =
                1.
            $
            On neighborhoods of both
            $
                [p_L-(\bm x^\circ)^\top\bm\theta^\circ,
                 p_U-(\bm x^\circ)^\top\bm\theta^\circ]
            $
            and \(\operatorname{supp}(\chi)\), the function \(g_0\) is
            bounded away from \(0\) and \(D\), and \(g_0'\) is uniformly
            negative.
        \end{enumerate}
        \end{lemma}
        
        Starting from Lemma~\ref{lem:lower_bound_flat_baseline}, partition the
        flat revenue interval into \(K_T\) disjoint price regions. For the
        \(k\)th region, add a smooth local perturbation \(\Delta_k\) and
        subtract the compensating term \(a_k\chi\), where
        $
            a_k
            =
            \int_{\mathbb R}\Delta_k(v)\,\mathrm dv.
        $
        Thus, the perturbed link is of the form
        $
            g_0(u)+\Delta_k(u)-a_k\chi(u).
        $
        Because \(\int_{\mathbb R}\chi(u)\,\mathrm du=1\),
        $
            \int_{\mathbb R}
            \left(
                \Delta_k(u)-a_k\chi(u)
            \right)\mathrm du
            =
            0.
        $
        Consequently, the perturbation does not alter the integral condition
        corresponding to the zero-mean valuation shock.
        
        The local perturbation has width \(\mathcal{O}(K_T^{-1})\) and raises the
        revenue in its associated region by order \(K_T^{-\beta}\). Its
        integral therefore satisfies
        $
            a_k
            =
            \mathcal{O}\!\left(K_T^{-(\beta+1)}\right).
        $
        Hence, the compensating adjustment is smaller than the local
        perturbation by a factor of order \(K_T^{-1}\). By placing its support
        away from the flat region and choosing the perturbation magnitude
        sufficiently small, all perturbed links remain nonincreasing, take
        values in \([0,D]\), and belong to a common H\"older ball.
        
        \begin{figure*}[t]
            \centering
            \includegraphics[
                width=0.88\textwidth
            ]{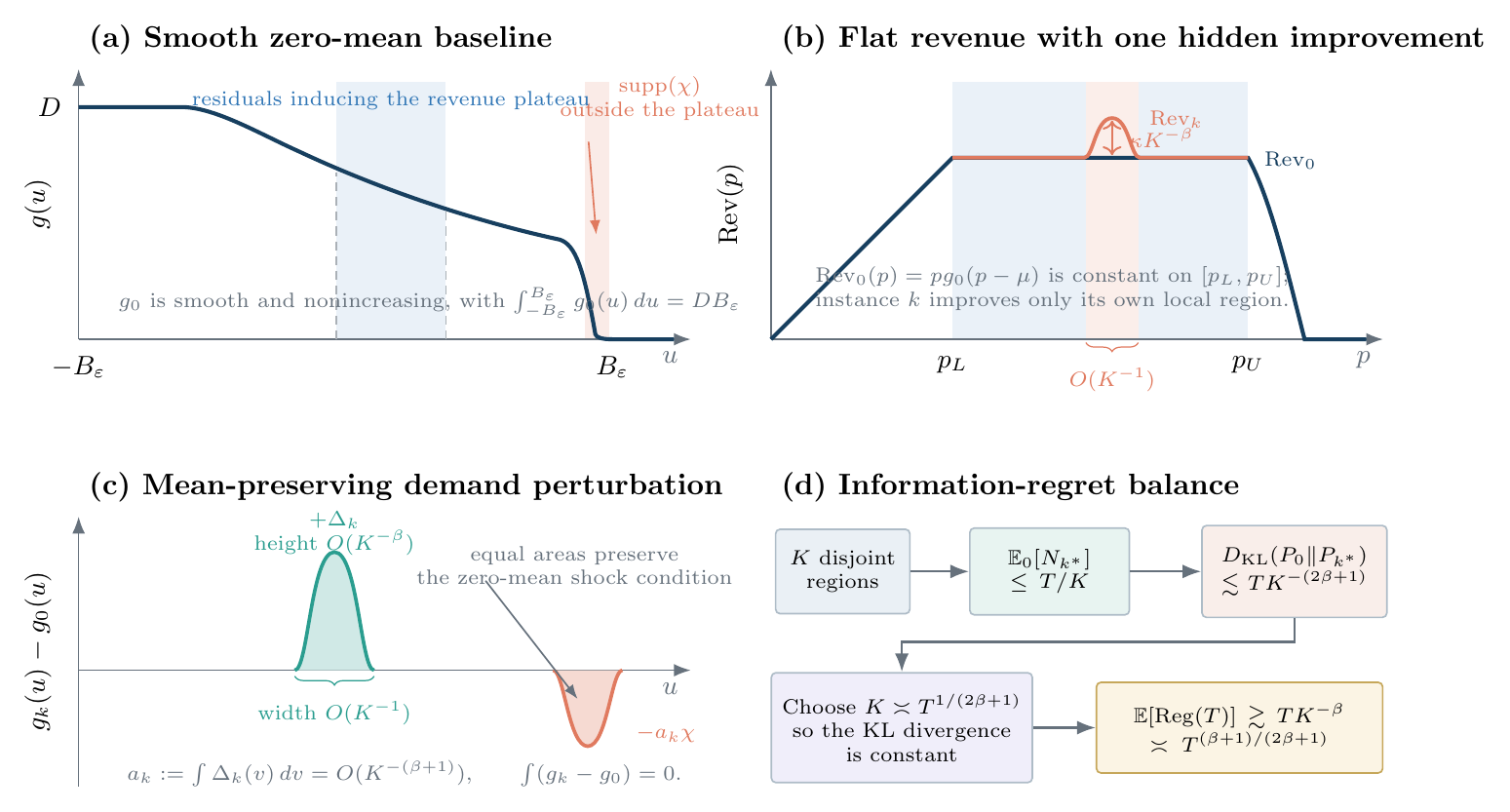}
            \caption{
            Illustration of the lower-bound construction.
            Panel (a) constructs a smooth, nonincreasing baseline demand
            function \(g_0\) satisfying the zero-mean integral condition.
            Panel (b) shows the induced flat revenue region and a local
            perturbation of width \(\mathcal{O}(K_T^{-1})\) and height
            \(\kappa K_T^{-\beta}\).
            Panel (c) adds the compensating term
            \(-a_k\chi\), where
            \(a_k=\int \Delta_k(v)\ \mathrm{d}v
            =\mathcal{O}(K_T^{-(\beta+1)})\), so that the perturbation preserves the
            zero-mean shock condition.
            Panel (d) summarizes the information--regret balance:
            choosing
            \(K_T\asymp T^{1/(2\beta+1)}\)
            keeps the KL divergence bounded and yields the regret lower bound
            \(\Omega\!\left(
                T^{(\beta+1)/(2\beta+1)}
            \right)\).
            }
            \label{fig:lower-bound-construction}
        \end{figure*}
        
        \paragraph{Information--regret balance.}
        Under the baseline instance, the expected total number of visits to all \(K_T\) disjoint perturbation regions is at most \(T\); hence some region is visited at most \(T/K_T\) times in expectation. 
        In that region, the perturbation changes the Bernoulli success probability by \(\Theta(K_T^{-\beta})\), 
        producing a per-observation KL divergence of order \(K_T^{-2\beta}\). 
        The local KL contribution is therefore \(\mathcal{O}\bigl(\frac{T}{K_T}K_T^{-2\beta}\bigr)=\mathcal{O}(T K_T^{-(2\beta+1)})\). 
        The compensating adjustment, of magnitude \(\mathcal{O}(K_T^{-(\beta+1)})\), contributes at most \(\mathcal{O}(T K_T^{-2(\beta+1)})\) to the KL divergence, 
        which is of lower order. Hence the total KL divergence remains bounded when \(K_T\asymp T^{1/(2\beta+1)}\). 
        A standard testing argument then shows that with probability bounded away from zero, the policy fails to identify the perturbed region for a constant fraction of the horizon. 
        Since any price outside that region incurs a revenue loss of order \(K_T^{-\beta}\), the expected regret is at least of order \(T K_T^{-\beta}\asymp T^{\frac{\beta+1}{2\beta+1}}\), proving Theorem~\ref{thm:lower_bound_survival}.
        Therefore, for fixed problem primitives, the regret rate
        $
            \widetilde{\mathcal O}\!\left(
                T^{\frac{\beta+1}{2\beta+1}}
            \right)
        $
        is minimax optimal in its dependence on the horizon.

        \begin{remark}[The effect of strong unimodality]
            \label{rem:strong-unimodality-lower-bound}
            The hard family of instances underlying Theorem~\ref{thm:lower_bound_survival}
            is based on a revenue function that is flat over a nondegenerate price interval.
            It therefore lies outside any class that satisfies the strong unimodality condition.
            The strong unimodality condition, however, rules out such flat alternatives and
            ensures that price errors translate into quadratic revenue losses.
            For \(\beta\ge2\), \cite{CDP_oracle_price_map} obtain the matching horizon exponent
            \(
                \frac{2\beta-1}{4\beta-3}.
            \)
            Thus, strong unimodality fundamentally changes the learning difficulty and
            permits higher-order smoothness to be exploited more effectively.
        \end{remark}

\section{Conclusion and Future Research}
\label{sec:conclusion}

This paper studies contextual dynamic pricing with an unknown linear
valuation component and an unknown nonparametric demand response. The
model accommodates bounded discrete or continuous purchase quantities
and imposes only H\"older smoothness on the induced demand function.
In particular, the revenue function may be multimodal and its maximizer
need not be unique. The framework therefore applies beyond the binary
purchase model and avoids the strong shape restrictions commonly used
to make contextual pricing analytically tractable.

To address the interaction between parametric estimation, nonparametric
learning, and adaptive pricing, we develop a layered decision-partitioning
framework that combines directional pilot learning, pilot-corrected local
polynomial regression, and permanent layer--bin assignments. The
pilot correction reduces the remaining index error to second order,
whereas the permanent data partition preserves predictability and enables
concentration under adaptively collected observations. Together with a
global elimination procedure, these ingredients yield the regret bound
$
    \Reg(T)
    =
    \widetilde{\mathcal O}\!\left(
        T^{\frac{\beta+1}{2\beta+1}}
    \right),
$
which matches our lower bound in its dependence on the horizon.

Two directions appear particularly promising for future research.
First, it would be valuable to develop a policy that adapts simultaneously
to the unknown H\"older radius \(L_g\) and smoothness order \(\beta\),
while preserving predictable data assignment and valid confidence
guarantees under adaptive sampling. An accompanying question is whether
such simultaneous adaptation necessarily incurs an additional statistical
cost. Second, although the present framework accommodates general,
possibly multimodal revenue functions, it does not exploit favorable local
revenue geometry when such structure is present. Developing a
shape-adaptive policy that achieves faster regret under quadratic revenue
growth or strong unimodality, while retaining the minimax guarantee over
the unrestricted revenue class, would yield a desirable best-of-both-worlds
result.

\newpage

\bibliographystyle{apalike}

\bibliography{ref}

\newpage
\appendix

\section{Proofs}

\paragraph{Proof roadmap.}
The appendix follows the logical dependency of the results. We first justify
the induced link and the uniform-price pilot identity. We then establish pilot
confidence, the pilot-corrected local approximation, and the uniform
layer--bin confidence event. These ingredients feed the one-period gap and
layer-counting arguments used in the contextual upper bound. The lower-bound
construction is proved next. 

\subsection{Proofs for the Model}

\noindent
\textbf{Proof of \cref{prop:induced_F}.}
\begin{myproof}
  Let \(\mathcal I_g = [-B + B_\epsilon, B - B_\epsilon]\).
  Since \(q:\mathbb R\to[0,D]\), the function
  \(g(u)=\mathbb E\!\left[q(\epsilon_t-u)\right]\)
  is well-defined and satisfies \(0\le g(u)\le D\) for all \(u\in\mathcal I_g\).
  Let \(\bar F(u):=\mathbb P(\epsilon_t\ge u)\) be the survival function of the noise.
  Since \(F\in\mathcal H(\beta,L_F;[-B,B])\) with \(\beta\ge1\), \(F\) is continuous.
  Hence \(\epsilon_t\) has no atoms.

  By Assumption~\ref{ass:bounded}, the support of \(\epsilon_t\) is contained in
  \([-B_\epsilon,B_\epsilon]\). Therefore \(\bar F\) is constant on
  \((-\infty,-B_\epsilon]\) and on \([B_\epsilon,\infty)\). Since
  \(B_\epsilon<B\), the function \(1-F\) is already constant on
  \([-B,-B_\epsilon]\) and on \([B_\epsilon,B]\). Thus, when viewed on the
  interval \([-B+B_\epsilon,2B-B_\epsilon]\),
  the survival function \(\bar F\) belongs to
  \(\mathcal H(\beta,L_F;[-B+B_\epsilon,2B-B_\epsilon])\).
  
  Because \(q\) has bounded variation on \([0,B]\), it has at most countably many discontinuities.
  Let \(q^+\) be a right-continuous representative of \(q\) on
\([0,B]\) with
\(\operatorname{TV}_{[0,B]}(q^+)\le V_q\), extend it by zero outside
\([0,B]\), and continue to denote the resulting function by \(q^+\).
  Since \(\epsilon_t\) has no atoms, for each fixed \(u\in\mathcal I_g\),
  $
  \mathbb E\!\left[q(\epsilon_t-u)\right]
  =
  \mathbb E\!\left[q^+(\epsilon_t-u)\right].
  $
  Hence \(g\) can be represented using \(q^+\). By the Lebesgue--Stieltjes representation for bounded-variation functions,
  there exists a finite signed measure \(\mu_q\) on \((0,B]\) such that
  \[
  q^+(w)=q^+(0)+\mu_q((0,w]), \qquad w\in[0,B],
  \]
  and moreover,
  \[
  |\mu_q|((0,B])\le V_q, \qquad |q^+(0)|\le D.
  \]
  
  For \(u\in\mathcal I_g\), Assumption~\ref{ass:bounded} implies that
  \(\epsilon_t-u\le B\) almost surely whenever \(\epsilon_t-u\ge0\). Therefore,
  \[
   q^+(\epsilon_t-u)
  =
  q^+(0)\mathbf 1\{\epsilon_t\ge u\}
  +
  \int_{(0,B]}
  \mathbf 1\{\epsilon_t\ge u+a\}\,\mu_q(\mathrm{d} a).
  \]
  Taking expectations and applying Fubini's theorem for finite signed measures gives
  \[
  g(u)=q^+(0)\bar F(u)+\int_{(0,B]}\bar F(u+a)\,\mu_q(\mathrm{d} a),
  \qquad \forall u\in\mathcal I_g .
  \]
  
  Let \(m:=\varpi(\beta)\). Since \(\mu_q\) is finite and
  \(\bar F\in\mathcal H(\beta,L_F;[-B+B_\epsilon,2B-B_\epsilon])\),
  differentiating under the finite signed measure yields
  \[
  g^{(m)}(u)
  =
  q^+(0)\bar F^{(m)}(u)
  +
  \int_{(0,B]} \bar F^{(m)}(u+a)\,\mu_q(\mathrm{d} a).
  \]
  
  Thus, for any \(u,u'\in\mathcal I_g\),
  \[
  \begin{aligned}
  |g^{(m)}(u)-g^{(m)}(u')|
  &\le |q^+(0)|\,|\bar F^{(m)}(u)-\bar F^{(m)}(u')| \\
  &\quad + \int_{(0,B]} |\bar F^{(m)}(u+a)-\bar F^{(m)}(u'+a)|\,|\mu_q|(\mathrm{d} a) \\
  &\le (|q^+(0)|+|\mu_q|((0,B]))\,L_F|u-u'|^{\beta-m} \\
  &\le (D+V_q)L_F|u-u'|^{\beta-m}.
  \end{aligned}
  \]
  
  Therefore \(g\in\mathcal H(\beta,L_g;\mathcal I_g)\) with \(L_g\le (D+V_q)L_F\).
  The proof is complete.
  \end{myproof}

\subsection{Proofs for the Contextual Policy}

\noindent
\textbf{Proof of \cref{lem:pilot-confidence}.}
\begin{myproof}
  On an exploratory round, \( p_i \sim \mathrm{Unif}[0,B] \) is drawn independently of \( \epsilon_i \) conditional on \( \mathcal F_{i-1}\vee\sigma(\bm x_i) \). 
  Since \( \mathbf 1\{y_i>0\}=\mathbf 1\{v_i\ge p_i\} \) and \( v_i\in[0,B] \), we have
  \[
  \mathbb E\!\left[
      B\mathbf 1\{y_i>0\}\mid \mathcal F_{i-1},\bm x_i,\{i\in\mathcal T^{\rm exp}\}
  \right]
  =
  \mathbb E[v_i\mid \mathcal F_{i-1},\bm x_i]
  =
  \bm x_i^\top\bm\theta_\star.
  \]
  Thus, \( B\mathbf 1\{y_i>0\} \) is an unbiased measurement of \( \bm x_i^\top\bm\theta_\star \) on pilot rounds. 
  The observation error is bounded by \( B \) in absolute value, and \( \mathbf 1\{i\in\mathcal T^{\rm exp}\}\bm x_i \) is determined before \( y_i \) is revealed. 
  By the standard self-normalized concentration inequality for adaptive linear regression, with probability at least \( 1-\delta/2 \), simultaneously for all \( t\le T \),
  \[
  \left\|
      \sum_{i<t, i\in\mathcal T^{\rm exp}}
      \bm x_i
      \left(
          B\mathbf 1\{y_i>0\}
          -
          \bm x_i^\top\bm\theta_\star
      \right)
  \right\|_{M_t^{-1}}
  \le
  B\sqrt{
      \log\det(M_t)
      +
      2\log\frac{4}{\delta}
  }.
  \]
  
  The update rules give
  \[
  M_t=I_d+\sum_{i<t, i\in\mathcal T^{\rm exp}}\bm x_i\bm x_i^\top,\qquad
  \bm b_t=\sum_{i<t, i\in\mathcal T^{\rm exp}}B\mathbf 1\{y_i>0\}\bm x_i.
  \]
  Since \( \|\bm x_i\|_2\le C_x \), standard elliptical-potential arguments yield
  \[
  \log\det(M_t)\le d\log\left(1+\frac{TC_x^2}{d}\right).
  \]

Since \(\widehat{\bm\theta}_t
  =
  M_t^{-1}\bm b_t\),
we have
\begin{align}
  \widehat{\bm\theta}_t-\bm\theta_\star
  &=
  M_t^{-1}\bm b_t
  -
  M_t^{-1}M_t\bm\theta_\star =
  M_t^{-1}
  \left(
      \bm b_t-M_t\bm\theta_\star
  \right).
\label{eq:pilot-error-first-decomposition}
\end{align}
Moreover, 
\begin{align}
  M_t\bm\theta_\star
  &=
  \left(
      I_d
      +
      \sum_{\substack{i<t, i\in \mathcal{T}^{\rm exp}}}
      \bm x_i\bm x_i^\top
  \right)
  \bm\theta_\star =
  \bm\theta_\star
  +
  \sum_{\substack{i<t, i\in \mathcal{T}^{\rm exp}}}
  \bm x_i
  \left(
      \bm x_i^\top\bm\theta_\star
  \right).
\label{eq:pilot-design-times-true-parameter}
\end{align}
Substituting
\eqref{eq:pilot-design-times-true-parameter}
into \eqref{eq:pilot-error-first-decomposition} yields
\begin{align*}
  \widehat{\bm\theta}_t-\bm\theta_\star
  &=
  M_t^{-1}
  \left[
      \sum_{i<t, i\in \mathcal{T}^{\rm exp}}
      B\mathbf 1\{y_i>0\}\bm x_i
      -
      \bm\theta_\star
      -
      \sum_{i<t, i\in \mathcal{T}^{\rm exp}}
      \bm x_i
      \left(
          \bm x_i^\top\bm\theta_\star
      \right)
  \right] \notag\\
  &=
  M_t^{-1}
  \left[
      \sum_{i<t, i\in \mathcal{T}^{\rm exp}}
      \bm x_i
      \left(
          B\mathbf 1\{y_i>0\}
          -
          \bm x_i^\top\bm\theta_\star
      \right)
      -
      \bm\theta_\star
  \right].
\end{align*}

  For any \( \bm x\in\mathbb R^d \), the Cauchy--Schwarz inequality in the \( M_t^{-1} \)-norm gives
  \[
  \begin{aligned}
  \left|
      \bm x^\top
      \bigl(
          \widehat{\bm\theta}_t-\bm\theta_\star
      \bigr)
  \right|
  &\le
  \|\bm x\|_{M_t^{-1}}
  \left[
      \left\|
          \sum_{i<t, i\in\mathcal T^{\rm exp}}
          \bm x_i
          \left(
              B\mathbf 1\{y_i>0\}
              -
              \bm x_i^\top\bm\theta_\star
          \right)
      \right\|_{M_t^{-1}}
      +
      \|\bm\theta_\star\|_{M_t^{-1}}
  \right] \\
  &\le
  \|\bm x\|_{M_t^{-1}}
  \left[
      B\sqrt{
          d\log\left(1+\frac{TC_x^2}{d}\right)
          +
          2\log\frac{4}{\delta}
      }
      +
      C_\theta
  \right],
  \end{aligned}
  \]
  where the second inequality uses
  $ 
      \|\bm\theta_\star\|_{M_t^{-1}}
  \le
  \|\bm\theta_\star\|_2
  \le
  C_\theta,$ and \( M_t^{-1}\preceq I_d \). 
  The bracketed expression is precisely \( \gamma_T \), completing the proof.
  \end{myproof}

\noindent
The proof of Proposition~\ref{prop:ucb} uses the following martingale
property.

  \begin{proposition}[Martingale property]
    \label{prop:mds_piloted_lpr}
    For notational convenience, set \(s_i=j_i=w_i=0\) whenever
    \(i\in\mathcal T^{\rm exp}\). Use the chronological fields defined in
    Section~\ref{sec:setting} and write
    \[
        \mathcal F_i^-:=\mathcal G_i,
        \qquad
        \mathcal F_i=\mathcal F_i^-\vee\sigma(y_i),
        \qquad i\in[T].
    \]
    In particular, the exploration decision, price, stopping layer, bin, and
    residual label are all \(\mathcal F_i^-\)-measurable.
    For each \(i\in[T]\), define
    \[
        \varepsilon_i
        :=
        y_i
        -
        g\!\left(
            p_i-\bm x_i^\top\bm\theta_\star
        \right).
    \]
    Then
    \[
        \mathbb E\!\left[
            \varepsilon_i
            \mid
            \mathcal F_i^-
        \right]
        =0,
        \qquad
        |\varepsilon_i|
        \le D
        \quad\text{a.s.}
    \]
    
    For fixed \(s\in[S]\) and \(j\in[N]\), define the zero-padded feature
    vector
    \[
        \overline\psi_{i,j,s}
        :=
        \begin{cases}
            \psi_{i,j}(w_i),
            &\text{if }
            i\notin\mathcal T^{\rm exp},\
            s_i=s,\
            j_i=j,\\[1mm]
            \bm 0,
            &\text{otherwise}.
        \end{cases}
    \]
    Then \(\overline\psi_{i,j,s}\) is
    \(\mathcal F_i^-\)-measurable, 
    \(
        \|\overline\psi_{i,j,s}\|_2
        \le C_\psi
    \)
    for some $  C_\psi > 0, a.s.$, and 
    $
    \label{eq:zero-padded-martingale-transform}
        \mathbb E\!\left[
            \overline\psi_{i,j,s}\varepsilon_i
            \mid
            \mathcal F_i^-
        \right]
        =
        \bm 0.
    $
    Consequently,
    \[
        \bm M_0^{s,j}:=\bm 0,
        \qquad
        \bm M_t^{s,j}
        :=
        \sum_{i=1}^t
            \overline\psi_{i,j,s}\varepsilon_i,
        \quad t\in[T],
    \]
    defines a vector-valued martingale with respect to
    \(\{\mathcal F_t\}_{t=0}^T\), whose increments satisfy
    \[
        \left\|
            \bm M_t^{s,j}-\bm M_{t-1}^{s,j}
        \right\|_2
        \le C_\psi D
        \quad\text{a.s.}
    \]
    In particular, for every \(t\in[T]\),
    \[
        \bm M_{t-1}^{s,j}
        =
        \sum_{i\in\Psi_{t,s}^j}
            \psi_{i,j}(w_i)\varepsilon_i.
    \]
    \end{proposition}

\begin{myproof}
    By policy nonanticipativity and the sequential exogeneity of the current
    valuation shock,
    $
        \mathbb P\!\left(
            \epsilon_i\le u
            \mid
            \mathcal F_i^-
        \right)
        =
        F(u).
    $
    Moreover, the surplus-demand model gives
    $
        \mathbb E\!\left[
            y_i
            \mid
            \mathcal F_i^-\vee\sigma(v_i)
        \right]
        =
        q(v_i-p_i).
    $
    Therefore, by the tower property,
    \[
        \mathbb E[y_i\mid\mathcal F_i^-]
        =
        \mathbb E\!\left[
            q(v_i-p_i)
            \mid
            \mathcal F_i^-
        \right]                                                   \\
        =
        \int
            q\!\left(
                \bm x_i^\top\bm\theta_\star+\epsilon-p_i
            \right)
            \,\mathrm dF(\epsilon)                                \\
        =
        g\!\left(
            p_i-\bm x_i^\top\bm\theta_\star
        \right).
    \]
    It follows that
    $
        \mathbb E\!\left[
            \varepsilon_i
            \mid
            \mathcal F_i^-
        \right]
        =0.
    $
    Since \(y_i\in[0,D]\) almost surely and \(g\in[0,D]\), we have
    $
        \varepsilon_i
        \in
        \left[
            -g\!\left(
                p_i-\bm x_i^\top\bm\theta_\star
            \right),
            D-g\!\left(
                p_i-\bm x_i^\top\bm\theta_\star
            \right)
        \right],
    $
    and hence
    $
        |\varepsilon_i|\le D.
    $
    
    Now fix \(s\in[S]\) and \(j\in[N]\). The exploration decision,
    permanent layer--bin labels, residual action, and posted price are all
    determined before \(y_i\) is observed. Hence
    \(\overline\psi_{i,j,s}\) is
    \(\mathcal F_i^-\)-measurable. Furthermore, whenever this feature is
    nonzero,
    \[
        w_i\in I_j,\qquad
        p_i-w_i=\widehat p_i(w_i)-w_i\in[0,B],
        \qquad
        \|\bm x_i\|_2\le C_x.
    \]
    Since the local-polynomial degree is fixed, the definition of
    \(\psi_{i,j}\) implies that there exists a deterministic constant
    \(C_\psi<\infty\) such that
    $
        \|\overline\psi_{i,j,s}\|_2
        \le C_\psi.
    $
    Consequently,
    \[
        \mathbb E\!\left[
            \overline\psi_{i,j,s}\varepsilon_i
            \mid
            \mathcal F_i^-
        \right]
        =
        \overline\psi_{i,j,s}
        \mathbb E\!\left[
            \varepsilon_i
            \mid
            \mathcal F_i^-
        \right]                                                   \\
        =
        \bm 0,
    \]
    which proves the martingale-difference property.
    
    To make the associated martingale explicit, define, within this proof,
    \[
        \bm M_0^{s,j}:=\bm 0,
        \qquad
        \bm M_t^{s,j}
        :=
        \sum_{i=1}^t
            \overline\psi_{i,j,s}\varepsilon_i,
        \quad t\in[T].
    \]
    The process \(\{\bm M_t^{s,j}\}_{t=0}^T\) is
    \(\{\mathcal F_t\}_{t=0}^T\)-adapted and integrable. Since
    \(\mathcal F_{t-1}\subseteq\mathcal F_t^-\), the tower property yields
    \[
        \mathbb E\!\left[
            \bm M_t^{s,j}
            \mid
            \mathcal F_{t-1}
        \right]
        =
        \bm M_{t-1}^{s,j}
        +
        \mathbb E\!\left[
            \mathbb E\!\left[
                \overline\psi_{t,j,s}\varepsilon_t
                \mid
                \mathcal F_t^-
            \right]
            \,\middle|\,
            \mathcal F_{t-1}
        \right]                                                   \\
        =
        \bm M_{t-1}^{s,j}.
    \]
    Thus, \(\{\bm M_t^{s,j},\mathcal F_t\}_{t=0}^T\) is a
    vector-valued martingale, interpreted componentwise, with increments
    satisfying
    \[
        \left\|
            \bm M_t^{s,j}-\bm M_{t-1}^{s,j}
        \right\|_2
        \le C_\psi D
        \qquad\text{a.s.}
    \]
    
    Finally, by the permanent-label definition,
    \[
        \Psi_{t,s}^j
        =
        \left\{
            i<t:
            i\notin\mathcal T^{\rm exp},\
            s_i=s,\
            j_i=j
        \right\},
    \]
    and therefore
    \[
        \bm M_{t-1}^{s,j}
        =
        \sum_{i\in\Psi_{t,s}^j}
            \psi_{i,j}(w_i)\varepsilon_i.
    \]
    This is precisely the stochastic score appearing in the layer--bin
    ridge estimator.
    \end{myproof}

\noindent
\textbf{Proof of \cref{prop:ucb}.}
    \begin{myproof}
      We proceed in five steps.

      \paragraph{Step 1: pilot-index control.}
      Let
      \[
          \mathcal E_{\rm pil}
          :=
          \left\{
              |\bm x^{\top}(\widehat{\bm\theta}_t-\bm\theta_\star)|
              \le
              \gamma_T\|\bm x\|_{M_t^{-1}}
              \text{ for every }t\le T\text{ and every }\bm x\in\mathbb R^d
          \right\}.
      \]
      Lemma~\ref{lem:pilot-confidence} gives
      \(
          \mathbb P(\mathcal E_{\rm pil})\ge1-\delta/2
      \).
      If \(i\) is a main-policy round, then the exploration test did not trigger,
      and therefore
      \(
          \gamma_T\|\bm x_i\|_{M_i^{-1}}\le\eta
      \).
      On \(\mathcal E_{\rm pil}\), we have
      \(
          |\bm x_i^{\top}(\widehat{\bm\theta}_i-\bm\theta_\star)|
          \le\eta
      \).
      Because Euclidean projection onto \([0,B]\) is nonexpansive and
      \(\bm x_i^{\top}\bm\theta_\star\in[0,B]\),
      \begin{equation}
      \label{eq:detailed-own-pilot-error}
          \left|
              \Pi_{[0,B]}(\bm x_i^{\top}\widehat{\bm\theta}_i)
              -\bm x_i^{\top}\bm\theta_\star
          \right|
          \le\eta.
      \end{equation}
      The same inequality holds for the current LDP round \(t\).
      
      \paragraph{Step 2: local Taylor expansion and joint linearization.}
      Since \(0\le g\le D\) on \(\mathcal I_g\),
      \(g\in\mathcal H(\beta,L_g;\mathcal I_g)\), and
      \(\mathcal I_g\) is a fixed nondegenerate compact interval, a standard
      one-dimensional interpolation inequality implies
      \[
          \max_{0\le k\le\varpi(\beta)}
          \sup_{u\in\mathcal I_g}
          |g^{(k)}(u)|
          \le C_g,
      \]
      where \(C_g\) is the fixed class-level envelope specified after
      Assumption~\ref{ass:g_smooth}.
      
      To make the Taylor coefficients at all bin endpoints well defined, extend
      \(g\) from
      \(\mathcal I_g=[-B+B_\epsilon,B-B_\epsilon]\)
      to \([-B,B]\) by its endpoint Taylor polynomials of order
      \(m:=\varpi(\beta)\), while continuing to use the same notation \(g\).
      
      The extension agrees with \(g\) and its derivatives up to order \(m\)
      at both endpoints of \(\mathcal I_g\), and is therefore \(m\)-times
      differentiable on \([-B,B]\). Moreover, \(g^{(m)}\) is constant on each
      added interval and equals the corresponding endpoint value. Hence,
      Assumption~\ref{ass:g_smooth} implies
      \[
          \left|
              g^{(m)}(u)-g^{(m)}(u')
          \right|
          \le
          L_g|u-u'|^{\beta-m},
          \quad
          u,u'\in[-B,B].
      \]
      Indeed, when \(u\) and \(u'\) lie in different regions, the original
      H\"older bound is applied to the corresponding endpoint or endpoints,
      whose separation is no greater than \(|u-u'|\). Therefore, the extension
      belongs to \(\mathcal H(\beta,L_g;[-B,B])\). After increasing the fixed value of \(C_g\), if necessary, the extension
      therefore satisfies
      \begin{equation}
      \label{eq:extended-derivative-bound}
          \max_{0\le k\le m}
          \sup_{u\in[-B,B]}
          |g^{(k)}(u)|
          \le
          C_g.
      \end{equation}

      This extension is used only in the analysis and leaves \(g\) unchanged
      on its original domain \(\mathcal I_g\); hence it does not alter the
      demand model or the algorithm.

      For the proof, use the convention \(g^{(0)}=g\) and define the
      coefficient vector corresponding to the joint feature. 
      Recall 
      \[
          \bm z_j=g(a_{j-1}).
      \]
      if \(\varpi(\beta)=0\), and 
      \[
          \bm z_j=
          \begin{pmatrix}
              g(a_{j-1})\\
              g'(a_{j-1})\\
              \vdots\\
              g^{(\varpi(\beta))}(a_{j-1})/\varpi(\beta)!\\[1mm]
              g'(a_{j-1})\bm\theta_\star\\
              \vdots\\
              \bigl[g^{(\varpi(\beta))}(a_{j-1})/\varpi(\beta)!\bigr]
              \bm\theta_\star
          \end{pmatrix}
      \]
      if \(\varpi(\beta)\ge1\).
      
      Define
      \begin{equation}
      \label{eq:C-z-definition}
          C_z
          :=
          C_g
          \left[
              \sum_{k=0}^{m}\frac{1}{(k!)^2}
              +
              C_\theta^2
              \sum_{k=1}^{m}\frac{1}{(k!)^2}
          \right]^{1/2}.
      \end{equation}
      Equation~\eqref{eq:extended-derivative-bound} and
      \(\|\bm\theta_\star\|_2\le C_\theta\) imply that, uniformly over \(j\in[N]\),
      \[
      \begin{aligned}
          \|\bm z_j\|_2^2
          &=
          \sum_{k=0}^{m}
          \left|
              \frac{g^{(k)}(a_{j-1})}{k!}
          \right|^2
          +
          \|\bm\theta_\star\|_2^2
          \sum_{k=1}^{m}
          \left|
              \frac{g^{(k)}(a_{j-1})}{k!}
          \right|^2\le
          C_g^2
          \left[
              \sum_{k=0}^{m}
              \frac{1}{(k!)^2}
              +
              C_\theta^2
              \sum_{k=1}^{m}
              \frac{1}{(k!)^2}
          \right]
          =
          C_z^2.
      \end{aligned}
      \]
      Since the right-hand side does not depend on \(j\), we obtain
      \begin{equation}
      \label{eq:detailed-z-bound}
          \sup_{j\in[N]}
          \|\bm z_j\|_2
          \le C_z.
      \end{equation}

       Define
      \begin{equation}
\label{eq:C-l-definition}
    C_l
    :=
    \max\left\{
        \frac{2^\beta L_g}{\varpi(\beta)!},
        \;
        C_g
        \sum_{k=2}^{\varpi(\beta)}
        \frac{
            (2^k-k-1)(2B)^{k-2}
        }{
            k!
        }
    \right\}.
\end{equation}

      We next establish the joint linearized representation on the event
      \(\mathcal E_{\rm pil}\). Fix \(i\in\Psi_{t,s}^j\). Since the bin label assigned to period \(i\)
      is permanent, \(w_i\in I_j\), and hence \(0\le w_i-a_{j-1}\le h\).
      Moreover, since \(p_i
          =
          \widehat p_i(w_i)
          =
          \Pi_{[0,B]}
          \left(
              \bm x_i^\top\widehat{\bm\theta}_i
          \right)
          +
          w_i\),
      \eqref{eq:detailed-own-pilot-error} gives that
      \begin{equation}
      \label{eq:detailed-pilot-shift}
          \left|
              (p_i-w_i)-\bm x_i^{\top}\bm\theta_\star
          \right|
          \le\eta.
      \end{equation}
      The true residual can therefore be written exactly as
      \begin{equation}
      \label{eq:detailed-true-residual}
      \begin{aligned}
          p_i-\bm x_i^{\top}\bm\theta_\star
          ={}&
          a_{j-1}
          +(w_i-a_{j-1})+\bigl((p_i-w_i)-\bm x_i^{\top}\bm\theta_\star\bigr).
      \end{aligned}
      \end{equation}
      Moreover, since \(p_i\in[0,B]\) and
      \(\bm x_i^\top\bm\theta_\star\in[B_\epsilon,B-B_\epsilon]\),
      \[
          p_i-\bm x_i^\top\bm\theta_\star
          \in
          [-B+B_\epsilon,B-B_\epsilon]
          =
          \mathcal I_g.
      \]
      Hence \(g(p_i-\bm x_i^\top\bm\theta_\star)\) is evaluated on the
      original domain of the induced demand link; the extension in Step~2 is
      used only to define the Taylor coefficients at \(a_{j-1}\). Thus, relative to the expansion point \(a_{j-1}\), the true residual
      contains two deviations: the within-bin deviation \(w_i-a_{j-1}\), whose
      absolute value is at most \(h\), and the pilot-index error
      \((p_i-w_i)-\bm x_i^\top\bm\theta_\star\), whose absolute value is at most
      \(\eta\).
      
      We next identify the function value represented by the joint feature.
      By the definitions of \(\psi_{i,j}(w_i)\), \(\phi_j\),
      \(\phi_j'\), \(\mathbf X_j\), and \(\bm z_j\), we have
      \begin{align}
      \label{eq:detailed-linearized-feature-identity}
          \psi_{i,j}(w_i)^\top\bm z_j
          ={}&
          \sum_{k=0}^{\varpi(\beta)}
          \frac{g^{(k)}(a_{j-1})}{k!}
          (w_i-a_{j-1})^k
        +
          \left[
              (p_i-w_i)
              -
              \bm x_i^\top\bm\theta_\star
          \right]
          \sum_{k=1}^{\varpi(\beta)}
          \frac{g^{(k)}(a_{j-1})}{(k-1)!}
          (w_i-a_{j-1})^{k-1}.
      \end{align}
      The first sum in
      \eqref{eq:detailed-linearized-feature-identity} is the local Taylor
      polynomial centered at \(a_{j-1}\) and evaluated at \(w_i\). The second
      sum is the derivative of this polynomial at \(w_i\), multiplied by the
      pilot-index error in \eqref{eq:detailed-pilot-shift}. Hence,
      \(\psi_{i,j}(w_i)^\top\bm z_j\) represents the local Taylor
      polynomial together with the complete first-order correction for the
      pilot-index error. When \(\varpi(\beta)=0\), the derivative term is absent.
      
      If \(\beta=1\), then \(\varpi(\beta)=0\) and
      \[
          \psi_{i,j}(w_i)^\top\bm z_j
          =
          g(a_{j-1}).
      \]
      Since \(g\) is Lipschitz continuous, \eqref{eq:detailed-true-residual}
      gives
      \[
      \begin{aligned}
       \left|
              g(p_i-\bm x_i^\top\bm\theta_\star)
              -
              \psi_{i,j}(w_i)^\top\bm z_j
          \right|&=
          \left|
              g(p_i-\bm x_i^\top\bm\theta_\star)
              -
              g(a_{j-1})
          \right|\\
          &\quad\le
      L_g
      \left(
          |w_i-a_{j-1}|
          +
          \left|
              (p_i-w_i)-\bm x_i^\top\bm\theta_\star
          \right|
      \right)\\
      &\quad\le
      L_g(h+\eta)
      \le 2L_g h \le
          C_l(h^\beta+\eta^2),
      \end{aligned}
      \]
      because \(h^\beta=h\), \(\eta \leq h\), and
      \(C_l\ge2L_g\).
      
      Suppose now that \(\beta>1\). Applying the H\"older Taylor expansion at
      \(a_{j-1}\) and using \eqref{eq:detailed-true-residual}, we obtain
      \begin{align}
      \label{eq:detailed-holder-taylor}
          g(p_i-\bm x_i^\top\bm\theta_\star)
          ={}&
          \sum_{k=0}^{\varpi(\beta)}
          \frac{g^{(k)}(a_{j-1})}{k!}
          \Bigl[
              (w_i-a_{j-1})
              +(p_i-w_i)-\bm x_i^\top\bm\theta_\star
          \Bigr]^k
          +R_i,
      \end{align}
      where
      \[
      \begin{aligned}
          |R_i|
          &\le
          \frac{L_g}{\varpi(\beta)!}
          \left(
              |w_i-a_{j-1}|
              +
              \left|
                  (p_i-w_i)-\bm x_i^\top\bm\theta_\star
              \right|
          \right)^\beta\le
          \frac{L_g}{\varpi(\beta)!}(h+\eta)^\beta\le
          \frac{2^\beta L_g}{\varpi(\beta)!}h^\beta
          \le
          C_lh^\beta.
      \end{aligned}
      \]
      
      For each \(k\ge1\), applying the binomial formula to the \(k\)th power in
      \eqref{eq:detailed-holder-taylor} gives
      \[
      \begin{aligned}
      &
      \Bigl[
          (w_i-a_{j-1})
          +(p_i-w_i)-\bm x_i^\top\bm\theta_\star
      \Bigr]^k=
      \sum_{r=0}^{k}
      \binom{k}{r}
      (w_i-a_{j-1})^{k-r}
      \bigl((p_i-w_i)-\bm x_i^\top\bm\theta_\star\bigr)^r.
      \end{aligned}
      \]
      The term with \(r=0\) does not contain the pilot-index error, while the
      term with \(r=1\) is linear in the pilot-index error. After multiplying
      by the Taylor coefficients and summing over \(k\), these terms are exactly
      the two terms in
      \eqref{eq:detailed-linearized-feature-identity}. Therefore,
      \(\psi_{i,j}(w_i)^\top\bm z_j\) contains all terms in the Taylor
      expansion that involve at most one factor of
      \(
          (p_i-w_i)-\bm x_i^\top\bm\theta_\star.
      \)
      
      The remaining terms correspond to \(r\ge2\). For every
      \(2\le r\le k\), using
      \[
          |w_i-a_{j-1}|\le h,
          \quad
          \left|
              (p_i-w_i)-\bm x_i^\top\bm\theta_\star
          \right|
          \le\eta,
          \quad
          \eta\le h,
      \]
      we obtain
      \[
      \begin{aligned}
      |w_i-a_{j-1}|^{k-r}
      \left|
          (p_i-w_i)-\bm x_i^\top\bm\theta_\star
      \right|^r
      &\le h^{k-r}\eta^r
      \le h^{k-r}\eta^2h^{r-2}\\
      &=h^{k-2}\eta^2
      \le (2B)^{k-2}\eta^2,
      \end{aligned}
      \]
      where the second inequality follows because \(\eta \leq h\) and \(r \ge 2\) and the last inequality follows because \(h\le2B\). Moreover,
      \eqref{eq:extended-derivative-bound} uniformly bounds the Taylor
      coefficients by \(C_g\), and the numbers of possible values of \(k\) and \(r\) depend
      only on \(\beta\). Hence,
      \[
      \begin{aligned}
      \sum_{k=2}^{\varpi(\beta)}
      \frac{|g^{(k)}(a_{j-1})|}{k!}
      \sum_{r=2}^{k}
      \binom{k}{r}
      |w_i-a_{j-1}|^{k-r}
      \left|
          (p_i-w_i)-\bm x_i^\top\bm\theta_\star
      \right|^r
      &\le C_g\eta^2
      \sum_{k=2}^{\varpi(\beta)}
      \frac{
          (2^k-k-1)(2B)^{k-2}
      }{
          k!
      }\\
      &\le C_l\eta^2,
      \end{aligned}
      \]
      where we used
      \(
          \sum_{r=2}^{k}\binom{k}{r}=2^k-k-1.
      \)
      
      Combining this bound with the Taylor remainder
      \(|R_i|\le C_lh^\beta\), we have
        \begin{equation}
\label{eq:detailed-linearized-sample}
    g(p_i-\bm x_i^\top\bm\theta_\star)
    =
    \psi_{i,j}(w_i)^\top\bm z_j
    +
    R_{i,j}(w_i),
    \qquad
    |R_{i,j}(w_i)|
    \le
    C_l(h^\beta+\eta^2).
\end{equation}
When \(1<\beta\le2\), we have \(\varpi(\beta)=1\), so the sums over
      \(k\ge2\) are empty. In this case, no quadratic or higher-order term in the
      pilot-index error is omitted, and the additional \(\eta^2\) term is
      unnecessary but remains a valid uniform upper bound.
      
      The same argument applies to every current candidate action
      \(w\in\mathcal W_{t,j}\). Indeed,
      \(\widehat p_t(w)\in[0,B]\), so Assumption~\ref{ass:bounded} implies
      \[
          \widehat p_t(w)-\bm x_t^\top\bm\theta_\star
          \in\mathcal I_g.
      \]
      Moreover, \(w\in I_j\) and Step~1 give
      \[
          0\le w-a_{j-1}\le h,
          \quad
          \left|
              \widehat p_t(w)-w-\bm x_t^\top\bm\theta_\star
          \right|
          \le\eta.
      \]

Therefore,
\begin{equation}
\label{eq:detailed-linearized-test}
    g\!\left(
        \widehat p_t(w)-\bm x_t^\top\bm\theta_\star
    \right)
    =
    \psi_{t,j}(w)^\top\bm z_j
    +
    R_{t,j}(w),
    \qquad
    |R_{t,j}(w)|
    \le
    C_l(h^\beta+\eta^2).
\end{equation}

      \paragraph{Step 3: ridge-error decomposition.}
      For every \(i\in\Psi_{t,s}^j\), Proposition~\ref{prop:mds_piloted_lpr}
      and \eqref{eq:detailed-linearized-sample} give
      \(
          y_i
          =
          \psi_{i,j}(w_i)^{\top}\bm z_j
          +
          R_{i,j}(w_i)
          +
          \varepsilon_i.
      \)
      Thus each observed demand consists of the linear approximation
      \(\psi_{i,j}(w_i)^{\top}\bm z_j\), the approximation error
      \(R_{i,j}(w_i)\), and the demand noise \(\varepsilon_i\).
      
      By the definition of \(\widehat{\bm z}_{t,s}^j\),
      \[
          (\Lambda_{t,s}^j+\lambda I)
          \widehat{\bm z}_{t,s}^j
          =
          \sum_{i\in\Psi_{t,s}^j}
          y_i\psi_{i,j}(w_i).
      \]
      Substituting the preceding expression for \(y_i\) yields
      \[
      \begin{aligned}
          (\Lambda_{t,s}^j+\lambda I)
          \widehat{\bm z}_{t,s}^j
          &=
          \sum_{i\in\Psi_{t,s}^j}
          \psi_{i,j}(w_i)
          \left(
              \psi_{i,j}(w_i)^{\top}\bm z_j
              +
              R_{i,j}(w_i)
              +
              \varepsilon_i
          \right)\\
          &=
          \sum_{i\in\Psi_{t,s}^j}
          \psi_{i,j}(w_i)
          \psi_{i,j}(w_i)^{\top}\bm z_j
          +
          \sum_{i\in\Psi_{t,s}^j}
          \psi_{i,j}(w_i)R_{i,j}(w_i)
          +
          \sum_{i\in\Psi_{t,s}^j}
          \psi_{i,j}(w_i)\varepsilon_i.
      \end{aligned}
      \]
      By the definition
      \[
          \Lambda_{t,s}^j
          =
          \sum_{i\in\Psi_{t,s}^j}
          \psi_{i,j}(w_i)
          \psi_{i,j}(w_i)^{\top},
      \]
      the first term on the right-hand side equals
      \(\Lambda_{t,s}^j\bm z_j\). Hence
      \[
          (\Lambda_{t,s}^j+\lambda I)
          \widehat{\bm z}_{t,s}^j
          =
          \Lambda_{t,s}^j\bm z_j
          +
          \sum_{i\in\Psi_{t,s}^j}
          \psi_{i,j}(w_i)R_{i,j}(w_i)
          +
          \sum_{i\in\Psi_{t,s}^j}
          \psi_{i,j}(w_i)\varepsilon_i.
      \]
      
      Subtracting
      \((\Lambda_{t,s}^j+\lambda I)\bm z_j\)
      from both sides gives
      \[
      \begin{aligned}
          (\Lambda_{t,s}^j+\lambda I)
          \left(
              \widehat{\bm z}_{t,s}^j-\bm z_j
          \right)
          &=
          \sum_{i\in\Psi_{t,s}^j}
          \psi_{i,j}(w_i)\varepsilon_i+
          \sum_{i\in\Psi_{t,s}^j}
          \psi_{i,j}(w_i)R_{i,j}(w_i)
          -
          \lambda\bm z_j.
      \end{aligned}
      \]
      The term \(-\lambda\bm z_j\) appears because
      \[
          \Lambda_{t,s}^j\bm z_j
          -
          (\Lambda_{t,s}^j+\lambda I)\bm z_j
          =
          -\lambda\bm z_j.
      \]
      
      Since \(\lambda>0\), the matrix
      \(\Lambda_{t,s}^j+\lambda I\) is positive definite and therefore
      invertible. Multiplying both sides by
      \((\Lambda_{t,s}^j+\lambda I)^{-1}\) gives
      \begin{equation}
      \label{eq:detailed-ridge-decomposition}
      \begin{aligned}
          \widehat{\bm z}_{t,s}^j-\bm z_j
          =
          (\Lambda_{t,s}^j+\lambda I)^{-1}
          \Bigg(
              &\sum_{i\in\Psi_{t,s}^j}
              \psi_{i,j}(w_i)\varepsilon_i+
              \sum_{i\in\Psi_{t,s}^j}
              \psi_{i,j}(w_i)R_{i,j}(w_i)
              -
              \lambda\bm z_j
          \Bigg).
      \end{aligned}
      \end{equation}
      The three terms on the right-hand side arise, respectively, from the
      random demand noise, the approximation error in
      \eqref{eq:detailed-linearized-sample}, and the ridge regularization.

      \paragraph{Step 4: noise, approximation, and regularization bounds.}
      Fix \(s\in[S]\) and \(j\in[N]\). Proposition~\ref{prop:mds_piloted_lpr}
      shows that \(\overline\psi_{i,j,s}\) is determined before \(y_i\) is
      observed and that
      \(
          \mathbb E\!\left[
              \overline\psi_{i,j,s}\varepsilon_i
              \mid
              \mathcal F_i^-
          \right]
          =
          \bm 0.
      \)
      Moreover, conditional on \(\mathcal F_i^-\), \(\varepsilon_i\) has mean
      zero and lies in an interval of length \(D\). Hence, by the conditional
      Hoeffding lemma, for every \(\gamma\in\mathbb R\),
      \[
          \mathbb E\!\left[
              \exp(\gamma\varepsilon_i)
              \mid
              \mathcal F_i^-
          \right]
          \le
          \exp\!\left(\frac{\gamma^2D^2}{8}\right).
      \]
      Thus, \(\varepsilon_i\) is conditionally \(D/2\)-sub-Gaussian, and the
      standard time-uniform self-normalized concentration inequality applies to
      the features \(\overline\psi_{i,j,s}\).
      
      By the definition of \(\overline\psi_{i,j,s}\), for every \(t\le T+1\),
      \[
          \sum_{i<t}
          \overline\psi_{i,j,s}\varepsilon_i
          =
          \sum_{i\in\Psi_{t,s}^j}
          \psi_{i,j}(w_i)\varepsilon_i,
      \]
      and
      \[
          \sum_{i<t}
          \overline\psi_{i,j,s}
          \bigl(\overline\psi_{i,j,s}\bigr)^\top
          =
          \sum_{i\in\Psi_{t,s}^j}
          \psi_{i,j}(w_i)\psi_{i,j}(w_i)^\top
          =
          \Lambda_{t,s}^j.
      \]
      
      Hence, for pair \((s,j)\), with probability at least
      \(1-\delta/(2SN)\), for every \(t\le T+1\),
      \[
      \begin{aligned}
          \left\|
              \sum_{i\in\Psi_{t,s}^j}
              \psi_{i,j}(w_i)\varepsilon_i
          \right\|_{(\Lambda_{t,s}^j+\lambda I)^{-1}}
          \le
          D\sqrt{
              \log
              \frac{
                  \det(\Lambda_{t,s}^j+\lambda I)
              }{
                  \det(\lambda I)
              }
              +
              2\log\left(\frac{2SN}{\delta}\right)
          }.
      \end{aligned}
      \]
      
      The matrix \(\Lambda_{t,s}^j\) has dimension
      \(1+\varpi(\beta)(d+1)\). Define
      \begin{equation}
\label{eq:C-psi-definition}
\begin{aligned}
    C_\psi^2
    :={}&
    1
    +
    \sum_{k=1}^{\varpi(\beta)}
    \left[
        (2B)^k
        +
        Bk(2B)^{k-1}
    \right]^2
    +
    C_x^2
    \sum_{k=1}^{\varpi(\beta)}
    k^2(2B)^{2(k-1)}.
\end{aligned}
\end{equation}
      Since
      \(0\le w_i-a_{j-1}\le h\le2B\),
      \(0\le p_i-w_i\le B\), and
      \(\|\bm x_i\|_2\le C_x\), the definition of the joint feature gives \(\|\psi_{i,j}(w_i)\|_2\le
      C_\psi\). The product of positive numbers is no larger
      than the corresponding power of their average, and hence
      \[
      \begin{aligned}
          \log
          \frac{
              \det(\Lambda_{t,s}^j+\lambda I)
          }{
              \det(\lambda I)
          }
          &\le
          \bigl[1+\varpi(\beta)(d+1)\bigr]
          \log\left(
              1+
              \frac{
                  \operatorname{tr}(\Lambda_{t,s}^j)
              }{
                  \lambda[1+\varpi(\beta)(d+1)]
              }
          \right)\\
          &\le
          C_\psi^2\bigl[1+\varpi(\beta)(d+1)\bigr]
          \log\left(
              1+\frac{T}{\lambda}
          \right)\le
          C_\psi^2\iota_T,
      \end{aligned}
      \]
      where the second inequality follows because
      \[
          \operatorname{tr}(\Lambda_{t,s}^j)
          =
          \sum_{i\in\Psi_{t,s}^j}
          \|\psi_{i,j}(w_i)\|_2^2
          \le
          C_\psi^2|\Psi_{t,s}^j|
          \le
          C_\psi^2T.
      \]
      
      It follows that, for this fixed pair \((s,j)\), with probability at least
      \(1-\delta/(2SN)\),
      \[
          \left\|
              \sum_{i\in\Psi_{t,s}^j}
              \psi_{i,j}(w_i)\varepsilon_i
          \right\|_{(\Lambda_{t,s}^j+\lambda I)^{-1}}
          \le
          C_\psi D\sqrt{\iota_T}
      \]
      simultaneously for every \(t\le T+1\).
      
      There are \(SN\) layer--bin pairs. Since the failure probability for each
      fixed pair is at most \(\delta/(2SN)\), the probability that the preceding
      bound fails for at least one pair is at most
      \[
          SN\frac{\delta}{2SN}
          =
          \frac{\delta}{2}.
      \]
      
      Therefore, there exists an event \(\mathcal E_{\rm sc}\) with
      \(\mathbb P(\mathcal E_{\rm sc}) \ge 1-\delta/2,
      \)
      on which, simultaneously for every \(t\le T+1\), \(s\in[S]\), and
      \(j\in[N]\),
      \begin{equation}
      \label{eq:detailed-self-normalized-final}
          \left\|
              \sum_{i\in\Psi_{t,s}^j}
              \psi_{i,j}(w_i)\varepsilon_i
          \right\|_{(\Lambda_{t,s}^j+\lambda I)^{-1}}
          \le
          C_\psi D\sqrt{\iota_T}.
      \end{equation}

      We first control the term involving the approximation errors \(R_{i,j}(w_i)\).
      For any vector \(a\), the Cauchy--Schwarz inequality gives
      \[
      \begin{aligned}
          \left|
              a^\top
              \sum_{i\in\Psi_{t,s}^j}
              \psi_{i,j}(w_i)R_{i,j}(w_i)
          \right|
          &=
          \left|
              \sum_{i\in\Psi_{t,s}^j}
              \bigl(a^\top\psi_{i,j}(w_i)\bigr)R_{i,j}(w_i)
          \right|\\
          &\le
          \left[
              \sum_{i\in\Psi_{t,s}^j}
              \bigl(a^\top\psi_{i,j}(w_i)\bigr)^2
          \right]^{1/2}\times
          \left[
              \sum_{i\in\Psi_{t,s}^j}
              R_{i,j}(w_i)^2
          \right]^{1/2}.
      \end{aligned}
      \]
      By the definition of \(\Lambda_{t,s}^j\),
      \[
      \begin{aligned}
          \sum_{i\in\Psi_{t,s}^j}
          \bigl(a^\top\psi_{i,j}(w_i)\bigr)^2
          &=
          a^\top\Lambda_{t,s}^j a\le
          a^\top
          (\Lambda_{t,s}^j+\lambda I)
          a=
          \|a\|_{\Lambda_{t,s}^j+\lambda I}^2.
      \end{aligned}
      \]
      Therefore,
      \[
          \left|
              a^\top
              \sum_{i\in\Psi_{t,s}^j}
              \psi_{i,j}(w_i)R_{i,j}(w_i)
          \right|
          \le
          \|a\|_{\Lambda_{t,s}^j+\lambda I}
          \left[
              \sum_{i\in\Psi_{t,s}^j}
              R_{i,j}(w_i)^2
          \right]^{1/2}.
      \]
      
      Taking
      \[
          a
          =
          (\Lambda_{t,s}^j+\lambda I)^{-1}
          \sum_{i\in\Psi_{t,s}^j}
          \psi_{i,j}(w_i)R_{i,j}(w_i)
      \]
      in the preceding inequality yields
      \[
      \begin{aligned}
      &
      \left\|
          \sum_{i\in\Psi_{t,s}^j}
          \psi_{i,j}(w_i)R_{i,j}(w_i)
      \right\|_{(\Lambda_{t,s}^j+\lambda I)^{-1}}^2\le
      \left\|
          \sum_{i\in\Psi_{t,s}^j}
          \psi_{i,j}(w_i)R_{i,j}(w_i)
      \right\|_{(\Lambda_{t,s}^j+\lambda I)^{-1}}
      \left[
          \sum_{i\in\Psi_{t,s}^j}
          R_{i,j}(w_i)^2
      \right]^{1/2}.
      \end{aligned}
      \]
      Canceling the common factor when it is nonzero, while the zero case is
      immediate, gives
      \[
          \left\|
              \sum_{i\in\Psi_{t,s}^j}
              \psi_{i,j}(w_i)R_{i,j}(w_i)
          \right\|_{(\Lambda_{t,s}^j+\lambda I)^{-1}}
          \le
          \left[
              \sum_{i\in\Psi_{t,s}^j}
              R_{i,j}(w_i)^2
          \right]^{1/2}.
      \]

      Since
      \(
          |R_{i,j}(w_i)|
          \le
          C_l (h^\beta+\eta^2)
      \),
      for every \(i\in\Psi_{t,s}^j\), we have
      \[
      \begin{aligned}
          \sum_{i\in\Psi_{t,s}^j}
          R_{i,j}(w_i)^2
          &\le
          C_l^2
          |\Psi_{t,s}^j|
          (h^\beta+\eta^2)^2.
      \end{aligned}
      \]
      It follows that
      \begin{equation}
      \label{eq:detailed-approximation-score}
          \left\|
              \sum_{i\in\Psi_{t,s}^j}
              \psi_{i,j}(w_i)R_{i,j}(w_i)
          \right\|_{(\Lambda_{t,s}^j+\lambda I)^{-1}}
          \le
          C_l\sqrt{|\Psi_{t,s}^j|}
          (h^\beta+\eta^2).
      \end{equation}
      
      It remains to control the term caused by the ridge regularization. Since
      \(
          \Lambda_{t,s}^j+\lambda I
          \succeq
          \lambda I,
      \)
      we have
      \[
          (\Lambda_{t,s}^j+\lambda I)^{-1}
          \preceq
          \frac{1}{\lambda}I.
      \]
      Therefore,
      \[
      \begin{aligned}
          \|\lambda\bm z_j\|_{(\Lambda_{t,s}^j+\lambda I)^{-1}}^2
          &=
          \lambda^2
          \bm z_j^\top
          (\Lambda_{t,s}^j+\lambda I)^{-1}
          \bm z_j\le
          \lambda
          \|\bm z_j\|_2^2.
      \end{aligned}
      \]
      Taking square roots and using
      \eqref{eq:detailed-z-bound} gives
      \begin{equation}
      \label{eq:detailed-ridge-bias}
          \|\lambda\bm z_j\|_{(\Lambda_{t,s}^j+\lambda I)^{-1}}
          \le
          \sqrt{\lambda}\|\bm z_j\|_2
          \le
          C_z\sqrt{\lambda}.
      \end{equation}

      \paragraph{Step 5: uniform confidence bound and completion.}
      We work on the event
      \(\mathcal E_{\rm pil}\cap\mathcal E_{\rm sc}\), which has probability at
      least \(1-\delta\). Fix a main-policy round \(t\), a layer
      \(s\in[s_t]\), and an action \((j,w)\in\mathcal A_{t,s}\). Assume first that  \(|\Psi_{t,s}^j|\ge1\).
      
      By \eqref{eq:detailed-ridge-decomposition} and multiplying both sides by
      \(\psi_{t,j}(w)^\top\), taking absolute values, and applying
      the triangle inequality gives
      \[
      \begin{aligned}
      &
      \left|
          \psi_{t,j}(w)^\top
          \left(
              \widehat{\bm z}_{t,s}^j-\bm z_j
          \right)
      \right|\le
      \left\|
          \psi_{t,j}(w)
      \right\|_{(\Lambda_{t,s}^j+\lambda I)^{-1}}
      \Bigg[ \ 
          \Bigg\|
              \sum_{i\in\Psi_{t,s}^j}
              \psi_{i,j}(w_i)\varepsilon_i
          \Bigg\|_{(\Lambda_{t,s}^j+\lambda I)^{-1}}\\
      &\hspace{42mm}
          +
          \Bigg\|
              \sum_{i\in\Psi_{t,s}^j}
              \psi_{i,j}(w_i)R_{i,j}(w_i)
          \Bigg\|_{(\Lambda_{t,s}^j+\lambda I)^{-1}}
          +
          \left\|
              \lambda\bm z_j
          \right\|_{(\Lambda_{t,s}^j+\lambda I)^{-1}}
      \Bigg].
      \end{aligned}
      \]
      Using
      \eqref{eq:detailed-self-normalized-final},
      \eqref{eq:detailed-approximation-score}, and
      \eqref{eq:detailed-ridge-bias}, we obtain
      \[
      \begin{aligned}
      \left|
          \psi_{t,j}(w)^\top
          \left(
              \widehat{\bm z}_{t,s}^j-\bm z_j
          \right)
      \right|
      &\le
      \left(
          C_{\psi}D\sqrt{\iota_T}
          +
          C_z\sqrt{\lambda}
      \right)
      \left\|
          \psi_{t,j}(w)
      \right\|_{(\Lambda_{t,s}^j+\lambda I)^{-1}}\\
      &\quad+
      C_l
      (h^\beta+\eta^2)
      \sqrt{|\Psi_{t,s}^j|}
      \left\|
          \psi_{t,j}(w)
      \right\|_{(\Lambda_{t,s}^j+\lambda I)^{-1}}.
      \end{aligned}
      \]
      
      The preceding inequality controls the error caused by replacing
      \(\bm z_j\) with its estimator \(\widehat{\bm z}_{t,s}^j\). We must also
      account for the approximation error at the current action. By
      \eqref{eq:detailed-linearized-test},
      \[
          g\!\left(
              \widehat p_t(w)-\bm x_t^\top\bm\theta_\star
          \right)
          =
          \psi_{t,j}(w)^\top\bm z_j
          +
          R_{t,j}(w),
      \]
      where \(|R_{t,j}(w)| \le C_l(h^\beta+\eta^2).\) Therefore,
      \[
      \begin{aligned}
      \left|
          \psi_{t,j}(w)^\top
          \widehat{\bm z}_{t,s}^j
          -
          g\!\left(
              \widehat p_t(w)-\bm x_t^\top\bm\theta_\star
          \right)
      \right|
      &=
      \left|
          \psi_{t,j}(w)^\top
          \left(
              \widehat{\bm z}_{t,s}^j-\bm z_j
          \right)
          -
          R_{t,j}(w)
      \right|\\
      &\le
      \left|
          \psi_{t,j}(w)^\top
          \left(
              \widehat{\bm z}_{t,s}^j-\bm z_j
          \right)
      \right|
      +
      |R_{t,j}(w)|\\
      &\le
      \left(
          C_{\psi}D\sqrt{\iota_T}
          +
          C_z\sqrt{\lambda}
      \right)
      \left\|
          \psi_{t,j}(w)
      \right\|_{(\Lambda_{t,s}^j+\lambda I)^{-1}}\\
      &\quad+
      C_l
      (h^\beta+\eta^2)
      \left[
          1+
          \sqrt{|\Psi_{t,s}^j|}
          \left\|
              \psi_{t,j}(w)
          \right\|_{(\Lambda_{t,s}^j+\lambda I)^{-1}}
      \right].
      \end{aligned}
      \]

      The algorithm restricts the estimated demand to the interval \([0,D]\):
      \[
          \widehat g_{t,s}^j(w)
          =
          \Pi_{[0,D]}\!\left(
              \psi_{t,j}(w)^\top
              \widehat{\bm z}_{t,s}^j
          \right).
      \]
      This restriction cannot increase the error because the true mean demand
      also belongs to \([0,D]\). Indeed, for every \(a\in\mathbb R\) and every
      \(b\in[0,D]\),
      \(
          \left|
              \Pi_{[0,D]}(a)-b
          \right|
          \le
          |a-b|.
      \)
      Moreover, both \(\Pi_{[0,D]}(a)\) and \(b\) belong to \([0,D]\), so
      \(
          \left|
              \Pi_{[0,D]}(a)-b
          \right|
          \le
          D.
      \)
      Applying these two inequalities with
      \(
          a
          =
          \psi_{t,j}(w)^\top
          \widehat{\bm z}_{t,s}^j
      \)
      and
      \(
          b
          =
          g\!\left(
              \widehat p_t(w)-\bm x_t^\top\bm\theta_\star
          \right)
      \)
      gives
      \[
      \begin{aligned}
      &
      \left|
          \widehat g_{t,s}^j(w)
          -
          g\!\left(
              \widehat p_t(w)-\bm x_t^\top\bm\theta_\star
          \right)
      \right|\\
      &\le
      \min\Biggl\{D,\,
      \left(
          C_{\psi}D\sqrt{\iota_T}
          +
          C_z\sqrt{\lambda}
      \right)
      \left\|
          \psi_{t,j}(w)
      \right\|_{(\Lambda_{t,s}^j+\lambda I)^{-1}}\\
      &\hspace{18mm}
      +
      C_l
      (h^\beta+\eta^2)
      \left[
          1+
          \sqrt{|\Psi_{t,s}^j|}
          \left\|
              \psi_{t,j}(w)
          \right\|_{(\Lambda_{t,s}^j+\lambda I)^{-1}}
      \right]
      \Biggr\}\le
      r_{t,s}^j(w).
      \end{aligned}
      \]

      It remains to consider the case \(|\Psi_{t,s}^j|=0\). Then
      \(\Psi_{t,s}^j=\emptyset\) and \(\Lambda_{t,s}^j=\bm 0\). The optimization
      problem reduces to
      \(
          \min_{\bm z}
          \lambda\|\bm z\|_2^2.
      \)
      Since \(\lambda>0\), its unique minimizer is
      \(
          \widehat{\bm z}_{t,s}^j=\bm 0.
      \)
      It follows that
      \(
          \widehat g_{t,s}^j(w)
          =
          \Pi_{[0,D]}(0)
          =
          0.
      \)
      The algorithm sets
      \(
          r_{t,s}^j(w)=D.
      \)
      Since the true mean demand belongs to \([0,D]\),
      \[
      \begin{aligned}
      &
      \left|
          \widehat g_{t,s}^j(w)
          -
          g\!\left(
              \widehat p_t(w)-\bm x_t^\top\bm\theta_\star
          \right)
      \right|=
      \left|
          g\!\left(
              \widehat p_t(w)-\bm x_t^\top\bm\theta_\star
          \right)
      \right|\le
      D
      =
      r_{t,s}^j(w).
      \end{aligned}
      \]
      
      Finally, the event \(\mathcal E_{\rm pil}\) controls all periods
      simultaneously, and the event \(\mathcal E_{\rm sc}\) controls all periods,
      layers, and bins simultaneously. Once these two events occur, the preceding
      inequalities can be applied to every residual action \(w\) because they use
      only the already established vector bounds and the Cauchy--Schwarz
      inequality. Thus no separate probability bound is needed for each \(w\).
      Therefore, \eqref{eq:linearized-uniform-confidence} holds simultaneously
      for every main-policy round \(t\), every \(s\in[s_t]\), and every
      \((j,w)\in\mathcal A_{t,s}\).

      \end{myproof}

\subsection{Proofs for the Contextual Regret Bounds}
\noindent
\textbf{Proof of \cref{lem:ldp_revenue_gap_discrete_fixed_residual}.}
 
\begin{myproof}
  We proceed in four steps.
  
  \paragraph{Step 1: Revenue-UCB sandwich.}
  Fix a visited layer \(s\) and an action
  \((j,w)\in\mathcal A_{t,s}\). On the confidence event,
  $
      \left|
          \widehat g_{t,s}^j(w)
          -
          g\!\left(
              \widehat p_t(w)-\bm x_t^\top\bm\theta_\star
          \right)
      \right|
      \le
      r_{t,s}^j(w).
  $
  Because the true conditional mean belongs to \([0,D]\),
  \[
  \begin{aligned}
      g\!\left(
          \widehat p_t(w)-\bm x_t^\top\bm\theta_\star
      \right)
      &\le
      \min\left\{
          D,\widehat g_{t,s}^j(w)+r_{t,s}^j(w)
      \right\}                                             \le
      g\!\left(
          \widehat p_t(w)-\bm x_t^\top\bm\theta_\star
      \right)
      +2r_{t,s}^j(w).
  \end{aligned}
  \]
  Multiplying by \(\widehat p_t(w)\ge0\) and using the definitions of
  \(U_{t,s}^j(w)\) and \(\operatorname{wid}_{t,s}^j(w)\) gives
  \begin{equation}
  \label{eq:revenue-ucb-sandwich-discrete}
  \begin{aligned}
      \mathsf{Rev}\bigl(\bm x_t,\widehat p_t(w)\bigr)
      &\le
      U_{t,s}^j(w)                                         \le
      \mathsf{Rev}\bigl(\bm x_t,\widehat p_t(w)\bigr)
      +2\operatorname{wid}_{t,s}^j(w).
  \end{aligned}
  \end{equation}
  
  \paragraph{Step 2: Preservation of the discrete benchmark.}
  Suppose that the precision check passes at a layer \(s<S\). Then
  \[
      \operatorname{wid}_{t,s}^j(w)
      \le BD\,2^{-s}
      \qquad
      \text{for every }(j,w)\in\mathcal A_{t,s}.
  \]
  Choose
  $
      (j^\star,w^\star)
      \in
      \operatorname*{argmax}_{(j,w)\in\mathcal A_{t,s}}
      \mathsf{Rev}\bigl(\bm x_t,\widehat p_t(w)\bigr).
  $
  By the lower side of \eqref{eq:revenue-ucb-sandwich-discrete},
  \(
      U_{t,s}^{j^\star}(w^\star)\ge V_{t,s}.
  \)
  By its upper side, for every \((j,w)\in\mathcal A_{t,s}\),
  \[
      U_{t,s}^j(w)
      \le
      V_{t,s}+2BD\,2^{-s}
      =
      V_{t,s}+BD\,2^{1-s}.
  \]
  Consequently,
  \[
      U_{t,s}^{j^\star}(w^\star)
      \ge
      \max_{(j,w)\in\mathcal A_{t,s}}U_{t,s}^j(w)
      -BD\,2^{1-s}.
  \]
  
  By the definition of \(\mathcal A_{t,s+1}\) in
  Algorithm~\ref{alg:lpdm_pilot_residual}, the preceding inequality implies
  that \((j^\star,w^\star)\in\mathcal A_{t,s+1}\).
  Since \((j^\star,w^\star)\) attains \(V_{t,s}\), it follows that
  \[
      V_{t,s+1}
      \ge
      \mathsf{Rev}\bigl(\bm x_t,\widehat p_t(w^\star)\bigr)
      =
      V_{t,s}.
  \]
  
  On the other hand,
  \(\mathcal A_{t,s+1}\subseteq\mathcal A_{t,s}\), and hence \(V_{t,s+1}\le V_{t,s}\). Therefore, \(V_{t,s+1}=V_{t,s}\).

  \paragraph{Step 3: Gap between the selected price and the initial discrete benchmark.}
  Suppose that \(s_t\ge2\). To reach layer \(s_t\), the algorithm must have
  passed the precision check at layer \(s_t-1\). Moreover,
  \((j_t,w_t)\in\mathcal A_{t,s_t}\), so this action survived the refinement
  from layer \(s_t-1\) to layer \(s_t\). Hence
  \begin{align}
      U_{t,s_t-1}^{j_t}(w_t)
      &\ge
      \max_{(j,w)\in\mathcal A_{t,s_t-1}}
      U_{t,s_t-1}^j(w)
      -BD\,2^{2-s_t}                                       \ge
      V_{t,s_t-1}-BD\,2^{2-s_t}.
  \label{eq:selected-action-survival-bound}
  \end{align}
  The upper side of \eqref{eq:revenue-ucb-sandwich-discrete} and the
  precision check at layer \(s_t-1\) give
  \begin{align}
      U_{t,s_t-1}^{j_t}(w_t)
      &\le
      \mathsf{Rev}(\bm x_t,p_t)
      +2BD\,2^{-(s_t-1)}                                      =
      \mathsf{Rev}(\bm x_t,p_t)
      +BD\,2^{2-s_t}.
  \label{eq:selected-action-previous-layer-upper}
  \end{align}
  Combining \eqref{eq:selected-action-survival-bound} and
  \eqref{eq:selected-action-previous-layer-upper} yields
  $
      V_{t,s_t-1}-\mathsf{Rev}(\bm x_t,p_t)
      \le
      2BD\,2^{2-s_t}.
  $
  Every layer preceding \(s_t\) passed the precision check. Repeated
  application of Step 2 therefore gives
  \(V_{t,s_t-1}=V_{t,1}\), which proves the second assertion of the lemma.
  
  \paragraph{Step 4: Discretization of the continuous oracle price.}
  If \(\beta=1\), Assumption~\ref{ass:g_smooth} implies that \(g\) is \(L_g\)-Lipschitz. If \(\beta>1\), Step~2 of the proof of Proposition~\ref{prop:ucb} gives
  \(\sup_{u\in\mathcal I_g}|g'(u)|\le C_g.\)
  Hence \(g\) is Lipschitz on \(\mathcal I_g\) with constant \(\max\{L_g,C_g\}\). Since both residual arguments below belong to
  \(\mathcal I_g\) by Assumption~\ref{ass:bounded}, for every
  \(\bm x\in\mathcal X\) and \(p,p'\in[0,B]\),
  \begin{align*}
      \left|
          \mathsf{Rev}(\bm x,p)-\mathsf{Rev}(\bm x,p')
      \right|
      &\le
      D|p-p'|
      +B\left|
          g(p-\bm x^\top\bm\theta_\star)
          -
          g(p'-\bm x^\top\bm\theta_\star)
      \right|                                              \le
      L_{\mathrm{Rev}}|p-p'|,
  \end{align*}
  where one may take \(L_{\mathrm{Rev}}=D+B\max\{L_g,C_g\}\).
  
  The oracle residual
  $
      p_t^\star
      -
      \Pi_{[0,B]}\!\left(\bm x_t^\top\widehat{\bm\theta}_t\right)
  $
  belongs to \([-B,B]\). The feasible residual set is the interval
  $
      \left[
          -\Pi_{[0,B]}\!\left(\bm x_t^\top\widehat{\bm\theta}_t\right),
          B-\Pi_{[0,B]}\!\left(\bm x_t^\top\widehat{\bm\theta}_t\right)
      \right],
  $
  which contains both zero and the oracle residual. By choosing a point of
  \(\mathcal W\) between zero and the oracle residual, the mesh definition
  of \(\mathcal W\), together with the explicitly included endpoints, gives
  a feasible \(w\in\mathcal W\) satisfying
  \[
      \left|
          w-
          \left[
              p_t^\star
              -
              \Pi_{[0,B]}\!\left(
                  \bm x_t^\top\widehat{\bm\theta}_t
              \right)
          \right]
      \right|
      \le
      T^{-1/2}.
  \]
  Assigning this \(w\) to its unique bin yields an action in
  \(\mathcal A_{t,1}\), and
  $
      \left|\widehat p_t(w)-p_t^\star\right|
      \le
      T^{-1/2}.
  $
  Hence
  \[
      \mathsf{Rev}(\bm x_t,p_t^\star)-V_{t,1}
      \le
      \frac{L_{\mathrm{Rev}}}{\sqrt T}.
  \]
  If \(s_t\ge2\), combining this inequality with the bound established in Step~3 proves
  \eqref{eq:one-step-revenue-gap-discrete-ldp-simplified}. If \(s_t=1\),
  then \(0\le\mathsf{Rev}(\bm x_t,p)\le BD\) for every feasible price, so
  \[
  \begin{aligned}
      \mathsf{Rev}(\bm x_t,p_t^\star)-\mathsf{Rev}(\bm x_t,p_t)
      &\le
      \frac{L_{\mathrm{Rev}}}{\sqrt T}
      +V_{t,1}-\mathsf{Rev}(\bm x_t,p_t)\le
      \frac{L_{\mathrm{Rev}}}{\sqrt T}+BD
      \le
      \frac{L_{\mathrm{Rev}}}{\sqrt T}+8BD\,2^{-s_t}.
  \end{aligned}
  \]
  Thus the same conclusion holds for \(s_t=1\).
  \end{myproof}

  \noindent
  \textbf{Proof of counting lemmas.}
  \begin{lemma}[Cumulative revenue uncertainty within one layer]
    \label{lem:bound_r_layered}
    There exists a finite constant \(C_\omega\ge1\), such that, for every
    \(s\in[S]\), pathwise,
    \begin{equation}
    \label{eq:cumulative-uncertainty-layered}
    \begin{aligned}
        \sum_{t\in\Psi_{T+1,s}}
        \operatorname{wid}_{t,s}^{j_t}(w_t)
        \le
        C_\omega B\Bigl[&
            \bigl(D\sqrt{\iota_T}+\sqrt\lambda\bigr)
            \sqrt{N|\Psi_{T+1,s}|\,\iota_T}         +
            (h^\beta+\eta^2)
            |\Psi_{T+1,s}|\sqrt{\iota_T}
        \Bigr].
    \end{aligned}
    \end{equation}
    The right-hand side is interpreted as zero when
    \(\Psi_{T+1,s}=\emptyset\).
    \end{lemma}
    
    \begin{myproof}
    Fix \(s\in[S]\). The claim is immediate if
    \(\Psi_{T+1,s}=\emptyset\). We first work within one nonempty bin and
    then sum over bins.
    
    \paragraph{Step 1: Sequential Gram matrices within a permanent bin.}
    Fix \(j\in[N]\) with \(\Psi_{T+1,s}^j\ne\emptyset\), and enumerate its
    observations chronologically as
    \[
        \Psi_{T+1,s}^j
        =
        \{t_{j,1}<\cdots<t_{j,|\Psi_{T+1,s}^j|}\}.
    \]
    Permanent labeling and the definition of \(\Lambda_{t,s}^j\) imply that,
    for every \(k\le |\Psi_{T+1,s}^j|\),
    \begin{equation}
    \label{eq:within-bin-sequential-gram}
    \begin{aligned}
        \Lambda_{t_{j,k},s}^j+\lambda I
        =
        \lambda I
        +
        \sum_{\ell<k}
        \psi_{t_{j,\ell},j}(w_{t_{j,\ell}})
        \psi_{t_{j,\ell},j}(w_{t_{j,\ell}})^\top,
        \qquad
        |\Psi_{t_{j,k},s}^j|=k-1.
    \end{aligned}
    \end{equation}
    
    \paragraph{Step 2: Elliptical-potential bound.}
    The uniform feature bound established in the proof of
    Proposition~\ref{prop:ucb} gives
    \(\|\psi_{t,j}(w_t)\|_2\le C_\psi\). Applying the matrix determinant lemma
    sequentially to \eqref{eq:within-bin-sequential-gram}, and using
    \(\min\{1,x\}\le2\log(1+x)\) for \(x\ge0\), yields
    \begin{align}
    &\sum_{k=1}^{|\Psi_{T+1,s}^j|}
    \min\left\{1,
        \left\|
            \psi_{t_{j,k},j}(w_{t_{j,k}})
        \right\|_{(\Lambda_{t_{j,k},s}^j+\lambda I)^{-1}}^2
    \right\}                                                          \notag\\
    &\quad\le
    2\log
    \frac{
        \det\!\left(
            \lambda I+
            \sum_{k=1}^{|\Psi_{T+1,s}^j|}
            \psi_{t_{j,k},j}(w_{t_{j,k}})
            \psi_{t_{j,k},j}(w_{t_{j,k}})^\top
        \right)
    }{
        \det(\lambda I)
    }                                                                  \notag\\
    &\quad\le
    2\bigl[1+\varpi(\beta)(d+1)\bigr]
    \log\!\left(
        1+
        \frac{
            C_\psi^2 |\Psi_{T+1,s}^j|
        }{
            \lambda[1+\varpi(\beta)(d+1)]
        }
    \right)                                                           \notag\\
    &\quad\le 2\max\{1,C_\psi^2\}\,\iota_T.
    \label{eq:layered-elliptical-potential-minimal}
    \end{align}
    The last inequality follows from \(|\Psi_{T+1,s}^j|\le T\), the definition
    of \(\iota_T\), and \(\log(1+ax)
        \le
        \max\{1,a\}\log(1+x)\)
        for \(a,x\ge0\).
    
    By Cauchy--Schwarz and
    \eqref{eq:layered-elliptical-potential-minimal},
    \begin{equation}
    \label{eq:sum-truncated-leverage}
    \begin{aligned}
    &\sum_{k=1}^{|\Psi_{T+1,s}^j|}
    \min\left\{1,
        \left\|
            \psi_{t_{j,k},j}(w_{t_{j,k}})
        \right\|_{(\Lambda_{t_{j,k},s}^j+\lambda I)^{-1}}
    \right\}                                                         \\
    &\qquad\le
    \sqrt{
        |\Psi_{T+1,s}^j|
        \sum_{k=1}^{|\Psi_{T+1,s}^j|}
        \min\left\{1,
            \left\|
                \psi_{t_{j,k},j}(w_{t_{j,k}})
            \right\|_{(\Lambda_{t_{j,k},s}^j+\lambda I)^{-1}}^2
        \right\}
    }
    \le \sqrt{2\max\{1,C_\psi^2\}}\,
    \sqrt{|\Psi_{T+1,s}^j|\iota_T}.
    \end{aligned}
    \end{equation}
    
    \paragraph{Step 3: Sum of confidence radii within one bin.}
    The first observation in the bin satisfies
    \(
        |\Psi_{t_{j,1},s}^j|=0,
    \)
    and therefore \eqref{eq:r} gives
    \(
        r_{t_{j,1},s}^j(w_{t_{j,1}})=D.
    \)
    For \(k\ge2\), we have
    \(|\Psi_{t_{j,k},s}^j|=k-1\le|\Psi_{T+1,s}^j|.\)
    Therefore, the radius definition and its truncation at \(D\) imply
    \begin{align}
    &r_{t_{j,k},s}^j(w_{t_{j,k}})             \le
    \max\{C_\psi,C_z\}
    \bigl(D\sqrt{\iota_T}+\sqrt\lambda\bigr)
    \min\left\{1,
        \left\|
            \psi_{t_{j,k},j}(w_{t_{j,k}})
        \right\|_{(\Lambda_{t_{j,k},s}^j+\lambda I)^{-1}}
    \right\}                                                         \notag\\
    &\qquad+
    C_l(h^\beta+\eta^2)
    \left[
        1+
        \sqrt{|\Psi_{T+1,s}^j|}
        \min\left\{1,
            \left\|
                \psi_{t_{j,k},j}(w_{t_{j,k}})
            \right\|_{(\Lambda_{t_{j,k},s}^j+\lambda I)^{-1}}
        \right\}
    \right].
    \label{eq:radius-truncated-leverage-bound}
    \end{align}

    Summing \eqref{eq:radius-truncated-leverage-bound} over \(k\ge2\), adding
    the first radius, and applying \eqref{eq:sum-truncated-leverage} gives
    \begin{align*}
        \sum_{t\in\Psi_{T+1,s}^j}r_{t,s}^j(w_t)
        \le{}&
        D+
        \max\{C_\psi,C_z\}
        \sqrt{2\max\{1,C_\psi^2\}}
        \bigl(D\sqrt{\iota_T}+\sqrt\lambda\bigr)
        \sqrt{|\Psi_{T+1,s}^j|\iota_T}\\
        &+
        C_l(h^\beta+\eta^2)
        \left[
            |\Psi_{T+1,s}^j|
            +
            \sqrt{2\max\{1,C_\psi^2\}}\,
            |\Psi_{T+1,s}^j|\sqrt{\iota_T}
        \right].
    \end{align*}
    Since \(|\Psi_{T+1,s}^j|\ge1\), \(\iota_T>1\), and
    \(D\sqrt{\iota_T}+\sqrt\lambda\ge D\), the first and third terms can be
    absorbed, yielding
    \begin{equation}
    \label{eq:within-bin-radius-sum}
    \begin{aligned}
        \sum_{t\in\Psi_{T+1,s}^j}r_{t,s}^j(w_t)
        \le{}&
        \Bigl[
            1+
            \max\{C_\psi,C_z\}
            \sqrt{2\max\{1,C_\psi^2\}}
        \Bigr]
        \cdot
        \bigl(D\sqrt{\iota_T}+\sqrt\lambda\bigr)
        \sqrt{|\Psi_{T+1,s}^j|\iota_T}\\
        &+
        C_l
        \left[
            1+\sqrt{2\max\{1,C_\psi^2\}}
        \right]
        (h^\beta+\eta^2)
        |\Psi_{T+1,s}^j|\sqrt{\iota_T}.
    \end{aligned}
    \end{equation}

    \paragraph{Step 4: Sum over bins and convert demand radii to revenue widths.}
    By Cauchy--Schwarz,
    \[
        \sum_{j=1}^N\sqrt{|\Psi_{T+1,s}^j|}
        \le
        \sqrt{
            N\sum_{j=1}^N|\Psi_{T+1,s}^j|
        }
        =
        \sqrt{N|\Psi_{T+1,s}|},
    \]
    and, by the permanent partition,
    \[
        \sum_{j=1}^N|\Psi_{T+1,s}^j|
        =
        |\Psi_{T+1,s}|.
    \]
    
    By choosing
    \[C_\omega
        =
        \max\Biggl\{
            1+
            \max\{C_\psi,C_z\}
            \sqrt{2\max\{1,C_\psi^2\}},
            \;
            C_l
            \left(
                1+\sqrt{2\max\{1,C_\psi^2\}}
            \right)
        \Biggr\},\] 
    both coefficients in \eqref{eq:within-bin-radius-sum} are bounded by
    \(C_\omega\). Summing over bins and using
    \[
        \operatorname{wid}_{t,s}^{j_t}(w_t)
        =
        \widehat p_t(w_t)r_{t,s}^{j_t}(w_t)
        \le
        Br_{t,s}^{j_t}(w_t),
    \]
    therefore proves \eqref{eq:cumulative-uncertainty-layered}.
    \end{myproof}

\begin{lemma}[Occupancy of a nonterminal LDP layer]
  \label{lem:bound_Psi_layered}
  Let \(C_\omega\) be as in Lemma~\ref{lem:bound_r_layered}. For every
  \(s<S\), pathwise,
  \begin{equation}
  \label{eq:layer-count-preliminary-specialized}
  \begin{aligned}
      D2^{-s}|\Psi_{T+1,s}|
      \le
      C_\omega\Bigl[&
          \bigl(D\sqrt{\iota_T}+\sqrt\lambda\bigr)
          \sqrt{N|\Psi_{T+1,s}|\,\iota_T}
          +
          (h^\beta+\eta^2)
          |\Psi_{T+1,s}|\sqrt{\iota_T}
      \Bigr].
  \end{aligned}
  \end{equation}
  Moreover, if
  \begin{equation}
  \label{eq:layer-count-bias-separation-specialized}
      D2^{-s}
      \ge
      2C_\omega(h^\beta+\eta^2)\sqrt{\iota_T},
  \end{equation}
  then
  \begin{equation}
  \label{eq:layer-count-statistical-specialized}
      |\Psi_{T+1,s}|
      \le
      4C_\omega^2
      \frac{
          \bigl(D\sqrt{\iota_T}+\sqrt\lambda\bigr)^2
      }{
          D^2
      }
      2^{2s}N\iota_T.
  \end{equation}
  \end{lemma}
  
  \begin{myproof}
  Fix \(s<S\). If \(\Psi_{T+1,s}=\emptyset\), both conclusions are
  immediate. Suppose that \(\Psi_{T+1,s}\ne\emptyset\).
  
  Every round \(t\in\Psi_{T+1,s}\) is assigned to the nonterminal layer
  \(s\). Under Algorithm~\ref{alg:lpdm_pilot_residual}, selection at such
  a layer can occur only through the under-explored branch. Therefore,
  \(
      \operatorname{wid}_{t,s}^{j_t}(w_t)
      >
      BD\,2^{-s}\) for
      \(
      t\in\Psi_{T+1,s}.
  \)
  Summing this inequality over \(t\in\Psi_{T+1,s}\) and applying
  Lemma~\ref{lem:bound_r_layered} gives
  \[
  \begin{aligned}
      BD\,2^{-s}|\Psi_{T+1,s}|
      &<
      \sum_{t\in\Psi_{T+1,s}}
      \operatorname{wid}_{t,s}^{j_t}(w_t) \\
      &\le
      C_\omega B\Bigl[
          \bigl(D\sqrt{\iota_T}+\sqrt\lambda\bigr)
          \sqrt{N|\Psi_{T+1,s}|\,\iota_T}
          +
          (h^\beta+\eta^2)
          |\Psi_{T+1,s}|\sqrt{\iota_T}
      \Bigr].
  \end{aligned}
  \]
  Since \(B>0\), dividing by \(B\) and weakening the resulting strict
  inequality proves
  \eqref{eq:layer-count-preliminary-specialized}.
  
  Suppose now that
  \eqref{eq:layer-count-bias-separation-specialized} holds. Then
  \[
      C_\omega(h^\beta+\eta^2)\sqrt{\iota_T}
      \le
      \frac{D}{2}2^{-s}.
  \]
  It follows from
  \eqref{eq:layer-count-preliminary-specialized} that
  \[
  \begin{aligned}
      \frac{D}{2}2^{-s}|\Psi_{T+1,s}|
      &\le
      \left[
          D2^{-s}
          -
          C_\omega(h^\beta+\eta^2)\sqrt{\iota_T}
      \right]
      |\Psi_{T+1,s}| \le
      C_\omega
      \bigl(D\sqrt{\iota_T}+\sqrt\lambda\bigr)
      \sqrt{N|\Psi_{T+1,s}|\,\iota_T}.
  \end{aligned}
  \]
  Both sides are nonnegative. Squaring this inequality gives
  \[
      \frac{D^2}{4}2^{-2s}|\Psi_{T+1,s}|^2
      \le
      C_\omega^2
      \bigl(D\sqrt{\iota_T}+\sqrt\lambda\bigr)^2
      N|\Psi_{T+1,s}|\iota_T.
  \]
  Because \(|\Psi_{T+1,s}|>0\), canceling this factor and rearranging
  yields
  \[
      |\Psi_{T+1,s}|
      \le
      4C_\omega^2
      \frac{
          \bigl(D\sqrt{\iota_T}+\sqrt\lambda\bigr)^2
      }{
          D^2
      }
      2^{2s}N\iota_T,
  \]
  which proves
  \eqref{eq:layer-count-statistical-specialized}.
  \end{myproof}

  \noindent
  \textbf{Proof of \cref{prop:ldp_step_regret_discrete_fixed_residual}.}
  
    \begin{myproof}
    We proceed in five steps. Throughout the proof, \(C\) denotes a finite
    constant depending only on the fixed problem primitives and on
    \(C_\psi,C_z,C_l\); its value may increase from line to line.
    
    \paragraph{Step 1: Reduction to a dyadic layer sum.}
    Note that, the sets \(\Psi_{T+1,1},\ldots,\Psi_{T+1,S}\) are disjoint and contain
    all LDP rounds. Therefore,
    \begin{equation}
    \label{eq:total-ldp-layer-count}
        \sum_{s=1}^S|\Psi_{T+1,s}|
        \le
        T.
    \end{equation}
    By Lemma~\ref{lem:ldp_revenue_gap_discrete_fixed_residual}, every LDP
    round \(t\) satisfies
    \[
        \mathsf{Rev}(\bm x_t,p_t^\star)
        -
        \mathsf{Rev}(\bm x_t,p_t)
        \le
        \frac{L_{\mathrm{Rev}}}{\sqrt T}
        +
        8BD\,2^{-s_t}.
    \]
    Summing over the LDP rounds and using their layer partition gives
    \begin{equation}
    \label{eq:regret-sum-by-layer-corrected}
    \begin{aligned}
    &\sum_{t\in[T]\setminus\mathcal T^{\rm exp}}
    \left[
        \mathsf{Rev}(\bm x_t,p_t^\star)
        -
        \mathsf{Rev}(\bm x_t,p_t)
    \right]                                                  \le
        L_{\mathrm{Rev}}\sqrt T
        +
        8BD\sum_{s=1}^S2^{-s}|\Psi_{T+1,s}|.
    \end{aligned}
    \end{equation}
    
    \paragraph{Step 2: Nonterminal layers dominated by statistical uncertainty.}
    For every nonterminal layer \(s<S\) satisfying
    \eqref{eq:layer-count-bias-separation-specialized},
    Lemma~\ref{lem:bound_Psi_layered} gives
    \[
        |\Psi_{T+1,s}|
        \le
        4C_\omega^2
        \frac{\bigl(D\sqrt{\iota_T}+\sqrt\lambda\bigr)^2}{D^2}
        2^{2s}N\iota_T.
    \]
    The deterministic bound \(|\Psi_{T+1,s}|\le T\) also holds. Note that for every \(a>0\), we have
    \begin{equation}
    \label{eq:generic-dyadic-min-bound}
        \sum_{s\ge1}2^{-s}\min\{T,a2^{2s}\}
        \le
        4\sqrt{aT},
    \end{equation}
    because if \(a\ge T\), the left-hand side is at most
    \(T\sum_{s\ge1}2^{-s}\le T\le\sqrt{aT}\), while if \(a<T\), choose an
    integer \(s_0\ge0\) such that
    \(2^{s_0}\le\sqrt{T/a}<2^{s_0+1}\), then
    \[
    \begin{aligned}
        \sum_{s\ge1}2^{-s}\min\{T,a2^{2s}\}
        &\le
        a\sum_{s\le s_0}2^s
        +
        T\sum_{s>s_0}2^{-s}                                  \le
        2a2^{s_0}+T2^{-s_0}
        \le
        4\sqrt{aT}.
    \end{aligned}
    \]
    Applying \eqref{eq:generic-dyadic-min-bound} with
    \[
        a
        =
        4C_\omega^2
        \frac{\bigl(D\sqrt{\iota_T}+\sqrt\lambda\bigr)^2}{D^2}
        N\iota_T
    \]
    yields
    \begin{equation}
    \label{eq:statistical-layer-sum-corrected}
    \begin{aligned}
    &\sum_{\substack{s<S:\
    D2^{-s}\ge2C_\omega(h^\beta+\eta^2)\sqrt{\iota_T}}}
    2^{-s}|\Psi_{T+1,s}|                                     \le 
    4\sqrt{aT}                                             =
    8C_\omega
    \frac{
        D\sqrt{\iota_T}+\sqrt\lambda
    }{
        D
    }
    \sqrt{NT\iota_T}.
    \end{aligned}
    \end{equation}
    
    \paragraph{Step 3: Nonterminal layers dominated by approximation bias.}
    For every remaining nonterminal layer,
    \[
        D2^{-s}
        <
        2C_\omega(h^\beta+\eta^2)\sqrt{\iota_T},
    \]
    and hence
    \[
        2^{-s}
        <
        \frac{2C_\omega(h^\beta+\eta^2)\sqrt{\iota_T}}{D}.
    \]
    Consequently,
    \begin{equation}
    \label{eq:bias-layer-sum-corrected}
    \begin{aligned}
    \sum_{\substack{s<S:\\
    D2^{-s}<
    2C_\omega(h^\beta+\eta^2)\sqrt{\iota_T}}}
    2^{-s}|\Psi_{T+1,s}|                                     
    &\le
    \frac{
        2C_\omega
        (h^\beta+\eta^2)
        \sqrt{\iota_T}
    }{
        D
    }
    \sum_{s<S}
    |\Psi_{T+1,s}|                                           \le
    \frac{
        2C_\omega
        (h^\beta+\eta^2)
        \sqrt{\iota_T}
    }{
        D
    }
    T.
    \end{aligned}
    \end{equation}
    \paragraph{Step 4: Terminal layer and completion of the explicit bound.}
    For the terminal layer,
    \begin{equation}
    \label{eq:terminal-layer-sum-corrected}
        2^{-S}|\Psi_{T+1,S}|
        \le
        T2^{-S}.
    \end{equation}
    Combining
    \eqref{eq:statistical-layer-sum-corrected},
    \eqref{eq:bias-layer-sum-corrected}, and
    \eqref{eq:terminal-layer-sum-corrected}, we obtain
    \[
    \begin{aligned}
        \sum_{s=1}^S2^{-s}|\Psi_{T+1,s}|
        \le{}&
        8C_\omega
        \frac{
            D\sqrt{\iota_T}+\sqrt\lambda
        }{
            D
        }
        \sqrt{NT\iota_T}                                     +
        \frac{
            2C_\omega
            (h^\beta+\eta^2)
            \sqrt{\iota_T}
        }{
            D
        }
        T
        +
        T2^{-S}.
    \end{aligned}
    \]
    Substituting this inequality into
    \eqref{eq:regret-sum-by-layer-corrected} gives
    \[
    \begin{aligned}
    &\sum_{t\in[T]\setminus\mathcal T^{\rm exp}}
    \left[
        \mathsf{Rev}(\bm x_t,p_t^\star)
        -
        \mathsf{Rev}(\bm x_t,p_t)
    \right]                                                           \\
    &\le
    64C_\omega B
    \bigl(D\sqrt{\iota_T}+\sqrt\lambda\bigr)
    \sqrt{NT\iota_T}
    +
    16C_\omega BT
    (h^\beta+\eta^2)\sqrt{\iota_T}+
    8BDT2^{-S}
    +
    L_{\mathrm{Rev}}\sqrt T.
    \end{aligned}
    \]
    For \(S=\max\left\{1,\left\lceil\log_2\sqrt T\right\rceil\right\}\),
    we have \(T2^{-S}\le\sqrt T\). Therefore,
    \eqref{eq:ldp_step_regret_with_S} follows.
    
    \paragraph{Step 5: Simplified rate under algorithmic tuning.} Let $\lambda>0$ be fixed independently of \(T\), let $\eta^2\le c_\eta h^\beta$, and let $N=\left\lceil T^{\frac{1}{2\beta+1}}\right\rceil$. Since \(T\ge1\), we have \(T^{\frac{1}{2\beta+1}} \le N
    \le 2T^{\frac{1}{2\beta+1}}\). It follows that \(\sqrt{NT} \le \sqrt{2}  T^{\frac{\beta+1}{2\beta+1}},\) and, because \(h=2B/N\),
    \[
    \begin{aligned}
        T(h^\beta+\eta^2)
        \le
        (1+c_\eta)Th^\beta                                    =
        (1+c_\eta)(2B)^\beta TN^{-\beta}                      \le
        (1+c_\eta)(2B)^\beta
        T^{\frac{\beta+1}{2\beta+1}}.
    \end{aligned}
    \]
    
    Moreover,
    \(
        \sqrt T
        \le
        T^{\frac{\beta+1}{2\beta+1}}
    \)
    for every finite \(\beta\ge1\). By the definition of \(\iota_T\), we have \(\iota_T>1\). Since
    \(\lambda>0\) is fixed independently of \(T\),
    \[
    \begin{aligned}
        \bigl(D\sqrt{\iota_T}+\sqrt{\lambda}\bigr)\sqrt{\iota_T}
        =
        D\iota_T+\sqrt{\lambda\iota_T}                               \le
        (D+\sqrt{\lambda})\iota_T,
    \end{aligned}
    \]
    and \(\sqrt{\iota_T}\le\iota_T\).
    
    Applying these inequalities to
    \eqref{eq:ldp_step_regret_with_S} gives
    \[
    \begin{aligned}
    &\sum_{t\in[T]\setminus\mathcal T^{\rm exp}}
    \left[
        \mathsf{Rev}(\bm x_t,p_t^\star)
        -
        \mathsf{Rev}(\bm x_t,p_t)
    \right]                                                          \\
    &\quad\le
    \Bigl[
        64\sqrt{2}\,C_\omega B
        \left(D+\sqrt{\lambda}\right)
        +
        16C_\omega B(1+c_\eta)(2B)^\beta
        +
        8BD
        +
        L_{\mathrm{Rev}}
    \Bigr]
    \iota_T
    T^{\frac{\beta+1}{2\beta+1}}.
    \end{aligned}
    \]
    One may take
    \begin{equation*}
    \label{eq:C_LDP}
    \begin{aligned}
        C_{\mathrm{LDP}}
        :={}&
        64\sqrt{2}\,C_\omega B
        \left(D+\sqrt{\lambda}\right)
        +
        16C_\omega B(1+c_\eta)(2B)^\beta
        +
        8BD
        +
        L_{\mathrm{Rev}}.
    \end{aligned}
    \end{equation*}
    This proves
    \eqref{eq:ldp_step_regret_exact_rate}.
    \end{myproof}

    \noindent
    \textbf{Proof of \cref{lem:pilot_exploration_regret}.}

      \begin{myproof}
      Fix an arbitrary realization of the policy. If
      \(\mathcal T^{\rm exp}=\emptyset\), all conclusions are immediate.
      Otherwise, enumerate the pilot-exploration rounds chronologically as
      \[
          \mathcal T^{\rm exp}
          =
          \left\{
              \tau_1<\cdots<
              \tau_{|\mathcal T^{\rm exp}|}
          \right\}.
      \]
      
      \paragraph{Step 1: Leverage on exploration rounds.}
      For every \(k=1,\ldots,|\mathcal T^{\rm exp}|\), period \(\tau_k\) is a pilot-exploration round.
      Therefore, by the uncertainty gate in
      Algorithm~\ref{alg:adaptive_pilot},
      $
          \gamma_T
          \|\bm x_{\tau_k}\|_{M_{\tau_k}^{-1}}
          >
          \eta.
      $
      Since \(\gamma_T>0\), squaring both sides gives
      \begin{equation}
      \label{eq:pilot-exploration-leverage-lower}
          \|\bm x_{\tau_k}\|_{M_{\tau_k}^{-1}}^2
          >
          \frac{\eta^2}{\gamma_T^2}.
      \end{equation}
      
      \paragraph{Step 2: Exact determinant recursion.}
      The matrix \(M_t\) is updated only on pilot-exploration rounds.
      In particular, \(M_{\tau_1}=M_1=I_d\), and, for every \(k=1,\ldots,|\mathcal T^{\rm exp}|-1\),
      \[
          M_{\tau_{k+1}}
          =
          M_{\tau_k}
          +
          \bm x_{\tau_k}\bm x_{\tau_k}^\top,
      \]
      because \(M_t\) remains unchanged on all LDP rounds strictly between
      \(\tau_k\) and \(\tau_{k+1}\). Similarly,
      \[
          M_{T+1}
          =
          M_{\tau_{|\mathcal T^{\rm exp}|}}
          +
          \bm x_{\tau_{|\mathcal T^{\rm exp}|}}
          \bm x_{\tau_{|\mathcal T^{\rm exp}|}}^\top.
      \]
      
      Since \(M_{\tau_k}\succeq I_d\), it is positive definite. The matrix
      determinant lemma therefore gives
      \[
      \begin{aligned}
          \det\left(
              M_{\tau_k}
              +
              \bm x_{\tau_k}\bm x_{\tau_k}^\top
          \right)
          &=
          \det(M_{\tau_k})
          \left(
              1+
              \bm x_{\tau_k}^\top
              M_{\tau_k}^{-1}
              \bm x_{\tau_k}
          \right)                                               =
          \det(M_{\tau_k})
          \left(
              1+
              \|\bm x_{\tau_k}\|_{M_{\tau_k}^{-1}}^2
          \right).
      \end{aligned}
      \]
      Iterating this identity and using \(\det(M_{\tau_1})=\det(I_d)=1\)
      yields
      \begin{equation}
      \label{eq:pilot-determinant-recursion}
          \log\det(M_{T+1})
          =
          \sum_{k=1}^{|\mathcal T^{\rm exp}|}
          \log\left(
              1+
              \|\bm x_{\tau_k}\|_{M_{\tau_k}^{-1}}^2
          \right).
      \end{equation}
      
      \paragraph{Step 3: Lower bound on the determinant growth.}
      By \eqref{eq:pilot-exploration-leverage-lower}, every summand in
      \eqref{eq:pilot-determinant-recursion} satisfies
      \[
          \log\left(
              1+
              \|\bm x_{\tau_k}\|_{M_{\tau_k}^{-1}}^2
          \right)
          >
          \log\left(
              1+\frac{\eta^2}{\gamma_T^2}
          \right).
      \]
      Summing over \(k\) gives,
      \begin{equation}
      \label{eq:pilot-determinant-lower}
          \left|\mathcal T^{\rm exp}\right|
          \log\left(
              1+\frac{\eta^2}{\gamma_T^2}
          \right)
          <
          \log\det(M_{T+1}).
      \end{equation}

      \paragraph{Step 4: Upper bound on the determinant.}
      By the update rule,
      \[
          M_{T+1}
          =
          I_d
          +
          \sum_{t\in\mathcal T^{\rm exp}}
          \bm x_t\bm x_t^\top.
      \]
      Therefore, we have
      \[
      \begin{aligned}
          \operatorname{tr}(M_{T+1})
          =
          d+
          \sum_{t\in\mathcal T^{\rm exp}}
          \|\bm x_t\|_2^2  \le
          d+TC_x^2.
      \end{aligned}
      \]
      
      Since \(M_{T+1}\) is positive definite, by the arithmetic--geometric mean inequality applied to the eigenvalues of \(M_{T+1}\),
      \[
          \det(M_{T+1})^{1/d}
          \le
          \frac{\operatorname{tr}(M_{T+1})}{d}.
      \]

      Consequently,
      \begin{equation}
      \label{eq:pilot-determinant-upper}
      \begin{aligned}
          \log\det(M_{T+1})
          \le
          d\log\left(
              \frac{\operatorname{tr}(M_{T+1})}{d}
          \right)                                                       
          \le
          d\log\left(
              1+\frac{TC_x^2}{d}
          \right).
      \end{aligned}
      \end{equation}
      
      Combining
      \eqref{eq:pilot-determinant-lower} and
      \eqref{eq:pilot-determinant-upper} yields
      \[
          \left|\mathcal T^{\rm exp}\right|
          <
          \frac{
              d\log\left(
                  1+TC_x^2/d
              \right)
          }{
              \log\left(
                  1+\eta^2/\gamma_T^2
              \right)
          }.
      \]
      Together with the trivial bound \(\left|\mathcal T^{\rm exp}\right|\le T\), this proves the first
      inequality in \eqref{eq:pilot-exploration-count}.
      
      For every \(x\ge0\),
      $
          \log(1+x)\ge x/(1+x).
      $
      Applying this inequality with
      $
          x=\eta^2/\gamma_T^2
      $
      gives
      \[
          \frac{1}{
              \log\left(
                  1+\eta^2/\gamma_T^2
              \right)
          }
          \le
          1+\frac{\gamma_T^2}{\eta^2}.
      \]
      This proves the second inequality in
      \eqref{eq:pilot-exploration-count}.
      
      \paragraph{Step 5: Pilot-exploration regret.}
      Since \(0\le g(u)\le D\) and \(0\le p\le B\),
      $
          0
          \le
          \mathsf{Rev}(\bm x,p)
          =
          p\,g(p-\bm x^\top\bm\theta_\star)
          \le
          BD.
      $
      Therefore, for every \(t\in\mathcal T^{\rm exp}\),
      \[
      \begin{aligned}
          0
          \le
          \mathsf{Rev}(\bm x_t,p_t^\star)
          -
          \mathsf{Rev}(\bm x_t,p_t)                                    
          \le
          \mathsf{Rev}(\bm x_t,p_t^\star)
          \le
          BD.
      \end{aligned}
      \]
      Summing over \(t\in\mathcal T^{\rm exp}\) and applying
      \eqref{eq:pilot-exploration-count} proves
      \eqref{eq:pilot-exploration-regret}.
      \end{myproof}

\noindent
\textbf{Proof of \cref{thm:regret_upper_bound}.}
       
\begin{myproof}
  We work throughout on the uniform confidence event in
  Proposition~\ref{prop:ucb}, which occurs with probability at least
  \(1-\delta\).
  
  \paragraph{Step 1: Decomposition of the regret.}
  The pilot-exploration rounds and the LDP rounds form a disjoint
  partition of the horizon.
  Therefore,
  \begin{equation}
  \label{eq:regret-exploration-ldp-decomposition}
  \begin{aligned}
      \mathrm{Reg}(T)
      =
      \sum_{t\in\mathcal T^{\rm exp}}
      \left[
          \mathsf{Rev}(\bm x_t,p_t^\star)
          -
          \mathsf{Rev}(\bm x_t,p_t)
      \right]  +
      \sum_{t\in[T]\setminus\mathcal T^{\rm exp}}
      \left[
          \mathsf{Rev}(\bm x_t,p_t^\star)
          -
          \mathsf{Rev}(\bm x_t,p_t)
      \right].
  \end{aligned}
  \end{equation}
  
  \paragraph{Step 2: Finite-sample bound.}
  Lemma~\ref{lem:pilot_exploration_regret} gives the pathwise
  \begin{equation}
  \label{eq:pilot-regret-main-theorem}
  \begin{aligned}
      \sum_{t\in\mathcal T^{\rm exp}}
      \left[
          \mathsf{Rev}(\bm x_t,p_t^\star)
          -
          \mathsf{Rev}(\bm x_t,p_t)
      \right] \le
      BD
      \min\left\{
          T,\,
          \frac{
              d\log\left(
                  1+TC_x^2/d
              \right)
          }{
              \log\left(
                  1+\eta^2/\gamma_T^2
              \right)
          }
      \right\}.
  \end{aligned}
  \end{equation}
  Proposition~\ref{prop:ldp_step_regret_discrete_fixed_residual} and the
  choice 
  \(
      S=
      \max\left\{
          1,
          \left\lceil\log_2\sqrt T\right\rceil
      \right\}
  \)
  give
  \begin{equation}
  \label{eq:ldp-regret-main-theorem}
  \begin{aligned}
  &\sum_{t\in[T]\setminus\mathcal T^{\rm exp}}
  \left[
      \mathsf{Rev}(\bm x_t,p_t^\star)
      -
      \mathsf{Rev}(\bm x_t,p_t)
  \right]                                                        \\
  &\le
  64C_\omega B
  \bigl(D\sqrt{\iota_T}+\sqrt{\lambda}\bigr)
  \sqrt{NT\iota_T}                                         +
  16C_\omega BT
  (h^\beta+\eta^2)\sqrt{\iota_T}
  +
  \bigl(8BD+L_{\mathrm{Rev}}\bigr)\sqrt T.
  \end{aligned}
  \end{equation}
  Substituting \eqref{eq:pilot-regret-main-theorem} and
  \eqref{eq:ldp-regret-main-theorem} into
  \eqref{eq:regret-exploration-ldp-decomposition} proves
  \eqref{eq:regret_bound_discrete_fixed_pilot_residual}.
  
  \paragraph{Step 3: LDP regret under the prescribed tuning.}
  Let 
  \(N=
      \left\lceil
          T^{\frac{1}{2\beta+1}}
      \right\rceil\) and \(\eta^2=\min\{h^2,h^\beta\}\).
  Because \(\eta^2\le h^2\), we have
  \(
      \eta\le h.
  \)
  Moreover, since \(\eta^2\le h^\beta\), we have \(h^\beta+\eta^2
      \le
      2h^\beta\).
  Thus the prescribed tuning satisfies the conditions used in
  Proposition~\ref{prop:ucb} and
  Proposition~\ref{prop:ldp_step_regret_discrete_fixed_residual}.
  
  Since \(T\ge1\), we have
  \[
      T^{\frac{1}{2\beta+1}}
      \le
      N
      \le
      T^{\frac{1}{2\beta+1}}+1
      \le
      2T^{\frac{1}{2\beta+1}}.
  \]
  It follows that
  \begin{equation}
  \label{eq:tuned-sqrt-NT}
  \begin{aligned}
      \sqrt{NT}
      &\le
      \sqrt{2}\,
      T^{\frac12+\frac{1}{2(2\beta+1)}}                     =
      \sqrt{2}\,
      T^{\frac{\beta+1}{2\beta+1}}.
  \end{aligned}
  \end{equation}
  
  Since \(N\ge T^{1/(2\beta+1)}\) and \(h=2B/N\),
  \begin{equation}
  \label{eq:tuned-approximation-term}
  \begin{aligned}
      Th^\beta
      &=
      T\left(\frac{2B}{N}\right)^\beta                       \le
      (2B)^\beta
      T^{1-\frac{\beta}{2\beta+1}}                           =
      (2B)^\beta
      T^{\frac{\beta+1}{2\beta+1}}.
  \end{aligned}
  \end{equation}
  Furthermore, since
  \(
      (\beta+1)/(2\beta+1)>1/2,
  \)
  we have \(\sqrt T
      \le
      T^{\frac{\beta+1}{2\beta+1}}\).
  By the definition of \(\iota_T\),
  \(
      \iota_T
      \ge
      2\log 4
      >
      1.
  \)
  Therefore,
  \(
      \sqrt{\iota_T}\le\iota_T.
  \)
  In addition,
  \begin{align}
  &\bigl(D\sqrt{\iota_T}+\sqrt{\lambda}\bigr)
  \sqrt{NT\iota_T}                                             =
  \bigl(D\iota_T+\sqrt{\lambda\iota_T}\bigr)\sqrt{NT}         \le
  (D+1)(1+\sqrt{\lambda})\iota_T\sqrt{NT}.
  \label{eq:tuned-statistical-term}
  \end{align}
  Combining \eqref{eq:tuned-statistical-term} with
  \eqref{eq:tuned-sqrt-NT} gives
  \[
  \begin{aligned}
  &64C_\omega B
  \bigl(D\sqrt{\iota_T}+\sqrt{\lambda}\bigr)
  \sqrt{NT\iota_T}                                         \le
  64\sqrt{2}\,C_\omega B(D+1)
  (1+\sqrt{\lambda})\iota_T
  T^{\frac{\beta+1}{2\beta+1}}.
  \end{aligned}
  \]
  
  Next, by \(h^\beta+\eta^2\le 2h^\beta\),
  \eqref{eq:tuned-approximation-term}, and
  \(\sqrt{\iota_T}\le\iota_T\),
  \[
  \begin{aligned}
  &16C_\omega BT
  (h^\beta+\eta^2)\sqrt{\iota_T}                             \le
  32C_\omega BTh^\beta\iota_T                                \le
  32C_\omega B(2B)^\beta
  (1+\sqrt{\lambda})\iota_T
  T^{\frac{\beta+1}{2\beta+1}}.
  \end{aligned}
  \]
  Finally, by \(\sqrt T
      \le
      T^{\frac{\beta+1}{2\beta+1}}\) and
  \((1+\sqrt{\lambda})\iota_T\ge1\),
  \[
  \begin{aligned}
      \bigl(8BD+L_{\mathrm{Rev}}\bigr)\sqrt T
      \le
      \bigl(8BD+L_{\mathrm{Rev}}\bigr)
      (1+\sqrt{\lambda})\iota_T
      T^{\frac{\beta+1}{2\beta+1}}.
  \end{aligned}
  \]
  Consequently, we have
  \begin{equation}
  \label{eq:ldp-regret-prescribed-tuning}
  \begin{aligned}
  &\sum_{t\in[T]\setminus\mathcal T^{\rm exp}}
  \left[
      \mathsf{Rev}(\bm x_t,p_t^\star)
      -
      \mathsf{Rev}(\bm x_t,p_t)
  \right]                                                        \\
  &\quad\le
  \Bigl[
      64\sqrt{2}\,C_\omega B(D+1)
      +
      32C_\omega B(2B)^\beta
      +
      8BD
      +
      L_{\mathrm{Rev}}
  \Bigr]                                                   
  (1+\sqrt{\lambda})\iota_T
  T^{\frac{\beta+1}{2\beta+1}}.
  \end{aligned}
  \end{equation}
  
  \paragraph{Step 4: Pilot-exploration regret under the prescribed tuning.}
  For every \(x\ge0\),
  \(
      \log(1+x)
      \ge
      x/(1+x).
  \)
  Applying this inequality with
  \(x=\eta^2/\gamma_T^2\) gives
  \[
      \frac{1}{
          \log\left(
              1+\eta^2/\gamma_T^2
          \right)
      }
      \le
      1+\frac{\gamma_T^2}{\eta^2}.
  \]
  Thus, by \eqref{eq:pilot-regret-main-theorem},
  \begin{equation}
  \label{eq:pilot-regret-simplified-main-theorem}
  \begin{aligned}
  &\sum_{t\in\mathcal T^{\rm exp}}
  \left[
      \mathsf{Rev}(\bm x_t,p_t^\star)
      -
      \mathsf{Rev}(\bm x_t,p_t)
  \right]                                                  \le
  BDd
  \left(
      1+\frac{\gamma_T^2}{\eta^2}
  \right)
  \log\left(
      1+\frac{TC_x^2}{d}
  \right).
  \end{aligned}
  \end{equation}
  
  Since
  \(
      N\le2T^{\frac{1}{2\beta+1}},
  \)
  we have
  \(
      h=2B/N
      \ge
      B T^{-\frac{1}{2\beta+1}}.
  \)
  Because
  \(
      \eta^2=\min\{h^2,h^\beta\},
  \)
  it follows that
  \begin{align}
      \frac1{\eta^2}
      &=
      \max\{h^{-2},h^{-\beta}\}                             \le
      \max\left\{
          B^{-2}T^{\frac{2}{2\beta+1}},
          B^{-\beta}T^{\frac{\beta}{2\beta+1}}
      \right\}                                               \le
      \max\{B^{-2},B^{-\beta}\}
      T^{\frac{\beta+1}{2\beta+1}}.
  \label{eq:tuned-eta-inverse}
  \end{align}
  
  Since \(T^{(\beta+1)/(2\beta+1)}\ge1\),
  \begin{align*}
      1+\frac{\gamma_T^2}{\eta^2}
      &\le
      \left[
          1+
          \max\{B^{-2},B^{-\beta}\}\gamma_T^2
      \right]
      T^{\frac{\beta+1}{2\beta+1}}                          \le
      \max\{1,B^{-2},B^{-\beta}\}
      (1+\gamma_T^2)
      T^{\frac{\beta+1}{2\beta+1}}.
  \end{align*}
  Substituting this inequality into
  \eqref{eq:pilot-regret-simplified-main-theorem} gives
  \begin{equation}
  \label{eq:pilot-regret-prescribed-tuning}
  \begin{aligned}
  &\sum_{t\in\mathcal T^{\rm exp}}
  \left[
      \mathsf{Rev}(\bm x_t,p_t^\star)
      -
      \mathsf{Rev}(\bm x_t,p_t)
  \right]                                                  \le
  BD\max\{1,B^{-2},B^{-\beta}\}
  d(1+\gamma_T^2)
  \log\left(
      1+\frac{TC_x^2}{d}
  \right)
  T^{\frac{\beta+1}{2\beta+1}}.
  \end{aligned}
  \end{equation}
  
  \paragraph{Step 5: Completion of the tuned bound.}
  Define
  \[
  \begin{aligned}
      C
      :=
      \max\Bigl\{&
          BD\max\{1,B^{-2},B^{-\beta}\},64\sqrt{2}\,C_\omega B(D+1)
          +
          32C_\omega B(2B)^\beta
          +
          8BD
          +
          L_{\mathrm{Rev}}
      \Bigr\}.
  \end{aligned}
  \]
  By Lemma~\ref{lem:bound_r_layered} and
  Lemma~\ref{lem:ldp_revenue_gap_discrete_fixed_residual}, \(C\) is
  finite and depends only on the fixed problem primitives.
  
  Combining
  \eqref{eq:ldp-regret-prescribed-tuning} and
  \eqref{eq:pilot-regret-prescribed-tuning} yields
  \[
  \begin{aligned}
      \mathrm{Reg}(T)
      \le
      C
      \left[
          d(1+\gamma_T^2)
          \log\left(
              1+\frac{TC_x^2}{d}
          \right)
          +
          (1+\sqrt{\lambda})\iota_T
      \right]
      T^{\frac{\beta+1}{2\beta+1}},
  \end{aligned}
  \]
  which proves
  \eqref{eq:regret_bound_discrete_fixed_pilot_residual_rate}.
  
  Finally, 
  the definitions of \(\gamma_T\) and \(\iota_T\) imply
  $
      \gamma_T^2
      =
      \mathcal O(\log T),
      \iota_T
      =
      \mathcal O(\log T).
  $
  Consequently,
  \[
      d\bigl(1+\gamma_T^2\bigr)
      \log\left(
          1+\frac{TC_x^2}{d}
      \right)
      +
      (1+\sqrt{\lambda})\iota_T
      =
      \mathcal O(\log^2 T),
  \]
  and it follows that
  \[
      \mathrm{Reg}(T)
      =
      \widetilde{\mathcal O}\!\left(
          T^{\frac{\beta+1}{2\beta+1}}
      \right)
  \]
  with probability at least \(1-\delta\).
  This completes the proof. 
  \end{myproof}

  \noindent \textbf{Proof of Lemma~\ref{lem:lower_bound_flat_baseline}}.

  \begin{myproof}
  We construct \(g_0\) in four steps. We first build a nonincreasing
  envelope that generates a flat globally optimal revenue region and has
  more mass than required by the zero-mean condition. We then smooth its
  only kink without changing the flat region, truncate its right tail so
  that the zero-mean condition holds exactly, and finally verify the
  remaining properties and construct the compensation function \(\chi\).
  
  \paragraph{Step 1: Constructing a flat-revenue envelope.}
  We first construct a nonincreasing function whose associated revenue is
  constant at its global maximum over an interior price interval. The
  construction initially has more total mass than required by the
  zero-mean condition; this excess mass will be removed in Step~3.

  Define
  \[
      \bar g(u)
      :=
      \min\left\{
          D,\,
          \frac{
              D((\bm x^\circ)^\top\bm\theta^\circ-B_\epsilon/2)
          }{
              (\bm x^\circ)^\top\bm\theta^\circ+u
          }
      \right\},
      \qquad
      u\in[-B_\epsilon,B_\epsilon],
  \]
  where the ratio is interpreted as \(+\infty\) when \((\bm x^\circ)^\top\bm\theta^\circ+u=0\).
  Then
  \[
      \bar g(u)
      =
      D, \,\
      \forall
      u\le-\frac{B_\epsilon}{2} \quad \text{and} \quad \bar g(u)
      =
      \frac{
          D((\bm x^\circ)^\top\bm\theta^\circ-B_\epsilon/2)
      }{
          (\bm x^\circ)^\top\bm\theta^\circ+u
      }, \,\
      \forall
      u>-\frac{B_\epsilon}{2}.
  \]
  Consequently,
  \begin{align}
      \int_{-B_\epsilon}^{B_\epsilon}
      \bar g(u)\,\mathrm du
      &=
      \frac{DB_\epsilon}{2}
      +
      D\left(
          (\bm x^\circ)^\top\bm\theta^\circ-\frac{B_\epsilon}{2}
      \right)
      \log\left(
          1+
          \frac{
              3B_\epsilon
          }{
              2((\bm x^\circ)^\top\bm\theta^\circ-B_\epsilon/2)
          }
      \right).
  \label{eq:lower-bound-envelope-integral}
  \end{align}
  
  For every \(a>0\), the function \(x\log\left(1+a/x\right)\) is strictly increasing in \(x\) on \((0,\infty)\), because
  \[
      \frac{\mathrm d}{\mathrm dx}
      \left[
          x\log\left(1+\frac{a}{x}\right)
      \right]
      =
      \log\left(1+\frac{a}{x}\right)
      -
      \frac{a}{x+a}
      >
      0.
  \]
  Since \((\bm x^\circ)^\top\bm\theta^\circ-B_\epsilon/2
      \ge
      B_\epsilon/2\), take \(a=3B_\epsilon/2\) and
  equation~\eqref{eq:lower-bound-envelope-integral} implies
  \begin{equation}
  \label{eq:lower-bound-envelope-excess}
      \int_{-B_\epsilon}^{B_\epsilon}
      \bar g(u)\,\mathrm du
      \ge
      \frac{DB_\epsilon}{2}
      +
      \frac{DB_\epsilon}{2}\log4
      >
      DB_\epsilon.
  \end{equation}
  
  Set
  \[
      p_L
      :=
      (\bm x^\circ)^\top\bm\theta^\circ-\frac{3B_\epsilon}{8},
      \quad
      p_U
      :=
      (\bm x^\circ)^\top\bm\theta^\circ-\frac{B_\epsilon}{4}.
  \]
  Because
  \((\bm x^\circ)^\top\bm\theta^\circ\in[B_\epsilon,B-B_\epsilon]\), we have \( 0<p_L<p_U<B\) and
  \[
      [p_L-(\bm x^\circ)^\top\bm\theta^\circ,p_U-(\bm x^\circ)^\top\bm\theta^\circ]
      =
      \left[
          -\frac{3B_\epsilon}{8},\,
          -\frac{B_\epsilon}{4}
      \right]
      \subset
      (-B_\epsilon,B_\epsilon).
  \]

  For every \(p\in[p_L,p_U]\), we have 
  \(
      p-(\bm x^\circ)^\top\bm\theta^\circ
      \ge
      -3B_\epsilon/8
      >
      -B_\epsilon/2.
  \)
  Hence,
  \begin{equation}
  \label{eq:lower-bound-envelope-flat-revenue}
  \begin{aligned}
      p\,
      \bar g\!\left(
          p-(\bm x^\circ)^\top\bm\theta^\circ
      \right)
      &=
      p\,
      \frac{
          D\bigl(
              (\bm x^\circ)^\top\bm\theta^\circ-B_\epsilon/2
          \bigr)
      }{p} &=
      D\left(
          (\bm x^\circ)^\top\bm\theta^\circ
          -
          \frac{B_\epsilon}{2}
      \right).
  \end{aligned}
  \end{equation}
  Thus, the revenue induced by \(\bar g\) is constant on
  \([p_L,p_U]\).

  \paragraph{Step 2: Smoothing the envelope without changing the flat region.}
  The envelope \(\bar g\) is continuous but has a kink at
  \(-B_\epsilon/2\). We smooth this kink over a short interval while
  preserving monotonicity, remaining below \(\bar g\), and leaving the
  flat-revenue region unchanged.
  
  Fix a nondecreasing infinitely differentiable function
  \(\rho:\mathbb R\to[0,1]\) satisfying
  \[
      \rho(z)=0
      \quad\text{for }z\le0,
      \qquad
      \rho(z)=1
      \quad\text{for }z\ge1,
      \qquad
      \rho'(z)>0
      \quad\text{for }z\in(0,1).
  \]
  
  Choose \(\zeta>0\) sufficiently small that
  \begin{equation}
  \label{eq:lower-bound-zeta-choice}
      \zeta
      <
      B_\epsilon/64
      \quad \text{and} \quad 
      6D\zeta
      <
      \int_{-B_\epsilon}^{B_\epsilon}
      \bar g(u)\,\mathrm du
      -
      DB_\epsilon.
  \end{equation}
  There exists \(z_0
      \in\left(
          -B_\epsilon/2-2\zeta,\,
          -B_\epsilon/2
      \right)\)
  such that
  \begin{align}
  &\int_{-B_\epsilon/2-2\zeta}^{-B_\epsilon/2+2\zeta}
      \frac{
          D((\bm x^\circ)^\top\bm\theta^\circ-B_\epsilon/2)
      }{
          ((\bm x^\circ)^\top\bm\theta^\circ+v)^2
      }
      \rho\left(
          \frac{v-z_0}{\zeta}
      \right)
      \,\mathrm dv
  =
      D-
      \frac{
          D((\bm x^\circ)^\top\bm\theta^\circ-B_\epsilon/2)
      }{
          (\bm x^\circ)^\top\bm\theta^\circ-B_\epsilon/2+2\zeta
      }.
  \label{eq:lower-bound-smoothing-match}
  \end{align}
  Since \((\bm x^\circ)^\top\bm\theta^\circ\ge B_\epsilon\),
  we have \( (\bm x^\circ)^\top\bm\theta^\circ-B_\epsilon/2-2\zeta>
  0,\)
  and the denominator above is strictly positive. Indeed, the left-hand side is continuous in \(z_0\), whereas the
  right-hand side equals
  \[
      \int_{-B_\epsilon/2}^{-B_\epsilon/2+2\zeta}
      \frac{
          D((\bm x^\circ)^\top\bm\theta^\circ-B_\epsilon/2)
      }{
          ((\bm x^\circ)^\top\bm\theta^\circ+v)^2
      }
      \,\mathrm dv.
  \]

  When \(z_0=-B_\epsilon/2-2\zeta\),
  we have
  \[
      \rho\left(
          \frac{v-z_0}{\zeta}
      \right)=1
      \quad
      \text{for every }
      v\ge-\frac{B_\epsilon}{2}-\zeta.
  \]
  Hence the left-hand side of
  \eqref{eq:lower-bound-smoothing-match} is at least
  \[
      \int_{-B_\epsilon/2-\zeta}^{-B_\epsilon/2+2\zeta}
      \frac{
          D((\bm x^\circ)^\top\bm\theta^\circ-B_\epsilon/2)
      }{
          ((\bm x^\circ)^\top\bm\theta^\circ+v)^2
      }
      \,\mathrm dv,
  \]
  which is strictly larger than the right-hand side since the
  additional integral over
  \(
      \left[
          -B_\epsilon/2-\zeta,
          -B_\epsilon/2
      \right]
  \)
  is strictly positive.
  
  In contrast, when \( z_0=-B_\epsilon/2\), the multiplier is zero for \(v\le-B_\epsilon/2\), lies strictly between zero and one for
  \(
      v\in
      \left(
          -B_\epsilon/2,
          -B_\epsilon/2+\zeta
      \right),
  \)
  and equals one for
  \(v\ge-B_\epsilon/2+\zeta\). Therefore, the left-hand side is strictly
  smaller than
  \[
      \int_{-B_\epsilon/2}^{-B_\epsilon/2+2\zeta}
      \frac{
          D((\bm x^\circ)^\top\bm\theta^\circ-B_\epsilon/2)
      }{
          ((\bm x^\circ)^\top\bm\theta^\circ+v)^2
      }
      \,\mathrm dv,
  \]
  which is precisely the right-hand side. Applying the intermediate value theorem
  therefore yields a
  \(
      z_0
      \in
      \left(
          -B_\epsilon/2-2\zeta,
          -B_\epsilon/2
      \right)
  \)
  for which \eqref{eq:lower-bound-smoothing-match} holds.
  
  Define
  \[
      \widetilde g(u)
      :=
      \begin{cases}
          D,
          &
          u\le-B_\epsilon/2-2\zeta,
          \\[1mm]
          \displaystyle
          D-
          \int_{-B_\epsilon/2-2\zeta}^{u}
          \frac{
              D((\bm x^\circ)^\top\bm\theta^\circ-B_\epsilon/2)
          }{
              ((\bm x^\circ)^\top\bm\theta^\circ+v)^2
          }
          \rho\left(
              \frac{v-z_0}{\zeta}
          \right)
          \,\mathrm dv,
          &
          -B_\epsilon/2-2\zeta
          <
          u
          <
          -B_\epsilon/2+2\zeta,
          \\[4mm]
          \displaystyle
          \frac{
              D((\bm x^\circ)^\top\bm\theta^\circ-B_\epsilon/2)
          }{
              (\bm x^\circ)^\top\bm\theta^\circ+u
          },
          &
          u\ge-B_\epsilon/2+2\zeta.
      \end{cases}
  \]
  The choice of \(z_0\) in
  \eqref{eq:lower-bound-smoothing-match} ensures that the middle piece
  and the fractional piece of \(\widetilde g\) have the same value at
  \(u=-B_\epsilon/2+2\zeta\). 
  At the left junction point, the multiplier involving \(\rho\) equals
  zero on a neighborhood, so the middle piece coincides locally with the
  constant \(D\). At the right junction point, the multiplier equals one
  on a neighborhood, and \eqref{eq:lower-bound-smoothing-match} ensures
  that the middle piece coincides locally with the fractional piece.
  Therefore, all derivatives match at both junction points, and
  \(\widetilde g\) is infinitely differentiable on
  \([-B_\epsilon,B_\epsilon]\).

  We verify the monotonicity and range of \(\widetilde g\).
  On the middle interval, we have
  \[
      \widetilde g'(u)
      =
      -
      \frac{
          D((\bm x^\circ)^\top\bm\theta^\circ-B_\epsilon/2)
      }{
          ((\bm x^\circ)^\top\bm\theta^\circ+u)^2
      }
      \rho\left(
          \frac{u-z_0}{\zeta}
      \right)
      \le
      0.
  \]
  Hence the middle piece is nonincreasing. It starts from \(D\) at
  \(u=-B_\epsilon/2-2\zeta\) and, by
  \eqref{eq:lower-bound-smoothing-match}, ends at
  \begin{equation}
  \label{eqn:end-tilde-g}
    \frac{
          D((\bm x^\circ)^\top\bm\theta^\circ-B_\epsilon/2)
      }{
          (\bm x^\circ)^\top\bm\theta^\circ-B_\epsilon/2+2\zeta
      }
      >
      0.  
  \end{equation}
  
  Since the left
  piece is constant at \(D\) and the right piece is positive and
  decreasing, \(\widetilde g\) is nonincreasing on
  \([-B_\epsilon,B_\epsilon]\) and satisfies
  \(0\le \widetilde g(u)\le D\) for \(u\in[-B_\epsilon,B_\epsilon]\).

  We next show that \(\widetilde g\le\bar g\). If
  \(u\le-B_\epsilon/2\), then
  \(\bar g(u)=D\) and \(\widetilde g(u)\le D\),
  so the desired inequality holds. If
  \(u\ge-B_\epsilon/2+2\zeta\), we have
  \(\widetilde g(u)=\bar g(u)\) by definition. It only remains 
  to consider the case \( u
      \in
      \left[
          -B_\epsilon/2,
          -B_\epsilon/2+2\zeta
      \right]\).
  For every such \(u\), by \eqref{eqn:end-tilde-g}, we have
  \[
  \begin{aligned}
      \widetilde g(u)
      &=
      \frac{
          D((\bm x^\circ)^\top\bm\theta^\circ-B_\epsilon/2)
      }{
          (\bm x^\circ)^\top\bm\theta^\circ-B_\epsilon/2+2\zeta
      }
     +
      \int_u^{-B_\epsilon/2+2\zeta}
      \frac{
          D((\bm x^\circ)^\top\bm\theta^\circ-B_\epsilon/2)
      }{
          ((\bm x^\circ)^\top\bm\theta^\circ+v)^2
      }
      \rho\left(
          \frac{v-z_0}{\zeta}
      \right)
      \,\mathrm dv
      \\
      &\le
      \frac{
          D((\bm x^\circ)^\top\bm\theta^\circ-B_\epsilon/2)
      }{
          (\bm x^\circ)^\top\bm\theta^\circ-B_\epsilon/2+2\zeta
      }
      +
      \int_u^{-B_\epsilon/2+2\zeta}
      \frac{
          D((\bm x^\circ)^\top\bm\theta^\circ-B_\epsilon/2)
      }{
          ((\bm x^\circ)^\top\bm\theta^\circ+v)^2
      }
      \,\mathrm dv
      \\
      &=
      \frac{
          D((\bm x^\circ)^\top\bm\theta^\circ-B_\epsilon/2)
      }{
          (\bm x^\circ)^\top\bm\theta^\circ+u
      }
      =
      \bar g(u).
  \end{aligned}
  \]
  Consequently, we have
  \(0\le\widetilde g(u)\le\bar g(u)\) for \(u\in[-B_\epsilon,B_\epsilon]\).

  Finally, the smoothing modification does not affect the flat-revenue
  region. Indeed, by \eqref{eq:lower-bound-zeta-choice},
  \[
      -\frac{B_\epsilon}{2}+2\zeta
      <
      -\frac{3B_\epsilon}{8}
      =
      p_L-(\bm x^\circ)^\top\bm\theta^\circ.
  \]
  Thus the smoothing interval lies strictly to the left of
  \(
      \left[
          p_L-(\bm x^\circ)^\top\bm\theta^\circ,\,
          p_U-(\bm x^\circ)^\top\bm\theta^\circ
      \right].
  \)
  It follows from the definition of \(\widetilde g\) that
  \[
      \widetilde g(u)
      =
      \frac{
          D((\bm x^\circ)^\top\bm\theta^\circ-B_\epsilon/2)
      }{
          (\bm x^\circ)^\top\bm\theta^\circ+u
      }
  \]
  on a neighborhood of
  \(
      \left[
          p_L-(\bm x^\circ)^\top\bm\theta^\circ,\,
          p_U-(\bm x^\circ)^\top\bm\theta^\circ
      \right].
  \)
  Hence the smoothing operation leaves the flat-revenue region unchanged.

  \paragraph{Step 3: Enforcing the zero-mean condition.}
  Although \(\widetilde g\) is smooth and preserves the flat-revenue
  region, its integral remains larger than \(DB_\epsilon\). Indeed,
  \(\widetilde g\) differs from \(\bar g\) only on
  \(
      \left[
          -B_\epsilon/2-2\zeta,\,
          -B_\epsilon/2+2\zeta
      \right],
  \)
  whose length is \(4\zeta\). Since both functions take values in
  \([0,D]\), equations
  \eqref{eq:lower-bound-envelope-excess} and
  \eqref{eq:lower-bound-zeta-choice} give
  \[
  \begin{aligned}
      \int_{-B_\epsilon}^{B_\epsilon}
      \widetilde g(u)\,\mathrm du
      &\ge
      \int_{-B_\epsilon}^{B_\epsilon}
      \bar g(u)\,\mathrm du
      -
      4D\zeta
      >
      DB_\epsilon.
  \end{aligned}
  \]
  We therefore remove part of the right tail of \(\widetilde g\)
  smoothly, while leaving the flat-revenue region unchanged.
  
  For every
  \(
      \xi
      \in
      [p_U-(\bm x^\circ)^\top\bm\theta^\circ+\zeta,\,
       B_\epsilon-2\zeta],
  \)
  define
  \[
      g_\xi(u)
      :=
      \widetilde g(u)
      \left[
          1-
          \rho\left(
              \frac{u-\xi}{\zeta}
          \right)
      \right],
      \quad
      u\in[-B_\epsilon,B_\epsilon].
  \]
  For \(u\le\xi\), the multiplicative factor equals one, whereas for
  \(u\ge\xi+\zeta\), it equals zero. On the interval
  \((\xi,\xi+\zeta)\), it decreases smoothly from one to zero. Thus
  \(g_\xi\) agrees with \(\widetilde g\) to the left of \(\xi\) and
  vanishes to the right of \(\xi+\zeta\).
  
  Since \(\widetilde g\) is nonnegative and nonincreasing, and since
  \(\rho\) is nondecreasing,
  \[
  \begin{aligned}
      g_\xi'(u)
      ={}&
      \widetilde g'(u)
      \left[
          1-
          \rho\left(
              \frac{u-\xi}{\zeta}
          \right)
      \right]
      -
      \frac{\widetilde g(u)}{\zeta}
      \rho'\left(
          \frac{u-\xi}{\zeta}
      \right)
      \le
      0.
  \end{aligned}
  \]
  Hence \(g_\xi\) is nonincreasing. Moreover, we have 
  \(
      0
      \le
      g_\xi(u)
      \le
      \widetilde g(u)
      \le
      \bar g(u).
  \)
  Because
  \(
      \xi
      \ge
      p_U-(\bm x^\circ)^\top\bm\theta^\circ+\zeta,
  \)
  the multiplicative factor equals one throughout
  \(
      [p_L-(\bm x^\circ)^\top\bm\theta^\circ,\,
       p_U-(\bm x^\circ)^\top\bm\theta^\circ].
  \)
  Therefore, on this interval,
  \[
      g_\xi(u)
      =
      \widetilde g(u)
      =
      \frac{
          D((\bm x^\circ)^\top\bm\theta^\circ-B_\epsilon/2)
      }{
          (\bm x^\circ)^\top\bm\theta^\circ+u
      }.
  \]
  In addition, \(g_\xi=D\) on a neighborhood of \(-B_\epsilon\) and
  \(g_\xi=0\) on a neighborhood of \(B_\epsilon\).
  
  For every fixed \(u\), the function \(g_\xi(u)\) is continuous and
  nondecreasing in \(\xi\). Since \(0\le g_\xi(u)\le D\), the dominated
  convergence theorem implies that
  \[
      \xi
      \longmapsto
      \int_{-B_\epsilon}^{B_\epsilon}
      g_\xi(u)\,\mathrm du
  \]
  is continuous and nondecreasing.
  
  At
  \(
      \xi
      =
      p_U-(\bm x^\circ)^\top\bm\theta^\circ+\zeta,
  \)
  the function \(g_\xi\) vanishes for
  \(
      u
      \ge
      p_U-(\bm x^\circ)^\top\bm\theta^\circ+2\zeta.
  \)
  Consequently,
  \[
  \begin{aligned}
      \int_{-B_\epsilon}^{B_\epsilon}
      g_\xi(u)\,\mathrm du
      &\le
      D\left(
          B_\epsilon+p_U-(\bm x^\circ)^\top\bm\theta^\circ+2\zeta
      \right)
      =
      D\left(
          \frac{3B_\epsilon}{4}+2\zeta
      \right)
      <
      DB_\epsilon.
  \end{aligned}
  \]
  
  At \(\xi=B_\epsilon-2\zeta\), the functions \(g_\xi\) and \(\bar g\)
  may differ only on
  \[
      \left[
          -\frac{B_\epsilon}{2}-2\zeta,\,
          -\frac{B_\epsilon}{2}+2\zeta
      \right]
      \cup
      [B_\epsilon-2\zeta,B_\epsilon],
  \]
  whose total length is at most \(6\zeta\). Since both functions take
  values in \([0,D]\), equations
  \eqref{eq:lower-bound-envelope-excess} and
  \eqref{eq:lower-bound-zeta-choice} imply
  \[
  \begin{aligned}
      \int_{-B_\epsilon}^{B_\epsilon}
      g_{B_\epsilon-2\zeta}(u)\,\mathrm du
      &\ge
      \int_{-B_\epsilon}^{B_\epsilon}
      \bar g(u)\,\mathrm du
      -
      6D\zeta
      >
      DB_\epsilon.
  \end{aligned}
  \]
  Therefore, by continuity, there exists
  \(
      \xi_\star
      \in
      [p_U-(\bm x^\circ)^\top\bm\theta^\circ+\zeta,\,
       B_\epsilon-2\zeta]
  \)
  such that
  \[
      \int_{-B_\epsilon}^{B_\epsilon}
      g_{\xi_\star}(u)\,\mathrm du
      =
      DB_\epsilon.
  \]

  \paragraph{Step 4: Verifying the baseline properties and constructing
  the compensation function.}
  We now set \(g_0=g_{\xi_\star}\). We first verify that \(g_0\)
  corresponds to an admissible zero-mean noise distribution and generates
  the desired flat globally optimal revenue region. We then construct
  \(\chi\), which will be used to preserve the zero-mean condition in the
  subsequent perturbation construction.
  
  Set \(g_0=g_{\xi_\star}\) on
  \([-B_\epsilon,B_\epsilon]\), and extend it to \(\mathbb R\) by \( g_0(u)=D\) for \(u\le-B_\epsilon\), and \( g_0(u)=0\) for \(u\ge B_\epsilon\).
  Because \(g_{\xi_\star}\) is constant on neighborhoods of both
  endpoints, this extension is infinitely differentiable on
  \(\mathbb R\). In particular, \(g_0\in\mathcal H(\beta)\) on
  \([-B,B]\), with a finite H\"older constant.

  The function \(g_0/D\) is continuous, nonincreasing, equals one to the
  left of \(-B_\epsilon\), and equals zero to the right of
  \(B_\epsilon\). It is therefore the survival function of a continuous
  distribution supported on \([-B_\epsilon,B_\epsilon]\). Moreover, by the choice
  of \(\xi_\star\) in Step~3, 
  \[
  \begin{aligned}
      \mathbb E[\epsilon_t]
      &=
      \int_{-B_\epsilon}^{B_\epsilon}
      \frac{g_0(u)}{D}
      \,\mathrm du
      -
      B_\epsilon
      =
      0.
  \end{aligned}
  \]
  
  We next verify the revenue properties. If \(p-(\bm x^\circ)^\top\bm\theta^\circ\in[-B_\epsilon,B_\epsilon]\) and \(p>0\), then
  \(g_0\le\bar g\), and hence
  \[
  \begin{aligned}
      p\,g_0(p-(\bm x^\circ)^\top\bm\theta^\circ)
      &\le
      p\min\left\{
          D,\,
          \frac{
              D((\bm x^\circ)^\top\bm\theta^\circ-B_\epsilon/2)
          }{
              p
          }
      \right\}
      \le
      D\left(
          (\bm x^\circ)^\top\bm\theta^\circ-\frac{B_\epsilon}{2}
      \right).
  \end{aligned}
  \]
  If \(p=0\), the same inequality is immediate. If
  \(p-(\bm x^\circ)^\top\bm\theta^\circ\le-B_\epsilon\), then \(g_0(p-(\bm x^\circ)^\top\bm\theta^\circ)=D\) and
  \[
      p\,g_0(p-(\bm x^\circ)^\top\bm\theta^\circ)
      =
      Dp
      \le
      D((\bm x^\circ)^\top\bm\theta^\circ-B_\epsilon)
      <
      D\left(
          (\bm x^\circ)^\top\bm\theta^\circ-\frac{B_\epsilon}{2}
      \right).
  \]
  If \(p-(\bm x^\circ)^\top\bm\theta^\circ\ge B_\epsilon\), then \(g_0(p-(\bm x^\circ)^\top\bm\theta^\circ)=0\). This proves the inequality in
  \eqref{eq:lower-bound-flat-revenue}.
  
  For every \(p\in[p_L,p_U]\), Step~3 gives
  \[
      g_0(p-(\bm x^\circ)^\top\bm\theta^\circ)
      =
      \frac{
          D((\bm x^\circ)^\top\bm\theta^\circ-B_\epsilon/2)
      }{
          p
      }.
  \]
  Therefore,
  \[
      p\,g_0\left(
          p-(\bm x^\circ)^\top\bm\theta^\circ
      \right)
      =
      D\left(
          (\bm x^\circ)^\top\bm\theta^\circ
          -
          \frac{B_\epsilon}{2}
      \right),
      \quad
      p\in[p_L,p_U],
  \]
  which proves the equality in \eqref{eq:lower-bound-flat-revenue}.

  Furthermore, for every \(p\in[p_L,p_U]\),
  \[
      0
      <
      \frac{
          D((\bm x^\circ)^\top\bm\theta^\circ-B_\epsilon/2)
      }{
          p_U
      }
      \le
      g_0(p-(\bm x^\circ)^\top\bm\theta^\circ)
      \le
      \frac{
          D((\bm x^\circ)^\top\bm\theta^\circ-B_\epsilon/2)
      }{
          p_L
      }
      <
      D.
  \]
  Because the same representation holds on a neighborhood of
  \(
      [p_L-(\bm x^\circ)^\top\bm\theta^\circ,\,
       p_U-(\bm x^\circ)^\top\bm\theta^\circ],
  \)
  the function \(g_0\) is bounded away from both \(0\) and \(D\), and
  \[
      g_0'(u)
      =
      -\frac{
          D((\bm x^\circ)^\top\bm\theta^\circ-B_\epsilon/2)
      }{
          ((\bm x^\circ)^\top\bm\theta^\circ+u)^2
      }
      <
      0
  \]
  throughout that neighborhood.

  It remains to construct the compensation function. Define
  \[
      \chi(u)
      :=
      \frac{2}{\zeta}
      \rho'\left(
          \frac{2(u-\xi_\star)}{\zeta}
          -
          \frac12
      \right).
  \]
  Since \(\rho'\ge0\),
  we have \(
      \chi(u)\ge0.
  \)
  By the definition of \(\chi\),
  \[
  \begin{aligned}
      \int_{\mathbb R}\chi(u)\,\mathrm du
      =
      \int_{\mathbb R}
          \frac{2}{\zeta}
          \rho'\left(
              \frac{2(u-\xi_\star)}{\zeta}
              -
              \frac12
          \right)
      \,\mathrm du
      =
      \left[
          \rho\left(
              \frac{2(u-\xi_\star)}{\zeta}
              -
              \frac12
          \right)
      \right]_{u=-\infty}^{u=\infty}
      =
      1-0
      =
      1.
  \end{aligned}
  \]

  Moreover, since \(\rho'\) vanishes outside \((0,1)\), by the choice of \(\xi_\star\) in Step~3
  \[
      \operatorname{supp}(\chi)
      \subset
      \left[
          \xi_\star+\frac{\zeta}{4},\,
          \xi_\star+\frac{3\zeta}{4}
      \right]
      \subset
      (-B_\epsilon,B_\epsilon)
      \setminus
      [p_L-(\bm x^\circ)^\top\bm\theta^\circ,p_U-(\bm x^\circ)^\top\bm\theta^\circ].
  \]
  
  On a sufficiently small neighborhood of \(\operatorname{supp}(\chi)\),
  the factor
  \(
      1-
      \rho\left(
          \frac{u-\xi_\star}{\zeta}
      \right)
  \)
  is bounded away from both zero and one. Since \(\widetilde g\) is also
  bounded away from both zero and \(D\) there, the same holds for \(g_0\).
  
  Finally,
  \[
  \begin{aligned}
      g_0'(u)
      ={}&
      \widetilde g'(u)
      \left[
          1-
          \rho\left(
              \frac{u-\xi_\star}{\zeta}
          \right)
      \right]
      -
      \frac{\widetilde g(u)}{\zeta}
      \rho'\left(
          \frac{u-\xi_\star}{\zeta}
      \right)
      <
      0,
  \end{aligned}
  \]
  because \(\widetilde g'(u)<0\), \(\widetilde g(u)>0\), and
  \(\rho'\ge0\) there. Since \(g_0'\) is continuous, it is uniformly
  negative on a sufficiently small closed neighborhood of
  \(\operatorname{supp}(\chi)\).
  \end{myproof}

  \noindent\textbf{Proof of Theorem~\ref{thm:lower_bound_survival}.}
  
  \begin{myproof}
  \paragraph{Step 1: Constructing the perturbed instances.}
  Let \(g_0,p_L,p_U\), and \(\chi\) be as in
  Lemma~\ref{lem:lower_bound_flat_baseline}. Set
  \(
      K_T
      :=
      \left\lceil
          T^{1/(2\beta+1)}
      \right\rceil
  \)
  and, for every \(k\in[K_T]\), define
  \[
      \bar p_k
      :=
      p_L+
      \left(
          k-\frac12
      \right)
      \frac{p_U-p_L}{K_T}.
  \]
  Thus, \(\bar p_1,\ldots,\bar p_{K_T}\) are equally spaced in
  \([p_L,p_U]\).
  For brevity, write
  \(
      u_k:=\bar p_k-(\bm x^\circ)^\top\bm\theta^\circ.
  \)
  
  Choose a sufficiently small constant \(\kappa>0\), depending only on
  the fixed problem primitives. For every \(k\in[K_T]\),
  define
  \begin{equation}
  \label{eq:lower-bound-local-perturbation}
      \Delta_k(u)
      :=
      \begin{cases}
          \displaystyle
          \frac{
              \kappa K_T^{-\beta}
          }{
              \bigl((\bm x^\circ)^\top\bm\theta^\circ+u\bigr)
              \rho'(1/2)
          }
          \rho'\left(
              \frac{
                  4K_T
                  (u-u_k)
              }{
                  p_U-p_L
              }
              +
              \frac12
          \right),
          &
          \displaystyle
          |u-u_k|
          <
          \frac{p_U-p_L}{8K_T},
          \\[5mm]
          0,
          &
          \textnormal{otherwise},
      \end{cases}
  \end{equation}
  and
  \begin{equation}
  \label{eq:lower-bound-perturbed-link}
      g_k(u)
      :=
      g_0(u)
      +
      \Delta_k(u)
      -
      \chi(u)
      \int_{\mathbb R}\Delta_k(v)\,\mathrm dv.
  \end{equation}
  The term \(\Delta_k\) creates a local revenue increase near
  \(\bar p_k\), while the term involving \(\chi\) removes the same total
  mass away from this region.
  
  Because \(\rho'(1/2)>0\), \(\Delta_k\) is well defined and
  nonnegative. Moreover,
  \[
      \operatorname{supp}(\Delta_k)
      \subset
      \left[
          \bar p_k-(\bm x^\circ)^\top\bm\theta^\circ
          -\frac{p_U-p_L}{8K_T},\,
          \bar p_k-(\bm x^\circ)^\top\bm\theta^\circ
          +\frac{p_U-p_L}{8K_T}
      \right].
  \]
  These supports are pairwise disjoint and contained in
  \(
      [p_L-(\bm x^\circ)^\top\bm\theta^\circ,\,
       p_U-(\bm x^\circ)^\top\bm\theta^\circ].
  \)
  Since
  \(
      (\bm x^\circ)^\top\bm\theta^\circ+u
      \ge
      p_L
      >
      0
  \)
  on their union, the height of \(\Delta_k\) is of order
  \(K_T^{-\beta}\), while the length of its support is of order
  \(K_T^{-1}\). Hence
  \begin{equation}
  \label{eq:lower-bound-perturbation-mass}
      0
      \le
      \int_{\mathbb R}\Delta_k(u)\,\mathrm du
      \le
      C_{\mathrm{lb}}\kappa K_T^{-(\beta+1)}
  \end{equation}
  uniformly in \(k\) and \(T\), where \(C_{\mathrm{lb}}<\infty\) denotes
  a constant that may change from line to line.
  
  Because \(\rho\) is infinitely differentiable and constant outside
  \([0,1]\), all derivatives of \(\rho'\) vanish at zero and one.
  Therefore, the piecewise-defined function \(\Delta_k\) is infinitely
  differentiable on \(\mathbb R\).
  
  Since the supports of \(\Delta_k\) and \(\chi\) lie in
  \((-B_\epsilon,B_\epsilon)\), and
  \(\int_{\mathbb R}\chi(u)\,\mathrm du=1\),
  \[
  \begin{aligned}
      \int_{-B_\epsilon}^{B_\epsilon}
      g_k(u)\,\mathrm du
      =
      \int_{-B_\epsilon}^{B_\epsilon}
      g_0(u)\,\mathrm du
      +
      \int_{\mathbb R}\Delta_k(u)\,\mathrm du
      -
      \left(
          \int_{\mathbb R}\chi(u)\,\mathrm du
      \right)
      \left(
          \int_{\mathbb R}\Delta_k(u)\,\mathrm du
      \right)
      =
      DB_\epsilon.
  \end{aligned}
  \]
  Thus, \(g_k\) and \(g_0\) have the same integral on
  \([-B_\epsilon,B_\epsilon]\).

  We first verify that the constructed functions satisfy the required smoothness condition.
  By enlarging \(C_{\mathrm{lb}}\), if necessary, repeated application of the product and chain
  rules gives
  \begin{equation}
  \label{eq:lower-bound-perturbation-derivatives}
      \sup_{u\in\mathbb R}
      \left|
          \Delta_k^{(r)}(u)
      \right|
      \le
      C_{\mathrm{lb}}\kappa K_T^{r-\beta},
      \quad
      r=0,1,\ldots,\varpi(\beta)+1.
  \end{equation}
  We now verify the H\"older condition. Fix \(u,u'\in\mathbb R\). If
  \(
      |u-u'|
      \le
      K_T^{-1},
  \)
  the mean-value theorem gives
  \[
  \begin{aligned}
  \left|
      \Delta_k^{(\varpi(\beta))}(u)
      -
      \Delta_k^{(\varpi(\beta))}(u')
  \right|
  &\le
  C_{\mathrm{lb}}\kappa
  K_T^{\varpi(\beta)+1-\beta}
  |u-u'|
  \\
  &=
  C_{\mathrm{lb}}\kappa
  |u-u'|^{\beta-\varpi(\beta)}
  \left(
      K_T|u-u'|
  \right)^{\varpi(\beta)+1-\beta}
  \\
  &\le
  C_{\mathrm{lb}}\kappa
  |u-u'|^{\beta-\varpi(\beta)}.
  \end{aligned}
  \]
  
  If
  \(
      |u-u'|
      >
      K_T^{-1},
  \)
  then
  \[
  \begin{aligned}
  \left|
      \Delta_k^{(\varpi(\beta))}(u)
      -
      \Delta_k^{(\varpi(\beta))}(u')
  \right|
  &\le
  2
  \sup_{v\in\mathbb R}
  \left|
      \Delta_k^{(\varpi(\beta))}(v)
  \right|
  \le
  C_{\mathrm{lb}}\kappa
  K_T^{\varpi(\beta)-\beta}
  \le
  C_{\mathrm{lb}}\kappa
  |u-u'|^{\beta-\varpi(\beta)}.
  \end{aligned}
  \]
  Therefore, we have
  \begin{equation}
  \label{eq:lower-bound-perturbation-holder}
  \left|
      \Delta_k^{(\varpi(\beta))}(u)
      -
      \Delta_k^{(\varpi(\beta))}(u')
  \right|
  \le
  C_{\mathrm{lb}}\kappa
  |u-u'|^{\beta-\varpi(\beta)}.
  \end{equation}
  
  By Lemma~\ref{lem:lower_bound_flat_baseline}, \(g_0\) satisfies the
  required H\"older condition. Moreover, \(\chi\) is fixed and infinitely
  differentiable, while
  \eqref{eq:lower-bound-perturbation-mass} uniformly bounds the
  coefficient multiplying \(\chi\). Combining these facts with
  \eqref{eq:lower-bound-perturbation-holder}, there exists a finite
  constant \(\bar L_g\), independent of \(k\) and \(T\), such that
  \(
      g_k
      \in
      \mathcal H(
          \beta,
          \bar L_g;
          [-B,B]
      )
  \)
  for every \(k\in[K_T]\).
  
  We next verify the range and monotonicity requirements. The supports of
  \(\Delta_k\) and \(\chi\) are disjoint. By
  Lemma~\ref{lem:lower_bound_flat_baseline}, \(g_0\) is bounded away from
  both \(0\) and \(D\), and \(g_0'\) is uniformly negative, on
  neighborhoods of these supports.
  
  By enlarging \(C_{\mathrm{lb}}\), if necessary,
  \eqref{eq:lower-bound-perturbation-derivatives} and
  \eqref{eq:lower-bound-perturbation-mass} imply
  \[
  \begin{aligned}
      \|\Delta_k\|_\infty
      +
      \|\Delta_k'\|_\infty
      +
      \left\|
          \chi(\cdot)
          \int_{\mathbb R}\Delta_k(v)\,\mathrm dv
      \right\|_\infty
      +
      \left\|
          \chi'(\cdot)
          \int_{\mathbb R}\Delta_k(v)\,\mathrm dv
      \right\|_\infty
      \le
      C_{\mathrm{lb}}\kappa.
  \end{aligned}
  \]
  On \(\operatorname{supp}(\Delta_k)\),
  \(
      g_k=g_0+\Delta_k,
  \)
  whereas on \(\operatorname{supp}(\chi)\),
  \[
      g_k
      =
      g_0-
      \chi(\cdot)
      \int_{\mathbb R}\Delta_k(v)\,\mathrm dv.
  \]
  Outside these two supports, \(g_k=g_0\). Combining these observations
  with the preceding bounds, for all sufficiently small \(\kappa>0\),
  uniformly in \(k\) and \(T\),
   \( 0
   \le g_k(u) \le D\) and \(g_k'(u) \le 0\) for \(u\in\mathbb R\).

  Since the supports of \(\Delta_k\) and \(\chi\) lie strictly inside
  \((-B_\epsilon,B_\epsilon)\), \( g_k(u)=D\) for \(u\le-B_\epsilon\) and \(g_k(u)=0\) for \(u\ge B_\epsilon\).
  Therefore, \(g_k/D\) is the survival function of a distribution
  supported on \([-B_\epsilon,B_\epsilon]\). Moreover, the preceding
  integral identity gives
  \[
  \begin{aligned}
      \mathbb E[\epsilon_t]
      &=
      \int_{-B_\epsilon}^{B_\epsilon}
      \frac{g_k(u)}{D}\,\mathrm du
      -
      B_\epsilon
      =
      0.
  \end{aligned}
  \]

Under instance \(k\), the shocks are i.i.d.\ with the constructed
distribution and are independent of the policy randomization. Moreover,
because \(y_t=D\mathbf 1\{v_t\ge p_t\}\),
the demand is measurable with respect to \(\sigma(p_t,v_t)\). Therefore,
\[
\begin{aligned}
    \mathbb E\!\left[
        y_t
        \,\middle|\,
        \mathcal F_{t-1}
        \vee
        \sigma(\bm x_t,p_t,v_t)
    \right]
    =
    D\mathbf 1\{v_t\ge p_t\} =
    q(v_t-p_t).
\end{aligned}
\]
Hence, Assumption~\ref{ass:quantity-response} holds.
Furthermore, by \eqref{eq:lower-bound-demand-subclass}, for every
  \(u\in\mathcal I_g\),
  \[
      \mathbb E[q(\epsilon_t-u)]
      =
      D\mathbb P(\epsilon_t\ge u)
      =
      g_k(u).
  \]
  Thus, each constructed instance satisfies the original model
  assumptions provided that \(L_g\ge\bar L_g\).

  \paragraph{Step 2: Establishing the revenue gap and statistical closeness.}
  Under instance \(k\), the perturbation raises the revenue near
  \(\bar p_k\), while prices outside its support receive no such
  increase.
  
  Since the support of \(\chi\) is disjoint from
  \(
      [p_L-(\bm x^\circ)^\top\bm\theta^\circ,\,
       p_U-(\bm x^\circ)^\top\bm\theta^\circ],
  \)
  the definition of \(\Delta_k\) gives
  \[
  \begin{aligned}
      \bar p_k
      g_k\left(
          \bar p_k-(\bm x^\circ)^\top\bm\theta^\circ
      \right)
      =
      D\left(
          (\bm x^\circ)^\top\bm\theta^\circ
          -
          \frac{B_\epsilon}{2}
      \right)
      +
      \kappa K_T^{-\beta}.
  \end{aligned}
  \]
  
  If
  \(
      p-(\bm x^\circ)^\top\bm\theta^\circ
      \notin
      \operatorname{supp}(\Delta_k),
  \)
  then \(\Delta_k=0\). Since the term involving \(\chi\) is nonpositive, 
  \[ p\, g_k\left( p-(\bm x^\circ)^\top\bm\theta^\circ \right) \le p\, g_0\left( p-(\bm x^\circ)^\top\bm\theta^\circ \right) \le D\left( (\bm x^\circ)^\top\bm\theta^\circ - \frac{B_\epsilon}{2} \right). 
  \] 
  Consequently, 
  \begin{equation} 
  \label{eq:lower-bound-revenue-gap} 
  \begin{aligned} 
  \max_{p'\in[0,B]} p'\, g_k\left( p'-(\bm x^\circ)^\top\bm\theta^\circ \right) - p\, g_k\left( p-(\bm x^\circ)^\top\bm\theta^\circ \right) 
  \ge \kappa K_T^{-\beta}, \quad p-(\bm x^\circ)^\top\bm\theta^\circ \notin \operatorname{supp}(\Delta_k). 
  \end{aligned} 
  \end{equation}
  
  Fix an arbitrary nonanticipating policy
  \(\pi\). Let \(\mathbb P_k\) denote the law of the complete observed
  history, including posted prices and demands, under instance \(k\), and
  let \(\mathbb P_0\) denote the corresponding law under the baseline
  instance \(g_0\). Conditional on the past and \(p_t\), the normalized observation
  \(y_t/D\) is Bernoulli with success probability \(g_k\left(
              p_t-(\bm x^\circ)^\top\bm\theta^\circ
          \right)/D\). Since the policy uses the same conditional decision rule under every
  instance, the chain rule for relative entropy gives
  \begin{align}
  &D_{\mathrm{KL}}\left(
      \mathbb P_0
      \,\middle\|\,
      \mathbb P_k
  \right)
  =
  \mathbb E_{\mathbb P_0}
  \left[
      \sum_{t=1}^T
      D_{\mathrm{KL}}\left(
          \operatorname{Bern}\left(
              \frac{
                  g_0\left(
                      p_t-(\bm x^\circ)^\top\bm\theta^\circ
                  \right)
              }{
                  D
              }
          \right)
          \,\middle\|\,
          \operatorname{Bern}\left(
              \frac{
                  g_k\left(
                      p_t-(\bm x^\circ)^\top\bm\theta^\circ
                  \right)
              }{
                  D
              }
          \right)
      \right)
  \right].
  \label{eq:lower-bound-kl-chain-rule}
  \end{align}
  
  By the properties established in Step~1, the relevant Bernoulli
  probabilities on the supports of \(\Delta_k\) and \(\chi\) lie in a
  fixed subinterval of \((0,1)\). By enlarging
  \(C_{\mathrm{lb}}\), if necessary,
  \[
      D_{\mathrm{KL}}\left(
          \operatorname{Bern}(a)
          \,\middle\|\,
          \operatorname{Bern}(b)
      \right)
      \le
      C_{\mathrm{lb}}(a-b)^2
  \]
  for all probabilities appearing above. Moreover,
  \[
  \begin{aligned}
      \left|
          g_k\left(
              p-(\bm x^\circ)^\top\bm\theta^\circ
          \right)
          -
          g_0\left(
              p-(\bm x^\circ)^\top\bm\theta^\circ
          \right)
      \right|
      &\le
      C_{\mathrm{lb}}\kappa K_T^{-\beta}
      \mathbf 1\left\{
          p-(\bm x^\circ)^\top\bm\theta^\circ
          \in
          \operatorname{supp}(\Delta_k)
      \right\}
      \\
      &\qquad+
      C_{\mathrm{lb}}\kappa K_T^{-(\beta+1)}
      \mathbf 1\left\{
          p-(\bm x^\circ)^\top\bm\theta^\circ
          \in
          \operatorname{supp}(\chi)
      \right\}.
  \end{aligned}
  \]
  Using \eqref{eq:lower-bound-perturbation-derivatives} with \(r=0\)
  and \eqref{eq:lower-bound-perturbation-mass}, and enlarging
  \(C_{\mathrm{lb}}\) if necessary, we obtain
  \begin{align}
  &D_{\mathrm{KL}}\left(
      \mathbb P_0
      \,\middle\|\,
      \mathbb P_k
  \right)
  \le
  C_{\mathrm{lb}}\kappa^2K_T^{-2\beta}
  \mathbb E_0
  \left[
      \sum_{t=1}^T
      \mathbf 1\left\{
          p_t-(\bm x^\circ)^\top\bm\theta^\circ
          \in
          \operatorname{supp}(\Delta_k)
      \right\}
  \right]
  +
  C_{\mathrm{lb}}\kappa^2K_T^{-2\beta-2}T.
  \label{eq:lower-bound-kl-visit-count}
  \end{align}
  
  Because the supports of
  \(\Delta_1,\ldots,\Delta_{K_T}\) are pairwise disjoint,
  \[
      \sum_{k=1}^{K_T}
      \sum_{t=1}^T
      \mathbf 1\left\{
          p_t-(\bm x^\circ)^\top\bm\theta^\circ
          \in
          \operatorname{supp}(\Delta_k)
      \right\}
      \le
      T
  \]
  pathwise. Hence there exists \(k^\star\in[K_T]\) such that
  \[
      \mathbb E_0
      \left[
          \sum_{t=1}^T
          \mathbf 1\left\{
              p_t-(\bm x^\circ)^\top\bm\theta^\circ
              \in
              \operatorname{supp}(\Delta_{k^\star})
          \right\}
      \right]
      \le
      \frac{T}{K_T}.
  \]
  For this \(k^\star\), we have
  \[
  \begin{aligned}
      D_{\mathrm{KL}}\left(
          \mathbb P_0
          \,\middle\|\,
          \mathbb P_{k^\star}
      \right)
      &\le
      C_{\mathrm{lb}}\kappa^2
      T K_T^{-(2\beta+1)}
      +
      C_{\mathrm{lb}}\kappa^2
      T K_T^{-(2\beta+2)}.
  \end{aligned}
  \]
  Since \(K_T\ge1\), by enlarging \(C_{\mathrm{lb}}\), if necessary, we have
  \begin{equation}
  \label{eq:lower-bound-kl-final}
      D_{\mathrm{KL}}\left(
          \mathbb P_0
          \,\middle\|\,
          \mathbb P_{k^\star}
      \right)
      \le
      C_{\mathrm{lb}}\kappa^2
      T K_T^{-(2\beta+1)}.
  \end{equation}
  
  \paragraph{Step 3: Deriving the regret lower bound.}
  The KL bound in Step~2 shows that instance \(k^\star\) is difficult to
  distinguish from the baseline. We now use this fact to show that the
  policy misses the corresponding perturbation region sufficiently
  often. Under instance \(k^\star\), \eqref{eq:lower-bound-revenue-gap} gives
  \begin{equation}
  \label{eq:lower-bound-regret-count}
      \mathrm{Reg}(T)
      \ge
      \kappa K_T^{-\beta}
      \sum_{t=1}^T
      \mathbf 1\left\{
          p_t-(\bm x^\circ)^\top\bm\theta^\circ
          \notin
          \operatorname{supp}(\Delta_{k^\star})
      \right\}.
  \end{equation}
  
  For every \(t\), apply the Bretagnolle--Huber inequality to the
  event
  \(
      \left\{
          p_t-(\bm x^\circ)^\top\bm\theta^\circ
          \in
          \operatorname{supp}(\Delta_{k^\star})
      \right\}.
  \)
  Summing over \(t\) gives
  \begin{align}
  &\mathbb E_0
  \left[
      \sum_{t=1}^T
      \mathbf 1\left\{
          p_t-(\bm x^\circ)^\top\bm\theta^\circ
          \in
          \operatorname{supp}(\Delta_{k^\star})
      \right\}
  \right]+
  \mathbb E_{k^\star}
  \left[
      \sum_{t=1}^T
      \mathbf 1\left\{
          p_t-(\bm x^\circ)^\top\bm\theta^\circ
          \notin
          \operatorname{supp}(\Delta_{k^\star})
      \right\}
  \right]
  \notag\\
  &\ge
  \frac{T}{2}
  \exp\left\{
      -
      D_{\mathrm{KL}}\left(
          \mathbb P_0
          \,\middle\|\,
          \mathbb P_{k^\star}
      \right)
  \right\}.
  \label{eq:lower-bound-bh-sum}
  \end{align}
  
  By the definition of \(K_T\), we have
  \(
      T K_T^{-(2\beta+1)}
      \le
      1.
  \)
  By further decreasing \(\kappa\), if necessary, assume that
  \(
      C_{\mathrm{lb}}\kappa^2
      \le
      \log2.
  \)
  This choice preserves all the range, monotonicity, and smoothness
  properties established above. Equation~\eqref{eq:lower-bound-kl-final}
  then implies
  \(
      D_{\mathrm{KL}}\left(
          \mathbb P_0
          \,\middle\|\,
          \mathbb P_{k^\star}
      \right)
      \le
      \log2.
  \)
  Increasing \(T_0\), if necessary, also guarantees \(K_T\ge8\).
  Consequently,
  \[
      \mathbb E_0
      \left[
          \sum_{t=1}^T
          \mathbf 1\left\{
              p_t-(\bm x^\circ)^\top\bm\theta^\circ
              \in
              \operatorname{supp}(\Delta_{k^\star})
          \right\}
      \right]
      \le
      \frac{T}{K_T}
      \le
      \frac{T}{8}.
  \]
  Combining this bound with \eqref{eq:lower-bound-bh-sum} gives
  \[
  \begin{aligned}
      &\mathbb E_{k^\star}
      \left[
          \sum_{t=1}^T
          \mathbf 1\left\{
              p_t-(\bm x^\circ)^\top\bm\theta^\circ
              \notin
              \operatorname{supp}(\Delta_{k^\star})
          \right\}
      \right]
     \ge
      \frac{T}{4}
      -
      \frac{T}{8}
      =
      \frac{T}{8}.
  \end{aligned}
  \]
  Taking \(\mathbb E_{k^\star}\) on both sides of
  \eqref{eq:lower-bound-regret-count} therefore yields
  \[
      \mathbb E_{k^\star}[\mathrm{Reg}(T)]
      \ge
      \frac{\kappa}{8}
      T K_T^{-\beta}.
  \]
  
  Finally, since
  \(
      K_T
      \le
      2T^{1/(2\beta+1)},
  \)
  we obtain
  \[
      \mathbb E_{k^\star}[\mathrm{Reg}(T)]
      \ge
      \frac{\kappa}{2^{\beta+3}}
      T^{\frac{\beta+1}{2\beta+1}}.
  \]
  The policy \(\pi\) was arbitrary. Taking the infimum
  over policies and the supremum over admissible instances completes the
  proof.
  \end{myproof}

\end{document}